\documentclass{article}

\usepackage[preprint]{neurips_2026}

\usepackage[utf8]{inputenc}
\usepackage[T1]{fontenc}
\usepackage{hyperref}
\usepackage{url}
\usepackage{booktabs}
\usepackage{amsmath,amssymb,amsfonts,amsthm}
\usepackage{bm}
\usepackage{nicefrac}
\usepackage{microtype}
\usepackage{xcolor}
\usepackage{enumitem}

\usepackage{xcolor}
\usepackage{amsmath,amssymb,amsthm,mathtools}
\usepackage{algorithm}
\usepackage{algpseudocode}
\usepackage{caption}
\usepackage{enumitem}
\usepackage{titletoc}

\usepackage{graphicx}
\usepackage{subcaption}
\newcommand{\bA}{\mathbf{A}}
\newcommand{\be}{\mathbf{e}}
\newcommand{\ba}{\mathbf{a}}

\newcommand{\bH}{\mathbf{H}}
\newcommand{\bR}{\mathbf{R}}
\newcommand{\bQ}{\mathbf{Q}}
\newcommand{\bSigma}{\mathbf{\Sigma}}
\newcommand{\bOmega}{\mathbf{\Omega}}
\newcommand{\bTheta}{\mathbf{\Theta}}
\newcommand{\bDelta}{\mathbf{\Delta}}
\newcommand{\balpha}{\bm \alpha}
\newcommand{\bbeta}{\bm \beta}
\newcommand{\bd}{\mathbf{d}}
\newcommand{\bq}{\mathbf{q}}
\newcommand{\bu}{\mathbf{u}}
\newcommand{\bI}{\mathbf{I}}
\newcommand{\bM}{\mathbf{M}}

\newcommand{\bV}{\mathbf{V}}
\newcommand{\bW}{\mathbf{W}}
\newcommand{\bN}{\mathbf{N}}
\newcommand{\bK}{\mathbf{K}}
\newcommand{\supp}{\mathrm{supp}}
\newcommand{\row}{\mathrm{row}}

\newcommand{\rank}{\mathrm{rank}}
\newcommand{\bB}{\mathbf{B}}
\newcommand{\bG}{\mathbf{G}}
\newcommand{\bD}{\mathbf{D}}
\newcommand{\bS}{\mathbf{S}}
\newcommand{\bP}{\mathbf{P}}
\newcommand{\bbR}{\mathbb{R}}
\newcommand{\bbE}{\mathbb{E}}
\newcommand{\cT}{\mathcal{T}}
\newcommand{\cG}{\mathcal{G}}
\newcommand{\cJ}{\mathcal{J}}
\newcommand{\cS}{\mathcal{S}}
\newcommand{\cP}{\mathcal{P}}
\newcommand{\cR}{\mathcal{R}}
\newcommand{\cI}{\mathcal{I}}
\newcommand{\cV}{\mathcal{V}}
\newcommand{\cA}{\mathcal{A}}
\newcommand{\cC}{\mathcal{C}}
\newcommand{\cE}{\mathcal{E}}
\newcommand{\cD}{\mathcal{D}}
\newcommand{\cB}{\mathcal{B}}
\newcommand{\rblk}{\mathrm{blk}}

\newtheorem{assumption}{Assumption}
\newtheorem{theorem}{Theorem}
\newtheorem{lemma}{Lemma}
\newtheorem{corollary}{Corollary}
\newtheorem{definition}{Definition}
\newtheorem{remark}{Remark}
\newtheorem{proposition}{Proposition}

\title{Jigsaw-CRL: Recovering Global Latent Causal Order from Fragmented Multi-Client Interventions}

\author{%
  Haijie Xu, Chen Zhang \\
  Tsinghua University, Department of Industrial Engineering \\
  \texttt{xu-hj22@mails.tsinghua.edu.cn}
}

\author{%
 Haijie Xu\\
 Department of Industrial Engineering\\
 Tsinghua University\\
 Beijing 100084, China  \\
 \texttt{xu-hj22@mails.tsinghua.edu.cn} \\
 \And 
 Chen Zhang\thanks{Corresponding author}\\
 Department of Industrial Engineering\\
 Tsinghua University\\
 Beijing 100084, China  \\
\texttt{zhangchen01@tsinghua.edu.cn} \\
}

\newcommand{\chen}[1]{{\color{red}{\bf{Chen says:}} \emph{#1}}}

\begin{document}

\maketitle

\begin{abstract}
Causal representation learning (CRL) aims to recover latent causal variables and their structural relations from high-dimensional observations. Existing CRL methods typically assume that all environments are defined over the same latent variables, or at least share a common latent representation space. We study a fragmented multi-client setting, where multiple clients interact with the same global latent causal system but each client only accesses and intervenes on a subset of the latent variables. In this regime, marginalizing unused latent variables induces bidirected edges, so a single client no longer admits a node-wise latent causal graph, and the global latent causal order must be recovered by assembling client-specific structural fragments. We propose \textbf{Jigsaw-CRL}, a framework for recovering global latent causal order from such fragmented interventions. Under soft interventions, differences between precision matrices across environments exhibit a low-rank structure governed by latent ancestor relations. This enables recovery, for each client, of a block partition, the corresponding block-level ancestral order, and latent subspaces, and then assembly of these fragments into the global node-level latent causal order. We establish identifiability guarantees, develop practical algorithms, and validate the framework on synthetic data. Our codes are available on \url{https://anonymous.4open.science/r/code-for-Jigsaw-CRL-7B26}
\end{abstract}

\section{Introduction}
\label{sec:intro}

Causal representation learning (CRL) aims to recover high-level latent causal variables and their structural relations from high-dimensional observations 
\citep{scholkopf2021toward,jin2024learning,buchholz2023learning,varici2024general}. 
Most existing CRL frameworks assume that all observations are defined over the same latent variables, or at least share a common latent representation space \citep{zhang2024causal,buchholz2023learning}. 
Under this shared-space assumption, the natural target is typically a node-level latent causal structure, or a standard equivalence class thereof. We call this single-client setting.
In many multi-client systems, however, no single client has access to the full latent causal system. 
Instead, multiple clients interact with the same global latent causal system, while each client accesses and intervenes on only a subset of the latent variables. 
Each client therefore provides only a fragmented structural view of the global system, as illustrated in Figure~\ref{fig:intro1}. We call
this multi-client setting. 
This raises a basic question: \textbf{can we recover the global latent causal order by assembling fragmented causal information from multiple clients?}

This question is not a straightforward multi-client extension of existing CRL, because fragmented access changes the local causal object itself. 
From the perspective of a client, the latent variables outside its accessed subset are marginalized out, which can induce directed shortcuts and bidirected dependencies among the remaining latent variables. 
Thus, the client-specific latent structure is no longer an ordinary DAG, and node-wise latent causal recovery is generally impossible from a single fragmented client. 
This differs from standard CRL methods \citep{scholkopf2021toward,squires2023linear,buchholz2023learning,varici2024general,zhang2024causal}, which assume a shared latent variable set across environments. 
It also differs from recent multi-domain or multi-view CRL works \citep{sturma2023unpaired,yao2024multi}: \citet{sturma2023unpaired} assume causally separated shared and client-specific components, as illustrated in Figure~\ref{fig:intro3}, while \citet{yao2024multi} allow subset-dependent views but focus on shared latent information rather than ancestral relations or global causal structure. 
This leads to a first question: \textbf{when missing latent variables induce bidirected dependencies through marginalization, what causal structure remains identifiable to a fragmented client?}

We show that the natural client-level identifiable object is not a node-wise causal graph, but a block-structured mixed object: a block partition induced by directed ancestry and bidirected dependencies, together with a block-level ancestral order. 
Nodes within a block are structurally entangled for a single client, whereas ancestral relations between blocks remain identifiable. 
To recover this object, we propose \textbf{Jigsaw-CRL}, which exploits the fact that single-node soft interventions induce low-rank precision perturbations whose row spaces follow block-level ancestor relations. 
Jigsaw-CRL first extracts from each client a block-level causal fragment, including locally inseparable latent groups, their block-level order, and the associated latent subspaces. 
It then assembles complementary fragments across clients: a Hasse edge hidden inside a non-singleton block for one client may appear as a singleton-block parent relation for another, enabling recovery of the global node-level latent causal order. 
Thus, Jigsaw-CRL yields a two-level identifiability picture: \emph{each client reveals a block-level causal fragment, while multiple complementary fragments recover the global node-level causal order}. 

\textbf{Our Contributions.} \textbf{1) Fragmented multi-client CRL.} We introduce a setting where multiple clients partially access and intervene on the same global latent causal system, and show that marginalizing unused latent variables changes each client's local causal object from a DAG to a block-structured mixed graph; \textbf{2) Client-level block identifiability.} We study client-level identifiability in the presence of bidirected dependencies induced by latent marginalization, a regime not covered by prior CRL methods. We prove that soft interventions induce low-rank precision perturbations that identify each client's block partition, block-level ancestral order, and latent subspaces up to intrinsic block-level ambiguity in Theorem~\ref{thm:client-main} and Corollary~\ref{cor:zk}; \textbf{3) Local-to-global assembly.} We show that complementary client fragments can recover the global latent ancestral relation in Theorem~\ref{thm:global}: under a cross-client coverage condition, each direct step in the global causal order is resolved by at least one client;  \textbf{4) Algorithms and experiments.} We develop recovery algorithms based on projected low-rank precision perturbations and validate Jigsaw-CRL on synthetic fragmented multi-client systems.

\textbf{Organization.} Due to space constraints, we defer related work to Appendix~\ref{app:literature}. 
Section~\ref{sec:setup} introduces the fragmented multi-client CRL setting and the induced block structure. 
Sections~\ref{sec:client} and~\ref{sec:global} establish client-level identifiability and global assembly, respectively. 
Section~\ref{sec:experiments} reports experimental results, and Section~\ref{sec:conslusion} concludes.

\begin{figure}[t]
    \centering
    \begin{subfigure}{0.58\linewidth}
        \centering
        \includegraphics[width=\linewidth]{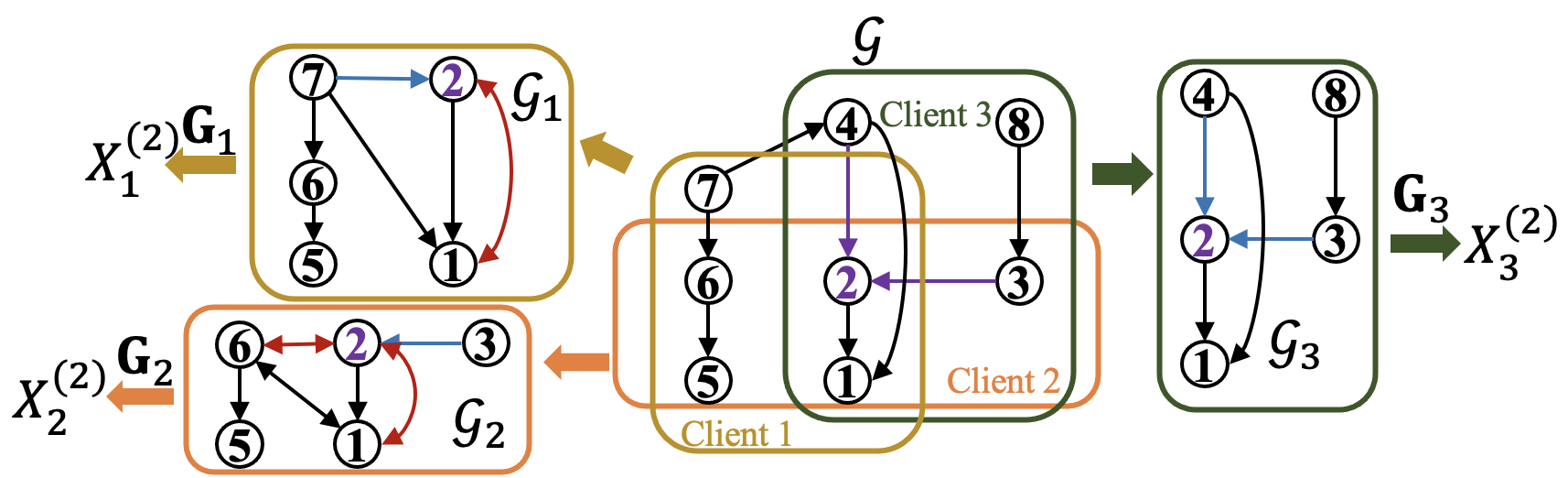}
        \caption{Multi-client setting (our setting)}
        \label{fig:intro1}
    \end{subfigure}
    \begin{subfigure}{0.39\linewidth}
        \centering
        \includegraphics[width=\linewidth]{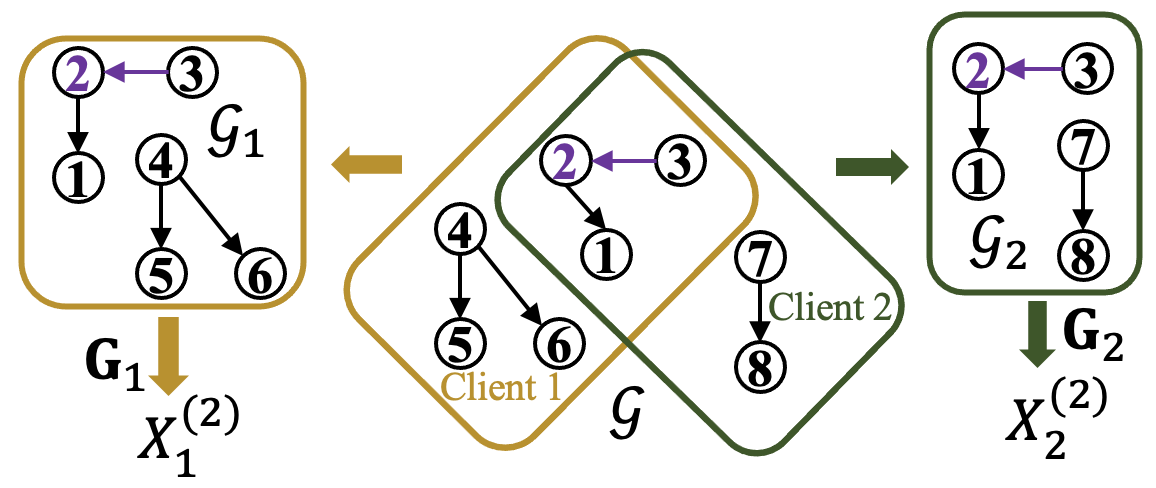}
        \caption{Multi-domain setting}
        \label{fig:intro3}
    \end{subfigure}

    \caption{
   \textbf{Fragmented client views of latent causal systems under different multi-client settings.}
All nodes in the figure represent latent variables. For client $k$, $X_k^{(j)}$ denotes the observed data after intervention $j$. \textcolor{purple}{Purple} highlights the intervened node $2$ together with the edges affected by the intervention.
(a) \textbf{Multi-client setting.} Different clients use different subsets of latent variables from the same global causal system. 
    Marginalizing the missing latent variables induces client-specific mixed graphs, where bidirected edges appear in addition to directed relations. 
    In a client-specific induced graph, the intervention on node $2$ may manifest as changes in directed edges (\textcolor{blue}{blue}) and bidirected edges (\textcolor{red}{red}), corresponding to $\bd_k^{(j)}$ and $\bu_k^{(j)}$ in Assumption \ref{ass:non-degeneracy}. 
(b) \textbf{Multi-domain setting.} Adopted in \cite{sturma2023unpaired}, latent variables are partitioned into a shared global component (e.g., $\{1,2,3\}$) and client-specific local components (e.g., $\{4,5,6\}$ and $\{7,8\}$). No edges exist across different local components, nor between global and local components. Consequently, client-specific induced graphs do not generate bidirected edges or extra directed edges after marginalization. This can be viewed as a special case of our more general setting.}
    
    \label{fig:intro}
\end{figure}
\section{Problem Setup and Client-Specific Block Structure}
\label{sec:setup}
This section formalizes the fragmented multi-client CRL setting and identifies the causal object that is locally visible to each client. We first define the global latent causal system and the client-specific observation model. We then show that marginalizing the latent variables unused by a client induces a mixed graph rather than a DAG. This motivates a client-specific block structure, which will be the basic unit of local identifiability in the sequel.

\textbf{Notation.} In the rest of the paper, we use non-bold lowercase letters to denote scalars (e.g., $p$), bold lowercase letters to denote vectors (e.g., $\bd$), non-bold uppercase letters to denote random variables (e.g., $Z$), and bold uppercase letters to denote matrices (e.g., $\bA$). For an index set $S$, $[\bA]_{S,:}$ and $[\bA]_{:,S}$ denote the submatrices of $\bA$ consisting of the rows and columns indexed by $S$, respectively, and $[\bd]_S$ denotes the subvector of $\bd$ consisting of the elements indexed by $S$. We use $\mathrm{row}(\{\bA_1,\dots,\bA_p\})$ to denote the row space jointly spanned by the rows of the matrices $\bA_1,\dots,\bA_p$, and use $\mathrm{rank}(\bA)$ to denote the rank of the matrix $\bA$. We use $\langle \ba_i| i = 1,\dots,p \rangle$ to denote the space spanned by $\{\ba_i \}_{i=1}^p$. For a positive integer $p$, let $[p] \triangleq \{1,\dots,p\}$.

\subsection{Problem formulation: fragmented multi-client CRL}

\textbf{Global latent SEM.}
Let $Z=(Z_{(1)},\dots,Z_{(p)})$ denote the global latent variables, generated by the linear structural equation model
\begin{equation}
\label{eq:gloabl latent}
    Z = \bA Z + E.
\end{equation}
Since $Z$ is latent, we may assume without loss of generality that $Z_j \to Z_i$ implies $j>i$. Under this ordering, $\bA\in\bbR^{p\times p}$ is strictly upper triangular. The noise vector $E$ has mean zero and diagonal covariance matrix $\bSigma\triangleq Cov(E)$.
The matrix $\bA$ defines a global latent DAG $\cG=(\cV,\cE)$ with $\cV=[p]$ and $\cE=\{j\to i:\ A_{ij}\neq 0\}$. It induces the global ancestral relation $\prec$ on $\cV$: we write $i\prec j$ if there exists a directed path from $j$ to $i$ in $\cG$, and write $i\preceq j$ if either $i\prec j$ or $i=j$.

\textbf{Fragmented clients.}
We consider $K$ clients. For client $k\in[K]$, the index set $O_k\subset[p]$ specifies the latent variables directly used by client $k$, and $M_k=[p]\setminus O_k$ denotes the latent variables not directly used by this client. Let
$Z_k\triangleq (Z_{(i)})_{i\in O_k}\in\bbR^{|O_k|},
p_k\triangleq |O_k|$. The client-specific observation is
\begin{equation}
\label{eq:X = GZ}
X_k = \bG_k Z_k,
\end{equation}
where $\bG_k\in\bbR^{n_k\times p_k}$ is full column rank with $n_k>p_k$. We write $\bH_k=\bG_k^\dagger$ for its Moore--Penrose pseudoinverse. Thus all clients share the same global latent system, but each client directly uses only a subset of the latent variables.

\textbf{Soft interventions.}
For each client $k$, we observe one observational regime indexed by $0$ and multiple soft interventional regimes indexed by $j\in\cJ_k\subset[p]$. We assume that client $k$ only intervenes on nodes in $O_k$, and that each intervention targets a single node. Let $t_j$ denote the unknown target node of intervention $j$. We model the corresponding intervention on the global latent system by
\begin{equation*}
\begin{aligned}
    \Tilde{\bA}_k^{(j)} = \bA + \be_{t_j} {\Delta\ba_k^{(j)}}^T, \qquad 
    \Tilde{\bSigma}_k^{(j)} = \bSigma + \Delta\sigma_k^{(j)} \be_{t_j}\be_{t_j}^T,
\end{aligned}
\end{equation*}
where $\be_{t_j}$ is the vector whose $t_j$-th entry is $1$ and all others are $0$, ${\Delta\ba_k^{(j)}}\in\bbR^p$ is supported on $pa(t_j)$ and represents the intervention-induced perturbation of the parent effects of node $t_j$, and $\Delta\sigma_k^{(j)}\in\bbR$ perturbs its exogenous noise variance. We write $X_k^{(j)}$ for the observation of client $k$ under intervention $j$.
Throughout the paper, we work with the second-order information
$\bTheta_k^{(j)}\triangleq \bbE(X_k^{(j)}{X_k^{(j)}}^\top)^\dagger$ for $j\in\cJ_k\cup\{0\}$, where $X_k^{(0)}\triangleq X_k$ denotes the observational regime.

\textbf{Recovery goal.}
Given the collection of precision matrices
$\{\bTheta_k^{(j)}\}_{k=1,\dots,K,\;j\in\cJ_k\cup\{0\}}$, the fragmented multi-client CRL problem has two recovery levels. At the client level, we first characterize and recover the structural information that remains identifiable from each client's fragmented access to the global latent system. At the global level, we assemble these client-specific fragments across $k=1,\ldots,K$ to recover the global latent ancestral relation $\prec$.

\subsection{Client-specific induced graph and block structure}

Since client $k$ directly uses only $Z_k=Z_{O_k}$, the variables in $M_k=[p]\setminus O_k$ are marginalized out from this client's perspective. 
This marginalization has two effects. First, directed paths that pass through variables in $M_k$ can appear as directed shortcuts among variables in $O_k$. 
Second, shared dependence through the marginalized variables can induce correlated effective noises, represented as bidirected edges. After the normalization discussed in Appendix \ref{app:normalize_sigma}, we work with $\bSigma=\bI$.
Thus the client-level latent object is generally a mixed graph rather than a DAG.
\begin{proposition}
\label{prop:margin Z}
For $Z$ generated by Eq.~\eqref{eq:gloabl latent}, let $Z_k=(Z_{(i)})_{i\in O_k}$. Then
\begin{equation}
Z_k = \bA_k Z_k + E_k,
\end{equation}
where
\begin{align}
\label{eq:Ak}
    \bA_k
    &=
    [\bA]_{O_k,O_k}
    +
    [\bA]_{O_k,M_k}
    (\bI-[\bA]_{M_k,M_k})^{-1}
    [\bA]_{M_k,O_k}, \\
    E_k
    &=
    E_{O_k}
    +
    [\bA]_{O_k,M_k}
    (\bI-[\bA]_{M_k,M_k})^{-1}
    E_{M_k}.
\end{align}
Moreover, $E_k$ has mean zero and covariance
\begin{equation}
\label{eq:Sigmak}
    \bSigma_k
    =
    [\bSigma]_{O_k,O_k}
    +
    [\bA]_{O_k,M_k}
    (\bI-[\bA]_{M_k,M_k})^{-1}
    [\bSigma]_{M_k,M_k}
    (\bI-[\bA]_{M_k,M_k})^{-T}
    [\bA]_{O_k,M_k}^{T}.
\end{equation}
\end{proposition}
Proposition~\ref{prop:margin Z} shows that after marginalizing out $M_k$, client $k$ no longer inherits a latent DAG on $O_k$. Instead, it induces a mixed graph $\cG_k=(\cV_k,\cE_k,\cD_k)$ with $\cV_k=O_k$, directed edges
$\cE_k=\{j\to i:\ [\bA_k]_{\rho_k(i),\rho_k(j)}\neq 0\}$, and bidirected edges
$\cD_k=\{j\leftrightarrow i:i\neq j,\ [\bSigma_k]_{\rho_k(i),\rho_k(j)}\neq 0\}$, where $\rho_k$ denotes the local index map from nodes in $O_k$ to coordinates in $Z_k$.
This mixed graph induces a client-specific ancestral relation $\prec_k$ on $\cV_k$: we write $i\prec_k j$ if there exists a directed path from $j$ to $i$ in $\cG_k$, and $i\preceq_k j$ if either $i\prec_k j$ or $i=j$. By the proof of Proposition~\ref{prop:margin Z} in Appendix~\ref{app:proof Prop matgin Z}, $\prec_k$ is exactly the restriction of the global ancestral relation $\prec$ to $O_k$.

For notational consistency, set $\bA_k^{(0)}\triangleq \bA_k$ and $\bSigma_k^{(0)}\triangleq \bSigma_k$. For $j\in\cJ_k$, marginalizing out $M_k$ under intervention $j$ yields $\bA_k^{(j)}$ and $\bSigma_k^{(j)}$ by replacing $\bA$ and $\bSigma$ in Eqs.~\eqref{eq:Ak} and~\eqref{eq:Sigmak} with $\Tilde{\bA}_k^{(j)}$ and $\Tilde{\bSigma}_k^{(j)}$, respectively. Then, by Proposition~\ref{prop:margin Z} and Lemma~\ref{lem:HKH}, for all $j\in\cJ_k\cup\{0\}$,
\begin{equation}
\label{eq:interventional_precision}
    \bTheta_k^{(j)}
    =
    \bH_k^\top (\bI-\bA_k^{(j)}){\bSigma_k^{(j)}}^{-1}(\bI-\bA_k^{(j)})^\top \bH_k .
\end{equation}

The mixed graph $\cG_k$ contains both directed ancestral relations and bidirected dependencies induced by latent marginalization. Consequently, node-wise structure is no longer the right unit of description at the client level. We therefore introduce a client-specific block structure.

\begin{definition}[Reachability and mutual reachability]
For two nodes $i,j \in \cV_k$, we say that $i$ is \emph{reachable} to $j$, written $i \rightsquigarrow_k j$, if there exists a sequence of nodes
$i = v_0, v_1, \dots, v_m = j$, 
such that for every $\ell = 0,\dots,m-1$, either $v_\ell \prec_k v_{\ell+1}$ or $v_\ell \leftrightarrow v_{\ell+1} \in \cD_k$.
Two nodes $i,j \in \cV_k$ are said to be \emph{mutually reachable} if $i \rightsquigarrow_k j$ and $j \rightsquigarrow_k i$.
\end{definition}

\begin{definition}[Block]
For client $k$, a \emph{block} is a maximal subset $C \subseteq \cV_k$ such that every pair of nodes in $C$ is mutually reachable.
\end{definition}

Since mutual reachability is an equivalence relation, all blocks form a partition
$\cC_k = \{C_{k,1},\dots,C_{k,m_k}\}$ of $\cV_k$. For example, in Figure~\ref{fig:intro}, the block partition for client $k$ is
$\cC_k = \big\{\{1,2\},\{5\},\{6\},\{7\}\big\}$.

Intuitively, nodes within the same block are structurally entangled by directed ancestry and bidirected confounding, while a meaningful ancestral order remains well defined across blocks.

\begin{definition}[Block-level order]
For two distinct blocks $C_{k,a}, C_{k,b} \in \cC_k$, write
$C_{k,a} \prec_k^{\rblk} C_{k,b}$ if there exist $u \in C_{k,a}$ and $v \in C_{k,b}$ such that $u \prec_k v$.
Also write $C_{k,a} \preceq_k^{\rblk} C_{k,b}$ if either
$C_{k,a} \prec_k^{\rblk} C_{k,b}$ or $C_{k,a}=C_{k,b}$.
A block $C_{k,b}$ is a \emph{block-level parent} of $C_{k,a}$ if
$C_{k,a} \prec_k^{\rblk} C_{k,b}$ and there is no block $C_{k,m}\in\cC_k$ such that
$C_{k,a} \prec_k^{\rblk} C_{k,m} \prec_k^{\rblk} C_{k,b}$.
\end{definition}

\textbf{Ancestor set notation.}
For the global graph, define $anc(i)\triangleq \{j\in\cV:\ i\prec j\}$ and $pa(i)\triangleq \{j\in\cV:\ j\to i\in\cE\}$, together with
$Anc(i)\triangleq anc(i)\cup\{i\}$ and $Pa(i)\triangleq pa(i)\cup\{i\}$.
For client $k$ and node $i\in\cV_k$, define
$anc_k(i)\triangleq \{j\in\cV_k:\ i\prec_k j\}$,
$pa_k(i)\triangleq \{j\in\cV_k:\ j\to i\in\cE_k\}$,
$Anc_k(i)\triangleq anc_k(i)\cup\{i\}$, and
$Pa_k(i)\triangleq pa_k(i)\cup\{i\}$.
We further define the spouse set
$sp_k(i)\triangleq \{j\in\cV_k:\ j\leftrightarrow i\in\cD_k\}$.
For any block $C\in\cC_k$, define
$anc_k^{\rblk}(C)\triangleq \{C'\in\cC_k:\ C \prec_k^{\rblk} C'\}$,
$pa_k^{\rblk}(C)\triangleq \{C'\in\cC_k:\ C' \text{ is a block-level parent of } C\}$,
$Anc_k^{\rblk}(C)\triangleq anc_k^{\rblk}(C)\cup\{C\}$, and
$Pa_k^{\rblk}(C)\triangleq pa_k^{\rblk}(C)\cup\{C\}$.
For a node $v\in\cV_k$, let $C_k(v)$ denote the block containing $v$.

The next proposition clarifies why block ancestry is the right coarse structural object: across blocks, non-ancestral relations remain absent, while block-level parents still witness genuine node-level parent relations.

\begin{proposition}
\label{prop:block anc}
For each client $k$, let $C_{k,i},C_{k,j}\in\cC_k$ be two blocks.  
\begin{enumerate}[label=(\roman*), leftmargin=2.4em, itemsep=0pt, topsep=1pt]
    \item If $C_{k,j}\notin anc_k^{\rblk}(C_{k,i})$, then for $\forall u\in C_{k,j}$ and $v\in C_{k,i}$, we have $u\notin anc_k(v)$.
    \item If $C_{k,j}\in pa_k^{\rblk}(C_{k,i})$, then there $\exists u\in C_{k,j}$ and $v\in C_{k,i}$ such that $u\in pa_k(v)$.
\end{enumerate}
\end{proposition}

For example, in client $k=1$ of Figure~\ref{fig:intro1}, $\{7\}\in pa_k^{\rblk}(\{1,2\})$, $\{7\}\in anc_k^{\rblk}(\{5\})$, but $\{7\}\notin pa_k^{\rblk}(\{5\})$.
Without loss of generality, for each client $k$, we assume that if $C_{k,i} \prec_k^{\rblk} C_{k,j}$, then $\rho_k(v) < \rho_k(u)$ for all $v \in C_{k,i}$ and $u \in C_{k,j}$.

\section{Client-level Identifiability from Low-rank Precision Perturbations}
\label{sec:client}


Section~\ref{sec:setup} shows that latent marginalization changes each client's local object from a node-wise DAG to a block-structured mixed graph. 
This section answers the client-level question raised in Section~\ref{sec:intro}: \textbf{when missing latent variables induce bidirected dependencies through marginalization, what causal structure remains identifiable to a fragmented client?}
We show that a single client can identify its block partition, block-level ancestral order, and associated latent subspaces up to intrinsic block-level ambiguity. 
The argument combines a block-adapted coordinate system, low-rank precision perturbations under soft interventions, and an iterative recovery rule based on projected row-space dimensions.

\subsection{Client-level assumptions}
We first state the assumptions for client-level identifiability. 

\begin{assumption}[Intervention indexing and coverage]
\label{ass: aligned intervention}
We assume: (i) \emph{client-level intervention coverage:} for each client $k$, $\cJ_k$ is in bijection with $\cV_k$, so each node in $\cV_k$ is targeted by exactly one intervention label; and (ii) \emph{cross-client label alignment:} interventions from different clients that target the same global latent node share the same label.
\end{assumption}

Let $\psi:j\mapsto t_j$ denote the mapping from the global intervention index to its target node, so that $\psi$ is a permutation on $[p]$.

\begin{assumption}[Intervention genericity]
\label{ass:non-degeneracy}
For each client $k$ and intervention $j\in\cJ_k$ with target $t_j\in C_{k,\ell}\subseteq O_k$, define
\begin{align*}
\bd_k^{(j)}
&\triangleq [\Delta\ba_k^{(j)}]_{O_k}
+ [\bA]_{M_k,O_k}^T(\bI-[\bA]_{M_k,M_k})^{-T}[\Delta\ba_k^{(j)}]_{M_k},\\
\bu_k^{(j)}
&\triangleq
[\bA]_{O_k,M_k}(\bI-[\bA]_{M_k,M_k})^{-1}
[\bSigma]_{M_k,M_k}(\bI-[\bA]_{M_k,M_k})^{-T}
[\Delta\ba_k^{(j)}]_{M_k}.
\end{align*}
We assume that the intervention parameters are chosen such that:
\begin{enumerate}[label=(\roman*), leftmargin=2.4em, itemsep=0pt, topsep=1pt]
    \item $[\bd_k^{(j)}]_{\rho_k(p)}\neq 0$ for every $p\in pa_k(t_j)$;
    \item $[\bu_k^{(j)}]_{\rho_k(s)}\neq 0$ for every $s\in sp_k(t_j)$;
    \item for any $\mathcal B\subsetneq C_{k,\ell}$ and any $i\in\cB$, let $\mathcal N_i$ be the support of
    \[
    [\bu_k^{(\psi^{-1}(i))}]_{\rho_k(C_{k,\ell})}
    +[\bSigma_k]_{\rho_k(C_{k,\ell}),\rho_k(C_{k,\ell})}
    \bigl(\bI-[\bA_k]_{\rho_k(C_{k,\ell}),\rho_k(C_{k,\ell})}^T\bigr)^{-1}
    [\bd_k^{(\psi^{-1}(i))}]_{\rho_k(C_{k,\ell})}.
    \]
    Then $\bigl|\bigcup_{i\in\mathcal B}\mathcal N_i\bigr|>|\mathcal B|$.
\end{enumerate}
\end{assumption}
Assumption~\ref{ass: aligned intervention} ensures that intervention environments are sufficiently informative and consistently indexed across clients. Assumption~\ref{ass:non-degeneracy} is a genericity condition that rules out pathological cancellations in the directed, bidirected, and within-block effects induced by soft interventions. The two assumptions are standard in spirit, and closely related conditions appear frequently in intervention-based latent causal discovery and CRL \citep{varici2025score,buchholz2023learning,squires2023linear}. 
More detailed discussion and motivation for these assumptions are provided in Appendix~\ref{app:assumptions}.

\subsection{Analytical coordinates: block RQ decomposition}
\label{sec:block RQ}
To relate observable precision perturbations to the client-specific block structure, we introduce a block-adapted factorization of $\bH_k$. This factorization is not computed by the learner, since $\bH_k$ is unknown. Rather, it provides a coordinate system in which row spaces can be described according to the block ancestral order. Its construction and further discussion are deferred to Appendix \ref{app:block RQ}.

\begin{definition}[Block RQ decomposition]
\label{def:RQ}
For client $k\in[K]$ with block ancestor structure $(\prec_k^\rblk,\cC_k)$, a block RQ decomposition of $\bH_k\in\bbR^{p_k\times n_k}$ is a factorization
\[
\bH_k=\bR_k\bQ_k,
\]
where $\bR_k\in\bbR^{p_k\times p_k}$ is block upper triangular with respect to $\cC_k$, namely
$[\bR_k]_{\rho_k(C_{k,i}),\rho_k(C_{k,j})}=0$  unless $C_{k,i}\preceq_k^\rblk C_{k,j}$,
and $\bQ_k\in\bbR^{p_k\times n_k}$ is row full-rank with unit-norm rows. Moreover, for every block $C\in\cC_k$ and every $i\in\rho_k(C)$, the row $[\bQ_k]_{i,:}$ is orthogonal to every row of $[\bQ_k]_{\bigcup_{C'\in anc_k^\rblk(C)}\rho_k(C'),:}$.
\end{definition}

Although this decomposition is generally not unique, any such factorization is sufficient for our analysis. Intuitively, $\bR_k$ records block-ancestor mixing, while $\bQ_k$ isolates block-associated row directions.

\begin{proposition}
\label{prop: RQ}
For every block $C\in\cC_k$ and every $i\in\rho_k(C)$, we have
$[\bH_k]_{i,:}\in
\mathrm{row}(
[\bQ_k]_{\bigcup_{C'\in Anc_k^\rblk(C)}\rho_k(C'),:})$.
\end{proposition}

Thus, the rows of $\bH_k$ are organized according to the block ancestral relation in the $\bQ_k$ coordinates, which will allow us to interpret observable precision perturbations in terms of block-ancestor subspaces.

\subsection{Low-rank perturbation of the precision matrix}

Client-level recovery is driven not by a single precision matrix, but by how the precision matrix changes across intervention environments. We therefore study the perturbation $\bTheta_k^{(j)}-\bTheta_k^{(0)}$. The next two lemmas show that such perturbations are low-rank and that their row spaces are confined to block-ancestor subspaces.

\begin{lemma}[Rank-$2$ perturbation]
\label{lem: rank 2}
For each client $k$ and each intervention $j$ targeting node $t_j$ (which is unknown to the learner), we have
\begin{equation}
\mathrm{row}(\bTheta_k^{(j)}-\bTheta_k^{(0)})
\subseteq
\left\langle
\bH_k^T(\bI-\bA_k)^T\bSigma_k^{-1}\be_{\rho_k(t_j)},
\;
\bH_k^T\!\left((\bI-\bA_k)^T\bSigma_k^{-1}\bu_k^{(j)}+\bd_k^{(j)}\right)
\right\rangle,
\end{equation}
where $\bu_k^{(j)}$ and $\bd_k^{(j)}$ are defined as in Assumption~\ref{ass:non-degeneracy}. In particular, $\mathrm{rank}(\bTheta_k^{(j)}-\bTheta_k^{(0)})\le 2$.
\end{lemma}

Thus, a single-node soft intervention affects only a two-dimensional row space, despite producing a $p_k\times p_k$ perturbation matrix. To use this fact for recovery, we must further locate this row space relative to the client-specific block structure.

\begin{lemma}[Row-space characterization via $\bQ_k$]
\label{lem: low rank Qk}
For each client $k$ and intervention $j$, suppose that the intervention target node $t_j$ belongs to block $C_{k,a}\in\cC_k$. Then
\[\mathrm{row}(\bTheta_k^{(j)}-\bTheta_k^{(0)})
\subseteq
\left\langle
[\bQ_k]_{\rho_k(i),:}
\;\middle|\;
i\in C_{k,b},\; C_{k,b}\in Anc_k^\rblk(C_{k,a})
\right\rangle.\]
\end{lemma}

Lemma \ref{lem: low rank Qk} shows that if the intervention target lies in block $C_{k,a}$, then the resulting perturbation only involves row directions associated with $C_{k,a}$ and its block ancestors. In this sense, soft interventions reveal ancestor-closed portions of the client-specific block structure. For example, in Figure \ref{fig:intro1}, if intervention $j$ for client $k=1$ targets node $t_j=2$, then $\mathrm{row}(\bTheta_k^{(j)}-\bTheta_k^{(0)})
\subseteq
\langle [\bQ_k]_{i,:}\mid i=\rho_k(1),\rho_k(2),\rho_k(7)\rangle$.
Combined with the orthogonality relations in Definition~\ref{def:RQ}, this implies that after projecting out the directions of recovered ancestor blocks, the remaining perturbation isolates the next unrecovered block.

\newsavebox{\algonebox}
\newsavebox{\algtwobox}
\newsavebox{\algthreebox}

\begin{lrbox}{\algonebox}
\begin{minipage}[t]{0.49\textwidth}
\captionsetup{type=algorithm,aboveskip=2pt,belowskip=2pt}
\hrule height 0.8pt
\vspace{2pt}
\caption{ID-Block-Ancestor}
\label{alg:block ancestor}
\vspace{2pt}
\hrule height 0.4pt
\vspace{2pt}
\footnotesize
\begin{algorithmic}[1]
\Require $\hat C\subseteq \cJ_k$, $\{\Theta_k^{(j)}\}_{j\in\{0\}\cup \hat C}$, $\mathcal I$, $\{\hat{\bq}_{k,i}\}_{i\in\mathcal I}$
\Ensure $\{\hat{\bq}_{k,i}\}_{i\in\hat C}$, $\mathcal A$
\State $\mathcal A \gets \mathcal I$
\For{each block $B\in\mathcal I$}
    \State $W_{-B}\gets \langle \hat{\bq}_{k,i}: i\in \mathcal I\backslash B\rangle$
    \State $V_{-B}\gets \cP_k(\hat C;W_{-B})$
    \If{$\dim(V_{-B})=|\hat C|$}
        \State $\mathcal A\gets \mathcal A\backslash B$
    \EndIf
\EndFor
\State $W\gets \langle \hat{\bq}_{k,i}: i\in \mathcal A\rangle$; $V\gets \cP_k(\hat C;W)$
\State Construct normalized vectors $\{\hat{\bq}_{k,i}\}_{i\in\hat C}$ such that
\Statex \hspace{\algorithmicindent} $V=\langle \hat{\bq}_{k,i}: i\in\hat C\rangle$
\State \Return $\{\hat{\bq}_{k,i}\}_{i\in\hat C},\mathcal A$
\end{algorithmic}
\vspace{2pt}
\hrule height 0.4pt
\end{minipage}
\end{lrbox}

\begin{lrbox}{\algtwobox}
\begin{minipage}[t]{0.49\textwidth}
\captionsetup{type=algorithm,aboveskip=2pt,belowskip=2pt}
\hrule height 0.8pt
\vspace{2pt}
\caption{ID-Block-Causal-Order}
\label{alg:causal order client}
\vspace{2pt}
\hrule height 0.4pt
\vspace{2pt}
\footnotesize
\begin{algorithmic}[1]
\Require $\{\Theta_k^{(j)}\}_{j\in\{0\}\cup\cJ_k}$, $p_k$
\Ensure $\hat{\mathcal C}_k$, $\hat{\prec}_k^\rblk$, $\hat{\bQ}_k$
\State $\mathcal I_0\gets\emptyset$, $R_0\gets\cJ_k$, $\hat{\bQ}_k\gets \mathbf 0_{p_k\times n_k}$, $t\gets 1$
\While{$R_{t-1}\neq\emptyset$}
    \State $r\gets 1$, $\mathrm{flag}\gets \mathrm{True}$
    \While{$\mathrm{flag}$}
        \State $W_t\gets \langle \hat{\bq}_{k,i}: i\in\mathcal I_{t-1}\rangle$
        \State $\mathcal R\gets\{C\subseteq R_{t-1}: |C|=r\}$
        \If{there exists $\hat C\in\mathcal R$ such that $\dim(\cP_k(\hat C;W_t))=r$}
            \State Pick one such $\hat C$,  $\mathrm{flag}\gets \mathrm{False}$
        \Else
            \State $r\gets r+1$
        \EndIf
    \EndWhile
    \State $(\{\hat{\bq}_{k,i}\}_{i\in\hat C},\mathcal A)\gets \mathrm{ID\text{-}Block\text{-}Ancestor}(\hat C,$
    \Statex \hspace{\algorithmicindent}\qquad $\{\Theta_k^{(j)}\}_{j\in\{0\}\cup\hat C},\mathcal I_{t-1},\{\hat{\bq}_{k,i}\}_{i\in\mathcal I_{t-1}})$
    \State Add $\hat C \,\hat{\prec}_k^\rblk\, C'$ for any 
     $\tilde C \,\hat{\preceq}_k^\rblk\, C'$ with $\tilde C\in\mathcal A$.
    \State Update $\hat{\bQ}_k\gets [\tilde{\bQ}_{k,\hat C};\hat{\bQ}_k]$, where $\tilde{\bQ}_{k,\hat C}$ stacks $\{\hat{\bq}_{k,i}\}_{i\in\hat C}$.
    \State $\mathcal I_t\gets \{\hat C\}\cup \mathcal I_{t-1}$
    \State $R_t\gets R_{t-1}\backslash \hat C$
    \State  $t\gets t+1$
\EndWhile
\State $\hat{\mathcal C}_k\gets \mathcal I_{t-1}$
\State \Return $\hat{\mathcal C}_k,\hat{\prec}_k^\rblk,\hat{\bQ}_k$
\end{algorithmic}
\vspace{2pt}
\hrule height 0.4pt
\end{minipage}
\end{lrbox}

\begin{lrbox}{\algthreebox}
\begin{minipage}[t]{0.49\textwidth}
\captionsetup{type=algorithm,aboveskip=2pt,belowskip=2pt}
\hrule height 0.8pt
\vspace{2pt}
\caption{ID-Global-Causal-Order}
\label{alg:global}
\vspace{2pt}
\hrule height 0.4pt
\vspace{2pt}
\footnotesize
\begin{algorithmic}[1]
\Require $\{\hat{\cC}_k,\hat{\prec}_k^\rblk\}_{k=1}^K$
\Ensure $\hat{\prec}$
\State Initialize $\tilde{\prec}\gets\emptyset$
\For{$k=1,2,\dots,K$}
    \State Find all pairs of blocks $C_{k,i},C_{k,j}\in\hat{\cC}_k$ such that $|C_{k,i}|=|C_{k,j}|=1$ and $C_{k,j}\in \hat{pa}_k^\rblk(C_{k,i})$
    \State For each such pair $C_{k,i}=\{i\}$ and $C_{k,j}=\{j\}$, add $i\,\tilde{\prec}\,j$
\EndFor
\State $\hat{\prec}\gets$ the transitive closure of $\tilde{\prec}$
\State \Return $\hat{\prec}$
\end{algorithmic}
\vspace{2pt}
\hrule height 0.4pt
\end{minipage}
\end{lrbox}

\begin{figure}[h]
\centering

\begin{minipage}[t]{0.51\textwidth}
\usebox{\algonebox}

\vspace{0.5em}

\usebox{\algthreebox}
\end{minipage}
\hfill
\begin{minipage}[t]{0.48\textwidth}
\usebox{\algtwobox}
\end{minipage}

\end{figure}

\subsection{Client-level recovery algorithms and identifiability result}

The preceding lemmas yield the population recovery rule behind Algorithms~\ref{alg:block ancestor} and~\ref{alg:causal order client}. 
For a label set $C\subseteq \cJ_k$ and a row subspace $W$, define
$\cP_k(C;W)\triangleq \Pi_{W^\perp}\!\left(
\operatorname{row}\{\Theta_k^{(j)}-\Theta_k^{(0)}:j\in C\}
\right)$.
Once the row spaces of ancestor blocks are included in $W$, the intervention labels of the next block are exactly the smallest set $C$ satisfying $\dim \cP_k(C;W)=|C|$. 
Algorithm~\ref{alg:causal order client} iteratively applies this rule to recover blocks, while Algorithm~\ref{alg:block ancestor} determines which previously recovered blocks are ancestors of the newly found block.

Combining Lemmas~\ref{lem: rank 2} and~\ref{lem: low rank Qk} with the orthogonality structure in Definition~\ref{def:RQ} yields the following client-level identifiability result.

\begin{theorem}[Client-level recovery]
\label{thm:client-main}
Assume Assumptions~\ref{ass: aligned intervention} and~\ref{ass:non-degeneracy}. Let $(\hat{\mathcal C}_k,\hat{\prec}_k^\rblk,\hat{\bQ}_k)$ be the output of Algorithm~\ref{alg:causal order client} for client $k$. Then:
\begin{enumerate}[label=(\roman*), leftmargin=2.4em, itemsep=0pt, topsep=1pt]
    \item $(\hat{\mathcal C}_k,\hat{\prec}_k^\rblk)$ recovers $(\mathcal C_k,\prec_k^\rblk)$ up to relabeling under the bijection $\psi:j\mapsto t_j$.
    \item For each true block $C\in\mathcal C_k$, let $\hat C=\psi^{-1}(C)$ denote the corresponding estimated block. Then
    $\langle \hat{\bq}_{k,i}: i\in\hat C\rangle
    =
    \mathrm{row}([\bQ_k]_{\rho_k(C),:}).$
\end{enumerate}
Consequently, $\hat{\bQ}_k=\bS_k\bP_{\sigma,k}\bQ_k$,
where $\bS_k$ is an invertible block-diagonal norm-preserving matrix and $\bP_{\sigma,k}$ is a row permutation matrix preserving the block ancestral order.
\end{theorem}

Theorem~\ref{thm:client-main} shows that a single fragmented client cannot generally recover a node-wise latent DAG, but can exactly recover the client-specific block structure, block-level ancestral relation, and row subspaces of $\bQ_k$. 
The relabeling by $\psi:j\mapsto t_j$ only changes the names of recovered latent targets. 
The additional ambiguity, captured by $\bS_k$ and $\bP_{\sigma,k}$, is intrinsic to the non-uniqueness of the block RQ decomposition: it only changes the within-block basis and order-preserving block labels, while preserving the identifiable block-associated row subspaces and causal structure.


Applying the recovered rows of $\hat{\bQ}_k$ in Theorem \ref{thm:client-main} to $X_k$ therefore yields the corresponding client-specific latent variables, up to the intrinsic ambiguity induced by block-level ancestors, as follows.

\begin{corollary}[Recovery of $Z_k$ up to block-level ancestral ambiguity]
\label{cor:zk}
Assume Assumptions~\ref{ass: aligned intervention} and~\ref{ass:non-degeneracy} hold. Let $\hat{\bQ}_k$ be the output of Algorithm~\ref{alg:causal order client}, and define
\[
\hat Z_k=\hat{\bQ}_kX_k.
\]
For each true block $C\in\cC_k$, let $\hat C=\psi^{-1}(C)$ denote the corresponding estimated block. Then, for every $i\in\hat C$, the estimated latent variable $\hat Z_{k,(i)}$ lies in the span of the true latent variables indexed by the block ancestors of $C$, namely
\[
\hat Z_{k,(i)}\in \mathrm{span}\{Z_{(j)}:j\in C', C' \in Anc_k^\rblk(C)\}.
\]
Thus, up to the relabeling induced by $\psi$, each estimated latent variable is identified up to block-level ancestor mixing.
\end{corollary}

Corollary~\ref{cor:zk} gives the latent-level interpretation of the client-specific recovery result. For a single fragmented client, exact node-wise latent recovery is generally impossible because marginalizing unobserved latent variables induces bidirected dependencies and merges nodes into block-structured objects. Nevertheless, the recovered intervention-induced subspaces identify the latent variables at the finest level permitted by the model, namely up to block-level ancestor mixing.

This completes the client-level part of the identifiability picture: from a single client, one can recover its block structure, the associated block-level ancestral relation, and the corresponding latent subspaces up to block-level ancestral ambiguity. In the next section, we show how such client-specific fragments can be assembled to recover the global latent causal order.

\section{Global Ancestral Recovery Across Clients}
\label{sec:global}

Section~\ref{sec:client} shows that each client recovers a correct block-structured fragment of the global latent system. We now return to the question raised in Section~\ref{sec:intro}: \textbf{can we recover the global latent causal order by assembling fragmented causal information from multiple clients?} This is not possible from a single client in general: marginalization may merge several global latent variables into one client-specific block. 
The key observation is that different clients may induce different blockings. Thus, two latent variables that are entangled in one client's view may be separated and ordered in another client's recovered block structure.

We first formalize the cross-client visibility needed for global assembly. 
\begin{assumption}[Cross-client coverage for global assembly]
\label{ass: global cover edge}
If $j\to i$ is a Hasse edge of the global latent DAG $\cG$, namely $j\prec i$ and there exists no node $m$ such that $j\prec m\prec i$, there exists a client $k\in[K]$ such that (i) $i,j \in O_k$; (ii) $\{i\},\{j\}\in \mathcal C_k$; (iii) $\{j\} \in pa_k^\rblk(\{i\})$.
\end{assumption}

Under Assumption~\ref{ass: global cover edge}, every global Hasse edge is witnessed by at least one client as a parent relation between singleton blocks. 
Thus Algorithm~\ref{alg:global} assembles the global order by collecting all recovered singleton-block parent relations $C_{k,j}\in \hat{pa}_k^\rblk(C_{k,i})$ with $C_{k,i}=\{i\}$ and $C_{k,j}=\{j\}$, and then taking their transitive closure. 
Theorem~\ref{thm:global} shows that this procedure recovers the global latent ancestral relation, up to the intervention-target relabeling $\psi:j\mapsto t_j$.

\begin{theorem}[Correctness of global ancestral recovery]
\label{thm:global}
Assume Assumptions~\ref{ass: aligned intervention}--\ref{ass: global cover edge} hold. Let $\{(\hat{\cC}_k,\hat{\prec}_k^\rblk)\}_{k=1}^K$ be the outputs of Algorithm~\ref{alg:causal order client} and the input to Algorithm~\ref{alg:global}. Let $\hat{\prec}$ be the output of Algorithm~\ref{alg:global}. Then $(\cV,\hat{\prec})$ recovers $(\cV,\prec)$ up to relabeling under the bijection $\psi:j\mapsto t_j$.
\end{theorem}

\begin{remark}
Theorems~\ref{thm:client-main} and~\ref{thm:global}, together with Corollary~\ref{cor:zk}, yield a two-level identifiability picture: a single client recovers only a block-structured fragment and the corresponding latent variables up to block-level ancestor mixing, whereas multiple such fragments can be assembled to recover the global node-level ancestral relation. This local-to-global gap is intrinsic to the fragmented multi-client setting.
\end{remark}

\section{Finite-sample Method and Experiments}
\label{sec:experiments}

The preceding sections establish population-level identifiability from exact precision matrices. 
We now evaluate a finite-sample implementation of the proposed recovery procedure on synthetic fragmented multi-client systems, and include the population case $n=\infty$ to verify the theoretical recovery guarantees.

\textbf{Finite-sample implementation.} In finite samples, each population precision matrix $\bTheta_k^{(j)}$ is replaced by an estimator $\hat{\bTheta}_k^{(j)}$. Since exact row-space dimensions are unstable under estimation noise, we use a thresholded singular-value criterion: for a matrix $M$, its effective row-space dimension is taken as the smallest $r$ such that
$\frac{\sum_{i=1}^{r}\sigma_i(M)}{\sum_i\sigma_i(M)}>1-\tau$,
where $\sigma_i(M)$ is the $i$-th largest singular value and $\tau$ is a tuning parameter. This gives finite-sample versions of Algorithms~\ref{alg:block ancestor} and~\ref{alg:causal order client}; details are deferred to Appendix~\ref{app: expriments}.


\textbf{Setup.} We consider three different global causal network structures $\bA$ together with different client-specific missing patterns $\{O_k\}_{k=1}^K$, representing different network scales and sparsity levels. We consider two noise distributions, Gaussian and Laplace. For each graph setting and noise setting, the full specifications are deferred to Appendix \ref{app: expriments}. We report results for sample sizes $n\in\{10^3,\dots,10^6\}$ together with the population case $n=\infty$.

\textbf{Evaluation metrics.} We use three metrics. For ease of exposition, we first align the row ordering of $\hat{\bQ}_k$ with that of $\bQ_k$ according to the permutation induced by $\bP_{\sigma,k}$; this does not affect the evaluation itself. The first metric, $E_1$, is the F1-score between the true global partial order $\prec$ and the estimate $\hat{\prec}$. The second metric, $E_2$, measures the average block-level row-subspace error between the true and estimated rows of $\bQ_k$. The third metric, $E_3$, measures support leakage in the estimated latent variables $\hat{Z}_k=\hat{\bQ}_kX_k$, namely the fraction of nonzero coefficients appearing outside the block-level ancestral support allowed by Corollary \ref{cor:zk}. All three metrics take values in $[0,1]$, where larger $E_1$ indicates better performance, while smaller $E_2$ and $E_3$ indicate more accurate recovery. Complete formulas are given in Appendix \ref{app: expriments}.

\textbf{Results.} For the six scenarios induced by three graph settings and two noise settings, Figure \ref{fig:setA_gau} and Figures \ref{fig:setA_lap}-\ref{fig:setC_lap} in Appendix \ref{app: expriments} plot $E_1,E_2,E_3$ against the sample size $n$. Across all scenarios, performance improves as $n$ increases: the estimated global ancestor relation becomes more accurate, the client-specific subspace error decreases, and the support leakage becomes smaller. In the population case $n=\infty$, where the estimated precision matrices are replaced by the corresponding population precision matrices, we obtain $E_1=1$ and $E_2=E_3=0$. This exactly matches Theorems \ref{thm:client-main} and \ref{thm:global} together with Corollary \ref{cor:zk}, and confirms that the finite-sample implementation converges to the population recovery guaranteed by the theory. The same qualitative behavior is observed across all three graph settings and under both Gaussian and Laplace noise, indicating that the method is robust to variations in graph scale, sparsity, and noise distribution.

\begin{figure}[h]
    \centering
    \begin{subfigure}[h]{0.32\linewidth}
        \centering
        \includegraphics[width=\linewidth]{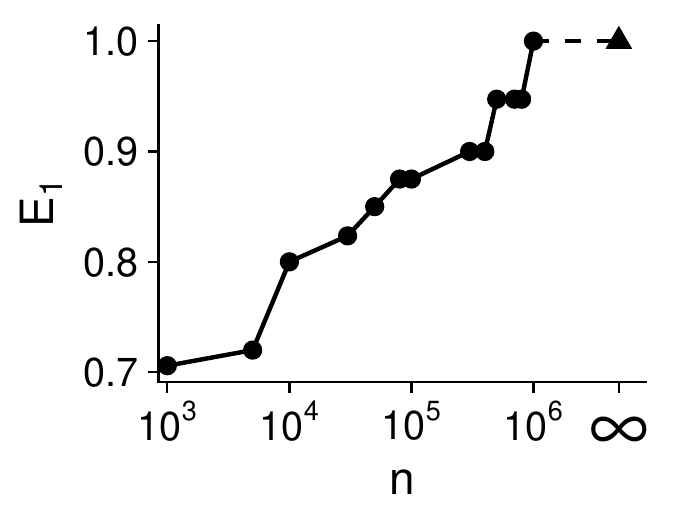}
        \caption{Global ancestral F1.}
        \label{fig:e1}
    \end{subfigure}
    \hfill
    \begin{subfigure}[h]{0.32\linewidth}
        \centering
        \includegraphics[width=\linewidth]{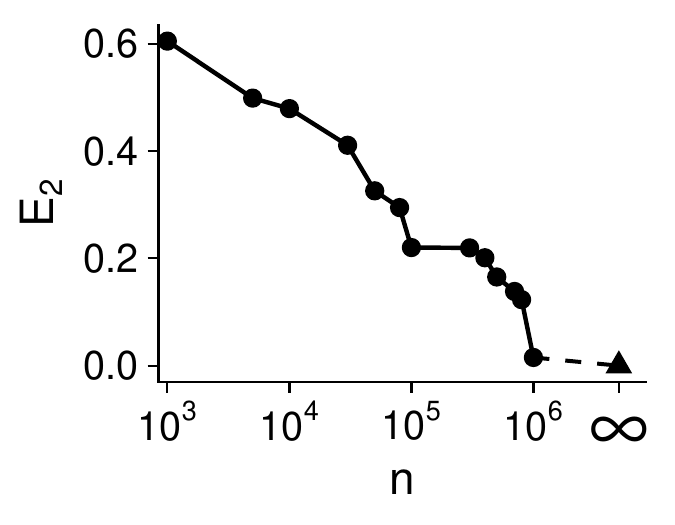}
        \caption{Subspace error of $\{\hat{\bQ}_k\}_{k=1}^K$.}
        \label{fig:e2}
    \end{subfigure}
    \hfill
    \begin{subfigure}[h]{0.32\linewidth}
        \centering
        \includegraphics[width=\linewidth]{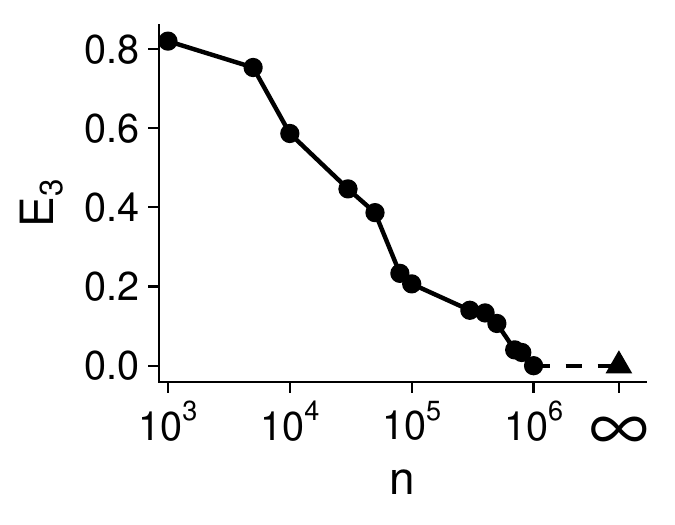}
        \caption{Support error of $\{\hat{Z}_k\}_{k=1}^K$.}
        \label{fig:e3}
    \end{subfigure}
    \caption{Finite-sample performance under one graph setting A and Gaussian noise. As $n$ increases, global ancestral recovery improves while subspace and support errors decrease. The rightmost marker at $\infty$ corresponds to the population-level precision-matrix input.}
    \label{fig:setA_gau}
\end{figure}

\section{Conclusion}
\label{sec:conslusion}
We study causal representation learning in a fragmented multi-client setting, where clients share a global latent causal system but each accesses and intervenes on only a subset of latent variables. We show that latent marginalization changes the client-level identifiable object from a node-wise latent DAG to a block structure, a block-level ancestral relation, and latent subspaces up to block-level ancestor mixing. We further show that these client-specific fragments can be assembled across clients to recover the global node-level latent causal order. Future directions include finite-sample guarantees, extensions to nonlinear models and more general observation mappings, and stronger identifiability under hard interventions, potentially including exact recovery of $\bA$ and $\bH$.

\bibliographystyle{plainnat}
\bibliography{ref.bib}

\appendix
\renewcommand{\thefigure}{A.\arabic{figure}}
\renewcommand{\thetheorem}{A.\arabic{theorem}}
\renewcommand{\thelemma}{A.\arabic{lemma}}
\renewcommand{\thecorollary}{A.\arabic{corollary}}
\renewcommand{\theproposition}{A.\arabic{proposition}}
\setcounter{figure}{0}
\setcounter{lemma}{0}
\setcounter{theorem}{0}
\setcounter{proposition}{0}
\setcounter{corollary}{0}

\clearpage
\startcontents[appendices]
\section*{Appendix Contents}
\printcontents[appendices]{}{1}{\setcounter{tocdepth}{1}}

\section{Related Work}
\label{app:literature}

\textbf{Intervention-based CRL in a shared latent space.}
A central line of CRL studies identifiability from interventions or multiple environments under a shared latent space \citep{scholkopf2021toward,squires2023linear,buchholz2023learning,varici2024general,varici2025score,zhang2024causal,liu2026identifying,zhang2023identifiability,liu2024identifiable,ahuja2023interventional,ng2025causal,bing2024identifying,von2023nonparametric}. These works differ in intervention regimes, mixing assumptions, and identifiability targets, but they typically assume that all environments are defined over the same latent variables. Consequently, the target remains a node-level latent causal structure, or an ambiguity class thereof, in a single latent space. In contrast, we consider fragmented clients that access and intervene on only subsets of a common global latent causal system. Marginalizing the unused latent variables induces client-specific bidirected dependencies, making the local identifiable object block-structured rather than a node-wise latent DAG.

\textbf{CRL under multi-view or multiple domains.}
A related line studies CRL with multiple views or multiple domains. \citet{yao2024multi} consider multi-view CRL where each view depends on a subset of latent variables and identify shared latent information across views, while \citet{sturma2023unpaired} recover a shared causal representation from unpaired multi-domain marginal distributions. These works are related to ours through partial observability or multiple clients, but they do not address how latent marginalization changes each client's local causal object. In our setting, the unused latent variables induce bidirected dependencies, so the client-level target becomes a block-structured fragment that must be assembled to recover the global latent causal order.

\textbf{Latent confounding and mixed graphs.}
Our client-specific induced graphs are related to causal models with latent confounders, where ancestral graphs and acyclic directed mixed graphs (ADMGs) use bidirected edges to encode dependencies induced by marginalization or unobserved common causes \citep{richardson2002ancestral,zhang2008completeness,tian2002general}. In linear structural equation models, mixed graphs have also been used to study global or generic parameter identifiability \citep{drton2011global,foygel2012half}. These works provide important graphical languages for latent confounding, but they mainly concern observed-variable causal structure or parameter identifiability. In contrast, our mixed graphs arise over client-specific latent variables accessed only through high-dimensional observations, and our target is block-level latent identifiability from interventional precision perturbations rather than observed-variable ADMG discovery.

\textbf{Causal discovery from overlapping variable sets.}
Another related line studies causal discovery from multiple datasets with overlapping variable sets, including methods that integrate locally learned structures or constraints across observational and interventional datasets \citep{danks2008integrating,triantafillou2010learning,triantafillou2015constraint,dhir2020integrating}. These works also assemble information from partial views of a system, but they typically operate on observed variables and target a graph or equivalence class over the union of measured variables. In contrast, our clients observe high-dimensional mixtures of latent variables, and missing latent variables induce bidirected dependencies and block-level ambiguity. Our task is therefore not to merge overlapping observed-variable graphs, but to recover a global latent causal order from block-structured latent fragments.

\section{Normalization of the Noise Covariance}
\label{app:normalize_sigma}

In the main text, we work under the normalized parametrization $\bSigma=\bI$. In this appendix, we explain why this assumption is made without loss of generality.

Suppose there exists an alternative parameterization
\[
\Bigl\{\bA',\bSigma',\{\bH_k'\}_{k\in[K]},\{\Tilde{\bA}^{(j)'},\Tilde{\bSigma}^{(j)'}\}_{j\in[p]}\Bigr\}
\]
with $\bSigma'\neq \bI$. Let
\[
\bD=(\bSigma')^{-1/2},
\]
and define the transformed parameters
\[
\bA=\bD\bA'{\bD}^{-1},
\qquad
\bSigma=\bI,
\qquad
\bH_k=[\bD]_{O_k,O_k}\bH_k',
\]
together with
\[
\Tilde{\bA}^{(j)}=\bD\Tilde{\bA}^{(j)'}\bD^{-1},
\qquad
\Tilde{\bSigma}^{(j)}=\bD\Tilde{\bSigma}^{(j)'}\bD .
\]

Under this transformation, the latent variables are rescaled but the induced second-order information of the observed variables remains unchanged. In particular, the collection of observed distributions
\[
\{X_k^{(j)}\}_{k\in[K],\, j\in\cJ_k\cup\{0\}}
\]
generated by the transformed parameterization is identical to that generated by the original one. Therefore, the two parameterizations are observationally equivalent.

Hence, whenever a model with general noise covariance $\bSigma'$ is observationally feasible, there exists an equivalent reparameterization with identity noise covariance. It is therefore sufficient to analyze the normalized case $\bSigma=\bI$ throughout the paper.

\section{Discussion of Assumptions}
\label{app:assumptions}
\paragraph{Discussion of Assumption \ref{ass: aligned intervention}.} 
Assumption~\ref{ass: aligned intervention} combines a standard within-client coverage requirement with a cross-client label alignment condition. The coverage part requires that, for each client $k$, the intervention set $\cJ_k$ is in bijection with the client-specific node set $\cV_k$, so that every node is targeted by exactly one intervention environment. Such assumptions are standard in single-client CRL and intervention-based latent causal discovery, where sufficiently rich intervention coverage is needed to identify causal order or ancestral relations in the worst case \citep{varici2025score,buchholz2023learning,squires2023linear}. In our setting, this requirement is especially natural because each client only observes a subset of the global latent system, and marginalizing the missing latent variables induces bidirected edges, so that even a single client can typically recover only a block-structured object. The additional cross-client alignment condition ensures that intervention environments referring to the same global latent target share a common label across clients. This does not strengthen within-client recovery, but is needed to place the client-specific structural fragments into a common global coordinate system for subsequent assembly. Under this assumption, we define $\psi:j\mapsto t_j$ as the map from the global intervention index to its target node.

\paragraph{Discussion of Assumption \ref{ass:non-degeneracy}.} Assumption~\ref{ass:non-degeneracy} is a standard genericity condition in intervention-based causal discovery and related CRL work, used to rule out pathological intervention choices that create exceptional cancellations in the structural signal \citep{varici2025score,squires2023linear}. The first two parts correspond to non-degeneracy of the directed and bidirected components, respectively. In view of Lemma \ref{lem:support-du}, the supports of $\bd_k^{(j)}$ and $\bu_k^{(j)}$ are structurally constrained by the underlying graph, and these two conditions require that the entries which can be nonzero are, generically, indeed nonzero. Intuitively, this means that the intervention sufficiently excites all directed and bidirected directions that are structurally reachable, rather than being canceled by special parameter choices; such cancellations can typically be regarded as measure-zero pathologies. The third part imposes a non-degeneracy condition within each block. Combining the first two parts, for any $i\in C_{k,\ell}$, the possible elements of $\mathcal N_i$ lie in $\rho_k\big((sp_k(Anc_k(i))\cup Anc_k(i))\cap C_{k,\ell}\big)$. Hence, for any proper subset $\mathcal B\subsetneq C_{k,\ell}$, the union $\bigcup_{i\in\mathcal B}\mathcal N_i$ typically extends beyond $\mathcal B$ itself. Indeed, by the definition of the block structure, any two nodes in the same block are connected through a combination of ancestral relations and bidirected connectivity, so the set of potentially nonzero coordinates necessarily reaches nodes in $C_{k,\ell}\setminus \mathcal B$ as well. Therefore, as long as one avoids further pathological cancellations on these potentially active coordinates, one expects $\left|\bigcup_{i\in\mathcal B}\mathcal N_i\right|>|\mathcal B|$. In this sense, the third part likewise excludes degenerate interventions that collapse the within-block information. Overall, Assumption~\ref{ass:non-degeneracy} only requires that interventions generically reveal the structural signal allowed by the graph, rather than imposing any strong additional structural restriction.

\paragraph{Discussion of Assumption \ref{ass: global cover edge}.} Assumption~\ref{ass: global cover edge} is a cross-client coverage condition tailored to global recovery. Its role is to constrain the missingness patterns $\{O_k\}_{k=1}^K$ across clients so that the client-specific structural fragments can be assembled into a global latent causal order. Some such condition is unavoidable in our setting. On the one hand, if $O_k=[p]$ for every client $k$, then the problem essentially reduces to the single-client case, and there is no nontrivial fragmented multi-client assembly. On the other hand, if $\cup_{k=1}^K O_k \subsetneq [p]$, then at least one global latent node is never observed by any client, and its ancestral relations with the remaining nodes are clearly unidentifiable. Assumption~\ref{ass: global cover edge} adopts a relatively weak intermediate condition: rather than requiring some client to observe all nodes, or requiring every global ancestral relation to be directly recovered in a single client, it only requires each global Hasse edge to be witnessed in some client through a block-level parent relation between singleton blocks. Since the global ancestral order is the transitive closure of the Hasse edges, this provides exactly the level of visibility needed for the subsequent global assembly step.

\section{Proof of Proposition \ref{prop:margin Z}}
\label{app:proof Prop matgin Z}

\begin{proposition}[Proposition \ref{prop:margin Z} in the main text]
    For $Z$ generated by Eq.~\eqref{eq:gloabl latent}, let $Z_k=(Z_{(i)})_{i\in O_k}$. Then
\begin{equation}
Z_k = \bA_k Z_k + E_k,
\end{equation}
where
\begin{align}
\label{eq:Ak2}
    \bA_k
    &=
    [\bA]_{O_k,O_k}
    +
    [\bA]_{O_k,M_k}
    (\bI-[\bA]_{M_k,M_k})^{-1}
    [\bA]_{M_k,O_k}, \\
    E_k
    &=
    E_{O_k}
    +
    [\bA]_{O_k,M_k}
    (\bI-[\bA]_{M_k,M_k})^{-1}
    E_{M_k}.
\end{align}
Moreover, $E_k$ has mean zero and covariance
\begin{equation}
\label{eq:Sigmak2}
    \bSigma_k
    =
    [\bSigma]_{O_k,O_k}
    +
    [\bA]_{O_k,M_k}
    (\bI-[\bA]_{M_k,M_k})^{-1}
    [\bSigma]_{M_k,M_k}
    (\bI-[\bA]_{M_k,M_k})^{-T}
    [\bA]_{O_k,M_k}^{T}.
\end{equation}

Furthermore $\prec_k$, which is induced by $\bA_k$, is exactly the restriction of the $\prec$, which is induced by $\bA$,  to $\cV_k = O_k$.
\end{proposition}

\begin{proof}
For client $k$, partition the global latent vector and exogenous noise as
\[
Z=
\begin{pmatrix}
Z_{O_k}\\
Z_{M_k}
\end{pmatrix},
\qquad
E=
\begin{pmatrix}
E_{O_k}\\
E_{M_k}
\end{pmatrix}.
\]
Then the global latent SEM in Eq.~(\ref{eq:gloabl latent}) can be written in block form as
\[
\begin{pmatrix}
Z_{O_k}\\
Z_{M_k}
\end{pmatrix}
=
\begin{pmatrix}
[\bA]_{O_k,O_k} & [\bA]_{O_k,M_k}\\
[\bA]_{M_k,O_k} & [\bA]_{M_k,M_k}
\end{pmatrix}
\begin{pmatrix}
Z_{O_k}\\
Z_{M_k}
\end{pmatrix}
+
\begin{pmatrix}
E_{O_k}\\
E_{M_k}
\end{pmatrix}.
\]
Equivalently,
\begin{align}
Z_{O_k}
&=
[\bA]_{O_k,O_k}Z_{O_k}
+
[\bA]_{O_k,M_k}Z_{M_k}
+
E_{O_k},
\label{eq:prop1_block_O}
\\
Z_{M_k}
&=
[\bA]_{M_k,O_k}Z_{O_k}
+
[\bA]_{M_k,M_k}Z_{M_k}
+
E_{M_k}.
\label{eq:prop1_block_M}
\end{align}

Since $\bA$ is strictly upper triangular under a topological order, the principal submatrix $[\bA]_{M_k,M_k}$ is also strictly upper triangular. Hence $\bI-[\bA]_{M_k,M_k}$ is invertible, and Eq.~(\ref{eq:prop1_block_M}) gives
\begin{equation}
\label{eq:prop1_ZM}
Z_{M_k}
=
(\bI-[\bA]_{M_k,M_k})^{-1}[\bA]_{M_k,O_k}Z_{O_k}
+
(\bI-[\bA]_{M_k,M_k})^{-1}E_{M_k}.
\end{equation}
Substituting Eq.~(\ref{eq:prop1_ZM}) into Eq.~(\ref{eq:prop1_block_O}), we obtain
\begin{align*}
Z_{O_k}
&=
[\bA]_{O_k,O_k}Z_{O_k}
+
[\bA]_{O_k,M_k}
(\bI-[\bA]_{M_k,M_k})^{-1}
[\bA]_{M_k,O_k}Z_{O_k}
\\
&\qquad
+
[\bA]_{O_k,M_k}
(\bI-[\bA]_{M_k,M_k})^{-1}E_{M_k}
+
E_{O_k}.
\end{align*}
Therefore,
\[
Z_k=\bA_k Z_k+E_k,
\qquad Z_k=Z_{O_k},
\]
where
\[
\bA_k
=
[\bA]_{O_k,O_k}
+
[\bA]_{O_k,M_k}
(\bI-[\bA]_{M_k,M_k})^{-1}
[\bA]_{M_k,O_k},
\]
and
\[
E_k
=
E_{O_k}
+
[\bA]_{O_k,M_k}
(\bI-[\bA]_{M_k,M_k})^{-1}
E_{M_k}.
\]
This proves the claimed induced structural equation.

It remains to compute the covariance of $E_k$. Let
\[
\bB_k
\triangleq
[\bA]_{O_k,M_k}
(\bI-[\bA]_{M_k,M_k})^{-1}.
\]
Then
\[
E_k=E_{O_k}+\bB_kE_{M_k}.
\]
Since the global exogenous noise covariance is $\bSigma$ and the coordinates of $E$ are mutually uncorrelated, we have
\[
Cov(E_{O_k},E_{M_k})=\mathbf 0.
\]
Hence
\begin{align*}
\bSigma_k
\triangleq
Cov(E_k)
&=
Cov(E_{O_k}+\bB_kE_{M_k}) \\
&=
[\bSigma]_{O_k,O_k}
+
\bB_k[\bSigma]_{M_k,M_k}\bB_k^T.
\end{align*}
Substituting the definition of $\bB_k$ yields
\[
\bSigma_k
=
[\bSigma]_{O_k,O_k}
+
[\bA]_{O_k,M_k}
(\bI-[\bA]_{M_k,M_k})^{-1}
[\bSigma]_{M_k,M_k}
(\bI-[\bA]_{M_k,M_k})^{-T}
[\bA]_{O_k,M_k}^T,
\]
which is exactly Eq.~(\ref{eq:Sigmak2}).

Finally, we justify the statement on the ancestral relation $\prec_k$. By construction, the directed edge set of the induced mixed graph $\cG_k$ is determined by the nonzero entries of $\bA_k$. From the formula above, a directed edge in $\cG_k$ arises either from a directed edge already present in $[\bA]_{O_k,O_k}$, or from a directed path in the global DAG whose endpoints lie in $O_k$ and whose internal nodes lie in $M_k$. Indeed, since $[\bA]_{M_k,M_k}$ is strictly upper triangular, it is nilpotent, and thus
\[
(\bI-[\bA]_{M_k,M_k})^{-1}
=
\bI+[\bA]_{M_k,M_k}+[\bA]_{M_k,M_k}^2+\cdots,
\]
where the sum is finite. Therefore,
\[
[\bA]_{O_k,M_k}
(\bI-[\bA]_{M_k,M_k})^{-1}
[\bA]_{M_k,O_k}
\]
collects exactly the contributions of all directed paths whose endpoints lie in $O_k$ and whose internal nodes lie in $M_k$. It follows that every directed path in $\cG_k$ corresponds to a directed path in the global latent DAG. Hence, for any $i,j\in O_k$,
\[
i\prec_k j \quad\Longrightarrow\quad i\prec j.
\]

Conversely, suppose $i,j\in O_k$ and $i\prec j$ in the global latent DAG. Then there exists a directed path from $j$ to $i$ in the global graph. Removing from this path all intermediate nodes that lie in $O_k$ decomposes it into consecutive segments whose endpoints lie in $O_k$ and whose internal nodes lie in $M_k$. Each such segment induces a directed edge in $\cG_k$ through the effective coefficient matrix $\bA_k$, possibly together with edges already present in $[\bA]_{O_k,O_k}$. Chaining these induced directed edges yields a directed path from $j$ to $i$ in $\cG_k$, and therefore
\[
i\prec j,\quad i,j\in O_k
\quad\Longrightarrow\quad
i\prec_k j.
\]
Combining the two directions, we conclude that $\prec_k$ is exactly the restriction of the global ancestral relation $\prec$ to $\cV_k=O_k$.
\end{proof}

\section{Proof of Proposition \ref{prop:block anc}}
\label{app:Proof Prop block anc}

\begin{proposition}[Proposition \ref{prop:block anc} in the main text]
    For each client $k$, let $C_{k,i},C_{k,j}\in\cC_k$ be two blocks.
\begin{enumerate}[label=(\roman*), leftmargin=2.4em, itemsep=0pt, topsep=1pt]
    \item If $C_{k,j}\notin anc_k^{\rblk}(C_{k,i})$, then for $\forall u\in C_{k,j}$ and $v\in C_{k,i}$, we have $u\notin anc_k(v)$.
    \item If $C_{k,j}\in pa_k^{\rblk}(C_{k,i})$, then there $\exists u\in C_{k,j}$ and $v\in C_{k,i}$ such that $u\in pa_k(v)$.
\end{enumerate}
\end{proposition}

\begin{proof}
We prove the two claims separately.

For the first claim, suppose $C_{k,j}\notin Anc_k^\rblk(C_{k,i})$. We show that there do not exist $u\in C_{k,j}$ and $v\in C_{k,i}$ such that $u\in Anc_k(v)$. Indeed, if such a pair existed, then by the definition of the block-level ancestral relation, the existence of $u\prec_k v$ would imply
\[
C_{k,j}\prec_k^\rblk C_{k,i},
\]
or equivalently $C_{k,j}\in Anc_k^\rblk(C_{k,i})$, which contradicts the assumption. Therefore, for any $u\in C_{k,j}$ and $v\in C_{k,i}$, we must have $u\notin Anc_k(v)$.

We next prove the second claim. Suppose $C_{k,j}\in Pa_k^\rblk(C_{k,i})$. By definition of the block parent relation, there exist $u\in C_{k,j}$ and $v\in C_{k,i}$ such that $u\in Anc_k(v)$, that is, there exists a directed path from $u$ to $v$ in the directed part of $\cG_k$. Among all such directed paths from a node in $C_{k,j}$ to a node in $C_{k,i}$, choose one of shortest length, and write it as
\[
u=w_0 \to w_1 \to \cdots \to w_m=v.
\]

Consider the sequence of blocks containing these nodes:
\[
C_k(w_0),\, C_k(w_1),\, \ldots,\, C_k(w_m),
\]
where $C_k(w_\ell)$ denotes the block in $\cC_k$ containing $w_\ell$. Since $w_0=u\in C_{k,j}$ and $w_m=v\in C_{k,i}$, the path starts in $C_{k,j}$ and ends in $C_{k,i}$. We claim that no intermediate node $w_\ell$ can belong to a block different from both $C_{k,j}$ and $C_{k,i}$. Otherwise, for some $1\le \ell\le m-1$, there exists a block $C_{k,m'}\in\cC_k$ such that
\[
w_\ell\in C_{k,m'}, \qquad C_{k,m'}\neq C_{k,j}, \qquad C_{k,m'}\neq C_{k,i}.
\]
Then the directed path above implies
\[
C_{k,j}\prec_k^\rblk C_{k,m'} \prec_k^\rblk C_{k,i},
\]
which contradicts the assumption that $C_{k,j}\in Pa_k^\rblk(C_{k,i})$, since a block parent cannot have another block strictly between it and $C_{k,i}$ in the block-level partial order.

Hence the path cannot pass through any third block. It follows that along the chosen directed path, once the path leaves $C_{k,j}$, it must enter $C_{k,i}$ immediately. Therefore, there exists an index $r\in\{1,\dots,m\}$ such that
\[
w_{r-1}\in C_{k,j}, \qquad w_r\in C_{k,i},
\]
and
\[
w_{r-1}\to w_r
\]
is a directed edge in $\cG_k$. Consequently,
\[
w_{r-1}\in Pa_k(w_r).
\]
Setting
\[
u^\star\triangleq w_{r-1}\in C_{k,j}, \qquad v^\star\triangleq w_r\in C_{k,i},
\]
we obtain a pair of nodes such that $u^\star\in Pa_k(v^\star)$, as required.
\end{proof}

\section{Block RQ Decomposition}
\label{app:block RQ}
We now give a constructive procedure for the block-RQ decomposition introduced in Section \ref{sec:block RQ}. Given $\bH_k$, a block topological order, and the corresponding block ancestral relation, the algorithm builds $\bQ_k$ block by block in reverse topological order. At each step, the current block row of $\bH_k$ is projected onto the orthogonal complement of the row space already explained by its ancestor blocks, and the resulting innovation component is used to define the new block row of $\bQ_k$. The coefficient matrix $\bR_k$ is then obtained by expressing each block row of $\bH_k$ in the row span of the corresponding ancestor blocks together with its own newly added block.

\begin{algorithm}[H]
\caption{Block-RQ decomposition}
\label{alg:block_rq}
\begin{algorithmic}[1]
\Require $\bH_k \in \mathbb{R}^{p_k \times n_k}$, block topological order $C_r, C_{r-1}, \dots, C_1$, block ancestor relation $anc_k^\rblk(C_\ell)$ for each block $C_\ell$
\Ensure $\bR_k \in \mathbb{R}^{p_k \times p_k}$, $\bQ_k \in \mathbb{R}^{p_k \times n_k}$

\State Initialize $\bR_k = \mathbf 0_{p_k \times p_k}$, $\bQ_k = \mathbf 0_{p_k \times n_k}$

\For{$\ell = r, r-1, \dots, 1$}
    \State Let $\widetilde{\bH}_\ell = [\bH_k]_{C_\ell, :}$
    \State Let $\bW_\ell = [\bQ_k]_{anc_k^\rblk(C_\ell), :}$
    \If{$\bW_\ell$ is empty}
        \State $\widetilde{\bH}_\ell^\perp = \widetilde{\bH}_\ell$
    \Else
        \State $\widetilde{\bH}_\ell^\perp = \Pi_{\bW_\ell^\perp}(\widetilde{\bH}_\ell)$
    \EndIf
    \State Normalize each row of $\widetilde{\bH}_\ell^\perp$ to unit norm, and denote the resulting matrix by $[\bQ_k]_{C_\ell, :}$
    \State Let $J_\ell = C_\ell \cup Anc(C_\ell)$
    \State Set
    \[
    [\bR_k]_{C_\ell, J_\ell}
    =
    \widetilde{\bH}_\ell \bigl([\bQ_k]_{J_\ell, :}\bigr)^\dagger
    \]
    \State Set $\bR_k[C_\ell, J_\ell^c] = 0$
\EndFor

\State \Return $\bR_k, \bQ_k$
\end{algorithmic}
\end{algorithm}

The next proposition shows that the above procedure indeed produces a valid block-RQ decomposition in the sense of Section \ref{sec:block RQ}.

\begin{proposition}
The output $(\bR_k,\bQ_k)$ of Algorithm~\ref{alg:block_rq} satisfies the block-RQ requirements in Section~\ref{sec:block RQ}. In particular,
\begin{enumerate}
    \item $\bH_k=\bR_k\bQ_k$;
    \item for each block $C_\ell$, the row block $[\bR_k]_{C_\ell,:}$ is supported only on $J_\ell=Anc_k^\rblk(C_\ell)$, namely $[\bR_k]_{C_\ell,J_\ell^c}=0$;
    \item for each block $C_\ell$, the row block $[\bQ_k]_{C_\ell,:}$ is orthogonal to the row space generated by its strict block ancestors, i.e.
    \[
    \operatorname{row}\bigl([\bQ_k]_{C_\ell,:}\bigr)
    \perp
    \operatorname{row}\bigl([\bQ_k]_{anc_k^\rblk(C_\ell),:}\bigr),
    \]
    and
    \[
    \operatorname{row}\bigl([\bQ_k]_{C_\ell,:}\bigr)
    =
    \operatorname{row}\!\left(
    \Pi_{\operatorname{row}([\bQ_k]_{anc_k^\rblk(C_\ell),:})^\perp}
    \bigl([\bH_k]_{C_\ell,:}\bigr)
    \right).
    \]
\end{enumerate}
\end{proposition}

\begin{proof}
Fix a block $C_\ell$. By construction,
\[
\widetilde{\bH}_\ell^\perp
=
\Pi_{\operatorname{row}([\bQ_k]_{anc_k^\rblk(C_\ell),:})^\perp}\bigl([\bH_k]_{C_\ell,:}\bigr),
\]
with the convention that $\widetilde{\bH}_\ell^\perp=[\bH_k]_{C_\ell,:}$ if $anc_k^\rblk(C_\ell)=\emptyset$. Since Step 10 only normalizes the rows of $\widetilde{\bH}_\ell^\perp$, it does not change its row space. Therefore
\[
\operatorname{row}\bigl([\bQ_k]_{C_\ell,:}\bigr)
=
\operatorname{row}\bigl(\widetilde{\bH}_\ell^\perp\bigr)
=
\operatorname{row}\!\left(
\Pi_{\operatorname{row}([\bQ_k]_{anc_k^\rblk(C_\ell),:})^\perp}
\bigl([\bH_k]_{C_\ell,:}\bigr)
\right).
\]
Because $\widetilde{\bH}_\ell^\perp$ is defined as the orthogonal projection onto
\[
\operatorname{row}([\bQ_k]_{anc_k^\rblk(C_\ell),:})^\perp,
\]
its row space is orthogonal to $\operatorname{row}([\bQ_k]_{anc_k^\rblk(C_\ell),:})$. Hence
\[
\operatorname{row}\bigl([\bQ_k]_{C_\ell,:}\bigr)
\perp
\operatorname{row}\bigl([\bQ_k]_{anc_k^\rblk(C_\ell),:}\bigr),
\]
which proves the third claim.

Next, by the orthogonal decomposition of $[\bH_k]_{C_\ell,:}$ with respect to
\[
\operatorname{row}\bigl([\bQ_k]_{anc_k^\rblk(C_\ell),:}\bigr),
\]
we have
\[
[\bH_k]_{C_\ell,:}
=
\Pi_{\operatorname{row}([\bQ_k]_{anc_k^\rblk(C_\ell),:})}\bigl([\bH_k]_{C_\ell,:}\bigr)
+
\widetilde{\bH}_\ell^\perp.
\]
Therefore,
\[
\operatorname{row}\bigl([\bH_k]_{C_\ell,:}\bigr)
\subseteq
\operatorname{row}\bigl([\bQ_k]_{J_\ell,:}\bigr),
\qquad
J_\ell=Anc_k^\rblk(C_\ell).
\]
It follows that
\[
[\bH_k]_{C_\ell,:}
=
[\bR_k]_{C_\ell,J_\ell}[\bQ_k]_{J_\ell,:},
\]
where
\[
[\bR_k]_{C_\ell,J_\ell}
=
[\bH_k]_{C_\ell,:}\bigl([\bQ_k]_{J_\ell,:}\bigr)^\dagger.
\]
Together with the assignment $[\bR_k]_{C_\ell,J_\ell^c}=0$, this gives
\[
[\bH_k]_{C_\ell,:}
=
[\bR_k]_{C_\ell,:}\bQ_k.
\]
Stacking over all blocks yields
\[
\bH_k=\bR_k\bQ_k,
\]
which proves the first claim.

Finally, the second claim follows directly from the construction of $\bR_k$, since Step 13 explicitly sets
\[
[\bR_k]_{C_\ell,J_\ell^c}=0
\]
for every block $C_\ell$. Thus each block row of $\bR_k$ is supported only on the current block and its block ancestors.
\end{proof}

This decomposition is generally not unique, since the algorithm only determines each block through its row space, while different normalized bases of the same subspace lead to different matrix representatives $\bQ_k$ and hence different $\bR_k$. However, for Theorem \ref{thm:client-main}, the relevant object is precisely the family of block row spaces $\{\operatorname{row}([\bQ_k]_{C_\ell,:})\}_\ell$, rather than a particular matrix representative. These block row spaces are uniquely determined by $\bH_k$, the prescribed block order, and the block ancestral relation, since each step extracts the orthogonal innovation subspace relative to already determined ancestor blocks. In this subspace sense, the decomposition is unique, which is sufficient for the identifiability argument in Theorem \ref{thm:client-main}.

\begin{proposition}[Proposition \ref{prop: RQ} in the main text]
    For every block $C \in \mathcal{C}_k$ and every $i \in \rho_k(C)$, we have $[\bH_k]_{i,:} \in \mathrm{row}([\bQ_k]_{\bigcup_{C' \in Anc_k^{\mathrm{blk}}(C)} \rho_k(C'),:})$.
\end{proposition}

\begin{proof}
Fix a block $C\in\cC_k$ and an index $i\in \rho_k(C)$. By Definition~2, the block-RQ decomposition satisfies
\[
\bH_k=\bR_k\bQ_k.
\]
Hence the $i$th row of $\bH_k$ can be written as
\[
[\bH_k]_{i,:}
=
\sum_{m=1}^{p_k} [\bR_k]_{i,m}[\bQ_k]_{m,:}.
\]
That is, $[\bH_k]_{i,:}$ is a linear combination of the rows of $\bQ_k$, with coefficients given by the $i$th row of $\bR_k$.

Since $i\in \rho_k(C)$, the $i$th row belongs to the block row indexed by $C$. By the block support property in Definition~2, this block row of $\bR_k$ can be nonzero only on the column blocks corresponding to $Anc_k^\rblk(C)$. Equivalently, for any block $C'\notin Anc_k^\rblk(C)$,
\[
[\bR_k]_{i,\rho_k(C')}=0.
\]
Therefore, the above sum reduces to
\[
[\bH_k]_{i,:}
=
\sum_{C'\in Anc_k^\rblk(C)}
\sum_{m\in \rho_k(C')}
[\bR_k]_{i,m}[\bQ_k]_{m,:}.
\]
This shows that $[\bH_k]_{i,:}$ lies in the row space generated by the rows of $\bQ_k$ indexed by the ancestor blocks of $C$, namely
\[
[\bH_k]_{i,:}
\in
\mathrm{row}\!\left(
[\bQ_k]_{\cup_{C'\in Anc_k^\rblk(C)}\rho_k(C'),:}
\right).
\]
The claim follows.
\end{proof}

\section{Proof of Lemma \ref{lem: rank 2}}
\label{app: proof Lemma rank 2}

\begin{lemma}[Lemma \ref{lem: rank 2} in the main text]
For each client $k$ and each intervention $j$ targeting node $t_j$ (which is unknown to the learner), we have
\begin{equation}
    \mathrm{row}(\bTheta_k^{(j)} - \bTheta_k^{(0)}) \subseteq \langle \bH_k^T(\bI-\bA_k)^{T}\bSigma_k^{-1}\be_{\rho_k(t_j)}, \bH_k^T\left ((\bI-\bA_k)^{T}\bSigma_k^{-1}\bu_k^{(j)} + \bd_k^{(j)}\right ) \rangle.
\end{equation}
where $\bu_k^{(j)},\bd_k^{(j)}$ have the same definition as in Assumption \ref{ass:non-degeneracy}. Therefore, it immediately follows that
$\mathrm{rank}(\bTheta_k^{(j)} - \bTheta_k^{(0)}) \leq 2$.
\end{lemma}

\begin{proof}
Fix a client $k$ and an intervention $j\in\cJ_k$, and write
\[
\bB_k \coloneqq \bI-\bA_k,\qquad
\bOmega_k \coloneqq \bSigma_k^{-1},\qquad
e \coloneqq \be_{\rho_k(t_j)},\qquad
d \coloneqq \bd_k^{(j)},\qquad
u \coloneqq \bu_k^{(j)}.
\]
By the precision factorization established earlier, we have
\[
\bTheta_k^{(j)}-\bTheta_k^{(0)}
=
\bH_k^T
\Bigl(
(\bI-\bA_k^{(j)})^T(\bSigma_k^{(j)})^{-1}(\bI-\bA_k^{(j)})
-
\bB_k^T\bOmega_k\bB_k
\Bigr)
\bH_k.
\]
Hence it suffices to study
\[
\bDelta_k^{(j)}
\coloneqq
(\bI-\bA_k^{(j)})^T(\bSigma_k^{(j)})^{-1}(\bI-\bA_k^{(j)})
-
\bB_k^T\bOmega_k\bB_k .
\]

By Lemma~\ref{lem:support-du}, the intervention changes only the row corresponding to the target node, so
\[
\bA_k^{(j)}=\bA_k+ed^T,
\qquad\text{hence}\qquad
\bI-\bA_k^{(j)}=\bB_k-ed^T.
\]
Moreover, there exists a scalar $\eta_k^{(j)}$ such that
\[
\bSigma_k^{(j)}
=
\bSigma_k+eu^T+ue^T+\eta_k^{(j)}ee^T .
\]
For brevity, write $\eta\coloneqq \eta_k^{(j)}$, and define
\[
\bV \coloneqq [e,u],
\qquad
\bW \coloneqq [u+\eta e,e].
\]
Then
\[
\bSigma_k^{(j)}=\bSigma_k+\bV\bW^T.
\]
Applying the Woodbury identity gives
\[
(\bSigma_k^{(j)})^{-1}
=
\bOmega_k-\bOmega_k\bV\bigl(\bI+\bW^T\bOmega_k\bV\bigr)^{-1}\bW^T\bOmega_k .
\]

Substituting this into $\bDelta_k^{(j)}$ yields
\begin{align}
\bDelta_k^{(j)}
&=
(\bB_k-ed^T)^T\bOmega_k(\bB_k-ed^T)-\bB_k^T\bOmega_k\bB_k \notag\\
&\quad
-(\bB_k-ed^T)^T\bOmega_k\bV\bigl(\bI+\bW^T\bOmega_k\bV\bigr)^{-1}\bW^T\bOmega_k(\bB_k-ed^T).
\label{eq:lemma1-delta-expand}
\end{align}
We now introduce
\[
q \coloneqq \bB_k^T\bOmega_k e,
\qquad
s \coloneqq \bB_k^T\bOmega_k u .
\]
Then the first line of \eqref{eq:lemma1-delta-expand} becomes
\[
(\bB_k-ed^T)^T\bOmega_k(\bB_k-ed^T)-\bB_k^T\bOmega_k\bB_k
=
-qd^T-dq^T+(e^T\bOmega_k e)\,dd^T.
\]
For the second line, direct multiplication gives
\[
(\bB_k-ed^T)^T\bOmega_k\bV
=
\bigl[q-(e^T\bOmega_k e)d,\ s-(e^T\bOmega_k u)d\bigr],
\]
and
\[
\bW^T\bOmega_k(\bB_k-ed^T)
=
\begin{bmatrix}
\bigl(s+\eta q-(e^T\bOmega_k u+\eta\,e^T\bOmega_k e)d\bigr)^T\\[2mm]
\bigl(q-(e^T\bOmega_k e)d\bigr)^T
\end{bmatrix}.
\]
Let
\[
\tau \coloneqq e^T\bOmega_k e,\qquad
\kappa \coloneqq e^T\bOmega_k u,\qquad
\nu \coloneqq u^T\bOmega_k u.
\]
Then
\[
\bI+\bW^T\bOmega_k\bV
=
\begin{pmatrix}
1+\kappa+\eta\tau & \nu+\eta\kappa\\
\tau & 1+\kappa
\end{pmatrix},
\]
which is invertible since $\bSigma_k^{(j)}$ is invertible. A straightforward simplification of
\eqref{eq:lemma1-delta-expand} then yields
\[
\bDelta_k^{(j)}
=
c_1\,qq^T
+
c_2\,q(s+d)^T
+
c_2\,(s+d)q^T
+
c_3\,(s+d)(s+d)^T
\]
for some scalars $c_1,c_2,c_3$ depending on $(k,j)$. For example, one may take
\[
c_1=\frac{\nu-\eta}{(1+\kappa)^2+\tau(\eta-\nu)},
\qquad
c_2=-\frac{1+\kappa}{(1+\kappa)^2+\tau(\eta-\nu)},
\qquad
c_3=\frac{\tau}{(1+\kappa)^2+\tau(\eta-\nu)}.
\]
In particular,
\[
\row(\bDelta_k^{(j)}) \subseteq \langle q,\ s+d\rangle.
\]
Recalling the definitions of $q$ and $s$, this becomes
\[
\row(\bDelta_k^{(j)})
\subseteq
\left\langle
\bB_k^T\bOmega_k e,\ 
\bB_k^T\bOmega_k u+d
\right\rangle
=
\left\langle
(\bI-\bA_k)^T\bSigma_k^{-1}\be_{\rho_k(t_j)},\ 
(\bI-\bA_k)^T\bSigma_k^{-1}\bu_k^{(j)}+\bd_k^{(j)}
\right\rangle .
\]

Finally, since
\[
\bTheta_k^{(j)}-\bTheta_k^{(0)}
=
\bH_k^T \bDelta_k^{(j)} \bH_k,
\]
we conclude that
\[
\row(\bTheta_k^{(j)}-\bTheta_k^{(0)})
\subseteq
\left\langle
\bH_k^T(\bI-\bA_k)^T\bSigma_k^{-1}\be_{\rho_k(t_j)},\ 
\bH_k^T\Bigl((\bI-\bA_k)^T\bSigma_k^{-1}\bu_k^{(j)}+\bd_k^{(j)}\Bigr)
\right\rangle.
\]
Therefore
\[
\rank(\bTheta_k^{(j)}-\bTheta_k^{(0)})\le 2.
\]
This proves the lemma.
\end{proof}

\section{Proof of Lemma \ref{lem: low rank Qk}}
\label{app:proof Lem low rank Qk}

\begin{lemma}[Lemma \ref{lem: low rank Qk} in the main text]
For each client $k$ and intervention $j$, suppose that the intervention target node $t_j$ belongs to block $C_{k,a}\in\cC_k$. Then
\[\mathrm{row}(\bTheta_k^{(j)}-\bTheta_k^{(0)})
\subseteq
\left\langle
[\bQ_k]_{\rho_k(i),:}
\;\middle|\;
i\in C_{k,b},\; C_{k,b}\in Anc_k^\rblk(C_{k,a})
\right\rangle.\]
\end{lemma}

\begin{proof}
Let $t_j\in C_{k,a}$ and write $r=\rho_k(t_j)$. By Lemma~\ref{lem: rank 2},
\[
\row(\bTheta_k^{(j)}-\bTheta_k^{(0)})
\subseteq
\left\langle
\bH_k^T(\bI-\bA_k)^T\bSigma_k^{-1}\be_r,\ 
\bH_k^T\Bigl((\bI-\bA_k)^T\bSigma_k^{-1}\bu_k^{(j)}+\bd_k^{(j)}\Bigr)
\right\rangle.
\]
Hence it suffices to show that both generators on the right-hand side belong to
\[
\left\langle
[\bQ_k]_{\rho_k(i),:}\ \middle|\ i\in C_{k,b},\ C_{k,b}\in Anc_k^\rblk(C_{k,a})
\right\rangle.
\]

For the first generator, define
\[
\balpha \coloneqq (\bI-\bA_k)^T\bSigma_k^{-1}\be_r.
\]
Then
\[
\bH_k^T(\bI-\bA_k)^T\bSigma_k^{-1}\be_r
=
\bH_k^T\balpha
=
\sum_{i\in V_k}\alpha_i[\bH_k]_{i,:}^T.
\]
Since $r=\rho_k(t_j)$ and $t_j\in C_{k,a}$, any nonzero coefficient $\alpha_i$ can only arise from indices $i$ belonging to blocks in $Anc_k^\rblk(C_{k,a})$. Therefore the above vector is a linear combination of rows of $\bH_k$ indexed by nodes in blocks belonging to $Anc_k^\rblk(C_{k,a})$.

For the second generator, define
\[
\bbeta \coloneqq (\bI-\bA_k)^T\bSigma_k^{-1}\bu_k^{(j)}+\bd_k^{(j)}.
\]
Then
\[
\bH_k^T\Bigl((\bI-\bA_k)^T\bSigma_k^{-1}\bu_k^{(j)}+\bd_k^{(j)}\Bigr)
=
\bH_k^T\bbeta
=
\sum_{i\in V_k}\beta_i[\bH_k]_{i,:}^T.
\]
By Lemma~\ref{lem:support-du},
\[
\supp(\bd_k^{(j)}) \subseteq \rho_k\bigl(pa_k(t_j)\bigr),
\qquad
\supp(\bu_k^{(j)}) \subseteq \rho_k\bigl(sp_k(t_j)\bigr).
\]
Hence the nonzero coordinates of $\bbeta$ can only come from set $\rho_k(Pa_k(dis_k(t_j)\cup\{t_j\}))$, where $dis_k(i)\triangleq \{j\in\cV_k\mid \text{$i$ and $j$ are connected by a bidirected path in $\cD_k$}\}$. By the definition of the block structure and the block-level ancestor relation, these nodes belong to blocks contained in $Anc_k^\rblk(C_{k,a})$. It follows that the second generator is also a linear combination of rows of $\bH_k$ indexed by nodes in blocks belonging to $Anc_k^\rblk(C_{k,a})$.

Now apply Proposition~3. For any node $i$ lying in a block $C\in\cC_k$, Proposition~3 implies that
\[
[\bH_k]_{i,:}
\in
\left\langle
[\bQ_k]_{\rho_k(\ell),:}\ \middle|\ \ell\in C_{k,b},\ C_{k,b}\in Anc_k^\rblk(C)
\right\rangle.
\]
In particular, whenever $C\in Anc_k^\rblk(C_{k,a})$, we also have
\[
[\bH_k]_{i,:}
\in
\left\langle
[\bQ_k]_{\rho_k(\ell),:}\ \middle|\ \ell\in C_{k,b},\ C_{k,b}\in Anc_k^\rblk(C_{k,a})
\right\rangle.
\]
Therefore both generators in Lemma~\ref{lem: rank 2} lie in
\[
\left\langle
[\bQ_k]_{\rho_k(i),:}\ \middle|\ i\in C_{k,b},\ C_{k,b}\in Anc_k^\rblk(C_{k,a})
\right\rangle.
\]
Taking their span proves that
\[
\row(\bTheta_k^{(j)}-\bTheta_k^{(0)})
\subseteq
\left\langle
[\bQ_k]_{\rho_k(i),:}\ \middle|\ i\in C_{k,b},\ C_{k,b}\in Anc_k^\rblk(C_{k,a})
\right\rangle.
\]
This completes the proof.
\end{proof}

\section{Proof of Theorem \ref{thm:client-main}}
\label{app: proof client main}

\begin{theorem}[Theorem \ref{thm:client-main} in the main text]
Assume Assumptions~\ref{ass: aligned intervention} and~\ref{ass:non-degeneracy}. Let $(\hat{\mathcal C}_k,\hat{\prec}_k^\rblk,\hat{\bQ}_k)$ be the output of Algorithm~\ref{alg:causal order client} for client $k$. Then:
\begin{enumerate}[label=(\roman*), leftmargin=2.4em, itemsep=0pt, topsep=1pt]
    \item $(\hat{\mathcal C}_k,\hat{\prec}_k^\rblk)$ recovers $(\mathcal C_k,\prec_k^\rblk)$ up to relabeling under the bijection $\psi:j\mapsto t_j$.
    \item For each true block $C\in\mathcal C_k$, let $\hat C=\psi^{-1}(C)$ denote the corresponding estimated block. Then
    $\langle \hat{\bq}_{k,i}: i\in\hat C\rangle
    =
    \mathrm{row}([\bQ_k]_{\rho_k(C),:}).$
\end{enumerate}
Consequently,$\hat{\bQ}_k=\bS_k\bP_{\sigma,k}\bQ_k$,
where $\bS_k$ is an invertible block-diagonal norm-preserving matrix and $\bP_{\sigma,k}$ is a row permutation matrix preserving the block ancestral order.
\end{theorem}

\begin{proof}
We proceed by induction on the outer-loop index of Algorithm~2. Let $T=|\cC_k|$ be the number of true blocks for client $k$. For each $t\in[T]$, we view the $t$-th iteration of the outer loop as the step that recovers the $t$-th block in the reverse block topological order.

At the beginning of iteration $t$, we maintain the following induction hypothesis:

\begin{itemize}
    \item[(IH1)] $\mathcal I_{t-1}$ consists of correctly identified blocks, and its image under $\psi$ is closed under the true block-level ancestor relation. Namely, if $C\in \psi(\mathcal I_{t-1})$, then
    $
    Anc_k^\rblk(C)\subseteq \psi(\mathcal I_{t-1}).
    $
    \item[(IH2)] The last $t-1$ recovered block rows of $\hat{\bQ}_k$ coincide with the corresponding block rows of some matrix of the form $\bS_k\bP_{\sigma,k}\bQ_k$, where $\bP_{\sigma,k}$ is a permutation consistent with the recovered block order and $\bS_k$ is a block-diagonal invertible matrix.
\end{itemize}

The base case $t=1$ is immediate: $\mathcal I_0=\emptyset$, so (IH1) holds trivially, and (IH2) is vacuous.

Now assume (IH1)--(IH2) hold at the beginning of iteration $t$. We prove that the $t$-th iteration correctly identifies one new block, correctly updates its ancestor information, and appends the corresponding correct block row space to $\hat{\bQ}_k$.

\medskip
\noindent\textbf{Step 1: identification of the new block in Steps 4--12 of Algorithm~2.}
Let
\[
W_t \;\triangleq\; \langle \hat{\bq}_{k,i}: i\in \mathcal I_{t-1}\rangle .
\]
We first record three claims.

\smallskip
\noindent\emph{Claim 1.} Under Assumptions \ref{ass: aligned intervention} and \ref{ass:non-degeneracy}, 
let $J=\{j_1,\dots,j_r\}\subseteq R_{t-1}$. If $\psi(J)$ is not a disjoint union of whole true blocks, then
\[
\dim\big(\cP_k(J;W_t)\big)>r.
\]

\smallskip
\noindent\emph{Claim 2.}Under Assumptions \ref{ass: aligned intervention} and \ref{ass:non-degeneracy}, 
let $J=\{j_1,\dots,j_r\}\subseteq R_{t-1}$. Suppose that $\psi(J)$ is a union of whole true blocks, but at least one true block parent of $\psi(J)$ is not contained in $\psi(\mathcal I_{t-1}\cup J)$. Then
\[
\dim\big(\cP_k(J;W_t)\big)>r.
\]

\smallskip
\noindent\emph{Claim 3.}
Under Assumptions \ref{ass: aligned intervention} and \ref{ass:non-degeneracy}, let $C\in\cC_k$ satisfy $C\notin \psi(\mathcal I_{t-1})$ and
\[
anc_k^\rblk(C)\subseteq \psi(\mathcal I_{t-1}).
\]
Let $J=\psi^{-1}(C)\subseteq R_{t-1}$. Then
\[
\dim\big(\cP_k(J;W_t)\big)=|J|=|C|.
\]

\smallskip
We postpone the proofs of the three claims and first show how they imply correctness of Steps 4--9.

Since $\psi(\mathcal I_{t-1})$ is ancestor-closed by (IH1), there exists at least one true block
\[
C^\star \in \cC_k\setminus \psi(\mathcal I_{t-1})
\]
whose true block ancestors are all contained in $\psi(\mathcal I_{t-1})$; equivalently, $C^\star$ is minimal in the remaining block partial order. By Claim~3, if $J^\star=\psi^{-1}(C^\star)$, then
\[
\dim\big(\cP_k(J^\star;W_t)\big)=|J^\star|.
\]
Hence the inner loop of Algorithm~2 must terminate.

Now let $\hat C$ be the set selected at Step~8 when the inner loop terminates. By construction, $\hat C$ is chosen with the smallest cardinality $r$ such that
\[
\dim\big(\cP_k(\hat C;W_t)\big)=r.
\]
Claim~1 rules out the possibility that $\psi(\hat C)$ contains only part of some true block. Claim~2 rules out the possibility that $\psi(\hat C)$ is a union of whole true blocks whose true block parents have not all been recovered already. Therefore, $\psi(\hat C)$ must be exactly one whole true block, say $C_t$, and moreover
\[
pa_k^\rblk(C_t)\subseteq \psi(\mathcal I_{t-1}).
\]
Since $\psi(\mathcal I_{t-1})$ is ancestor-closed by (IH1), this further implies
\[
anc_k^\rblk(C_t)\subseteq \psi(\mathcal I_{t-1}).
\]

\medskip
\noindent\textbf{Step 2: identification of the ancestor set in Algorithm \ref{alg:block ancestor}.}
We next analyze the call
\[
(\{\hat{\bq}_{k,i}\}_{i\in\hat C},\mathcal A)
=
\mathrm{ID\text{-}Block\text{-}Ancestor}
\Big(
\hat C,\{\Theta_k^{(j)}\}_{j\in\{0\}\cup\hat C},\mathcal I_{t-1},\{\hat{\bq}_{k,i}\}_{i\in\mathcal I_{t-1}}
\Big).
\]

For each block $B\in \mathcal I_{t-1}$, Algorithm~1 removes $B$ temporarily and checks whether
\[
\dim\big(\cP_k(\hat C;W_{-B})\big)=|\hat C|,
\qquad
W_{-B}\triangleq \langle \hat{\bq}_{k,i}: i\in \mathcal I_{t-1}\backslash B\rangle.
\]
We use Claim~2 again. If $B$ is a true block parent of $C_t$, then after removing $B$, the projected space still retains an unremoved parent contribution, hence
\[
\dim\big(\cP_k(\hat C;W_{-B})\big)>|\hat C|,
\]
so $B$ cannot be deleted from $\mathcal A$. Therefore,
\[
pa_k^\rblk(C_t)\subseteq \mathcal A.
\]

We also record one additional claim.

\smallskip
\noindent\emph{Claim 4.} Under Assumptions \ref{ass: aligned intervention} and \ref{ass:non-degeneracy}, 
if $B\in \mathcal I_{t-1}$ is not a true block ancestor of $C_t$, then
\[
\dim\big(\cP_k(\hat C;W_{-B})\big)=|\hat C|,
\]
and hence Algorithm~1 removes $B$ from $\mathcal A$.

\smallskip
Again postponing its proof, Claim~4 implies
\[
\mathcal A \subseteq anc_k^\rblk(C_t).
\]
Combining this with the previous inclusion yields
\[
pa_k^\rblk(C_t)\subseteq \mathcal A \subseteq anc_k^\rblk(C_t).
\]

Since $\psi(\mathcal I_{t-1})$ is ancestor-closed by (IH1), every true ancestor of every block in $\mathcal I_{t-1}$ has already been recovered. Therefore, once Algorithm~1 returns $\mathcal A$, the update rule in Step~14 of Algorithm~2 correctly inserts all block-level ancestral relations between the newly found block $C_t$ and the previously recovered true ancestor blocks. In particular, the recovered partial order remains correct after adding $\hat C$.

\medskip
\noindent\textbf{Step 3: update of the block row space.}
Let
\[
W \;\triangleq\; \langle \hat{\bq}_{k,i}: i\in \mathcal A\rangle .
\]
By the conclusion of Step~2, $W$ is exactly the span of the previously recovered block row spaces corresponding to true block ancestors of $C_t$. Hence
\[
\cP_k(\hat C;W)
\]
is precisely the innovation subspace associated with the block $C_t$. By the uniqueness statement for the block-RQ decomposition at the level of block row spaces, this innovation subspace is uniquely determined by $\bH_k$, the true block order, and the true ancestor structure. Therefore, the normalized vectors $\{\hat{\bq}_{k,i}\}_{i\in\hat C}$ constructed at Step~10 of Algorithm~1 span exactly the block row space corresponding to $C_t$ in some representation of the form $\bS_k\bP_{\sigma,k}\bQ_k$.

After Step~15 of Algorithm~2, the last $t$ recovered block rows of $\hat{\bQ}_k$ therefore coincide with the last $t$ block rows of some $\bS_k\bP_{\sigma,k}\bQ_k$. This proves (IH2) for the next iteration. Moreover, after Step~16, the set $\mathcal I_t$ consists of correctly recovered blocks and remains ancestor-closed under the true block-level partial order, proving (IH1) for the next iteration.

This completes the induction. After $T$ iterations, all true blocks of client $k$ have been recovered, the recovered block-level ancestral relation is correct, and the recovered $\hat{\bQ}_k$ agrees blockwise with $\bQ_k$ up to the intrinsic permutation and blockwise invertible transformation ambiguity. Hence the conclusion of the theorem follows.
\end{proof}

\subsection{Proof of \emph{Claim 1.}}
\begin{proof}
Let
\[
J=\{i_1,\dots,i_r\}\subseteq R_{t-1},
\]
and suppose that the corresponding targets involve exactly the true blocks
\[
C_{k,1},\dots,C_{k,q}\in \cC_k,
\]
with at least one of these blocks only partially selected by $J$. For each $s\in[q]$, let $p_s$ denote the length of the shortest block-parent path from $C_{k,s}$ to a root block in the true block partial order. Among all partially selected blocks, choose one, denoted by $C_{k,u}$, such that $p_u$ is maximal.

Define
\[
\cB_1
\;\triangleq\;
\{j\in J:\ C_k(j)=C_{k,s}\ \text{for some } s \text{ with } p_s\le p_u\},
\qquad
\cB_2
\;\triangleq\;
J\setminus \cB_1.
\]
Then $C_{k,u}$ is involved in $\cB_1$, every block involved in $\cB_2$ is fully selected, and no block involved in $\cB_1$ is a descendant of $C_{k,u}$. We will show
\[
\dim\big(\cP_k(\cB_1;W_t)\big)>|\cB_1|,
\qquad
\dim\Big(\Pi_{W'^\perp}\big(\cP_k(\cB_2;W_t)\big)\Big)\ge |\cB_2|
\]
for a suitable subspace $W'$ containing $\cP_k(\cB_1;W_t)$. This will imply
\[
\dim\big(\cP_k(J;W_t)\big)>|\cB_1|+|\cB_2|=r.
\]

\medskip
\noindent\textbf{Step 1: the contribution of $\cB_1$.}
By the induction hypothesis (IH2), the subspace
\[
W_t=\langle \hat{\bq}_{k,i}: i\in \mathcal I_{t-1}\rangle
\]
coincides with the span of the corresponding true block rows of $\bQ_k$, up to the intrinsic blockwise invertible transformation. Hence, for the purpose of dimension counting, we may work directly with the corresponding block rows of $\bQ_k$ and suppress this equivalence in the notation.

Set
\[
S_1
\;\triangleq\;
\rho_k\!\Big(
\bigcup_{j\in \cB_1} Anc_k^\rblk(C_k(j))
\setminus \psi(\mathcal I_{t-1})
\Big).
\]
By Lemma~\ref{lem:projected-rowspace-restriction},
\[
\cP_k(\cB_1;W_t)
\]
is contained in the span of the vectors
\[
[\bar{\bH}_k]_{S_1,:}^{T}
(\bI-[\bA_k]_{S_1,S_1})^{T}
[\bSigma_k]_{S_1,S_1}^{-1}
[\be_{\rho_k(t_j)}]_{S_1},
\]
and
\[
[\bar{\bH}_k]_{S_1,:}^{T}
\Big(
(\bI-[\bA_k]_{S_1,S_1})^{T}
[\bSigma_k]_{S_1,S_1}^{-1}
[\bu_k^{(j)}]_{S_1}
+
[\bd_k^{(j)}]_{S_1}
\Big),
\qquad j\in \cB_1.
\]
Moreover, by the block-RQ property, the projected rows $[\bar{\bH}_k]_{S_1,:}$ remain row independent on the residual innovation subspace. Therefore it suffices to prove that
\begin{align}
\dim\Big\langle\,
&(\bI-[\bA_k]_{S_1,S_1})^{T}
[\bSigma_k]_{S_1,S_1}^{-1}
[\be_{\rho_k(t_j)}]_{S_1},
\notag\\
&
(\bI-[\bA_k]_{S_1,S_1})^{T}
[\bSigma_k]_{S_1,S_1}^{-1}
[\bu_k^{(j)}]_{S_1}
+
[\bd_k^{(j)}]_{S_1}
\ \Big|\
j\in \cB_1
\Big\rangle
>|\cB_1|.
\label{eq:claim1-step1-goal}
\end{align}

The first family,
\[
(\bI-[\bA_k]_{S_1,S_1})^{T}
[\bSigma_k]_{S_1,S_1}^{-1}
[\be_{\rho_k(t_j)}]_{S_1},
\qquad j\in \cB_1,
\]
already consists of $|\cB_1|$ linearly independent vectors, since the standard basis vectors $[\be_{\rho_k(t_j)}]_{S_1}$ are distinct and
\[
(\bI-[\bA_k]_{S_1,S_1})^{T}
[\bSigma_k]_{S_1,S_1}^{-1}
\]
is invertible.

It remains to show that at least one vector from the second family is not in the span of the first family. Since $C_{k,u}$ is only partially selected by $J$, there exists an index
\[
b\in \rho_k(C_{k,u})
\]
such that $b$ is not the target coordinate of any intervention in $\cB_1\cap \psi^{-1}(C_{k,u})$. By construction of $\cB_1$, no descendant block of $C_{k,u}$ is involved in $\cB_1$. Hence, after reordering the blocks in $S_1$ so that $C_{k,u}$ appears first, the block upper-triangular structure of $\bA_k$ and the block-diagonal structure of $\bSigma_k$ imply that for every $j\in \cB_1\cap \psi^{-1}(C_{k,u})$,
\begin{align}
&\Big[
[\bSigma_k]_{S_1,S_1}
(\bI-[\bA_k]_{S_1,S_1}^{T})^{-1}
[\bd_k^{(j)}]_{S_1}
+
[\bu_k^{(j)}]_{S_1}
\Big]_b
\notag\\
&\qquad=
\Big[
[\bSigma_k]_{\rho_k(C_{k,u}),\rho_k(C_{k,u})}
(\bI-[\bA_k]_{\rho_k(C_{k,u}),\rho_k(C_{k,u})}^{T})^{-1}
[\bd_k^{(j)}]_{\rho_k(C_{k,u})}
+
[\bu_k^{(j)}]_{\rho_k(C_{k,u})}
\Big]_b.
\label{eq:claim1-localize}
\end{align}
By the third clause of Assumption~\ref{ass:non-degeneracy}, there exists some
\[
j^\star \in \cB_1\cap \psi^{-1}(C_{k,u})
\]
such that the quantity in \eqref{eq:claim1-localize} is nonzero at the unselected index $b$.

Now suppose, for contradiction, that
\[
(\bI-[\bA_k]_{S_1,S_1})^{T}
[\bSigma_k]_{S_1,S_1}^{-1}
[\bu_k^{(j^\star)}]_{S_1}
+
[\bd_k^{(j^\star)}]_{S_1}
\]
lies in the span of the first $|\cB_1|$ vectors. Left-multiplying by the invertible matrix
\[
[\bSigma_k]_{S_1,S_1}
(\bI-[\bA_k]_{S_1,S_1}^{T})^{-1}
\]
would imply that
\[
[\bu_k^{(j^\star)}]_{S_1}
+
[\bSigma_k]_{S_1,S_1}
(\bI-[\bA_k]_{S_1,S_1}^{T})^{-1}
[\bd_k^{(j^\star)}]_{S_1}
\]
lies in the span of the basis vectors $[\be_{\rho_k(t_j)}]_{S_1}$ for $j\in \cB_1$. But every such basis vector vanishes at the unselected coordinate $b$, whereas by construction and Assumption~\ref{ass:non-degeneracy}, the vector above has a nonzero $b$-th entry. This is impossible. Therefore \eqref{eq:claim1-step1-goal} holds, and hence
\[
\dim\big(\cP_k(\cB_1;W_t)\big)>|\cB_1|.
\]

\medskip
\noindent\textbf{Step 2: the contribution of $\cB_2$.}
Define
\[
W'
\;\triangleq\;
\Big\langle
[\bQ_k]_{i,:}
\ \Big|\
i\in
\rho_k\!\Big(
\bigcup_{j\in \cB_1} Anc_k^\rblk(C_k(j))
\Big)
\cup \mathcal I_{t-1}
\Big\rangle .
\]
By construction, $\cP_k(\cB_1;W_t)\subseteq W'$, and therefore
\[
\Pi_{W'^\perp}\big(\cP_k(\cB_2;W_t)\big)=\cP_k(\cB_2;W').
\]
Let
\[
S_2
\;\triangleq\;
\rho_k\!\Big(
\bigcup_{j\in \cB_2} Anc_k^\rblk(C_k(j))
\setminus
\Big(
\bigcup_{j\in \cB_1} Anc_k^\rblk(C_k(j))
\cup \psi(\mathcal I_{t-1})
\Big)
\Big).
\]
Applying Lemma~\ref{lem:projected-rowspace-restriction} again, it suffices to study the vectors
\[
(\bI-[\bA_k]_{S_2,S_2})^{T}
[\bSigma_k]_{S_2,S_2}^{-1}
[\be_{\rho_k(t_j)}]_{S_2},
\qquad j\in \cB_2.
\]
Since the blocks involved in $\cB_2$ are all completely selected, these vectors are linearly independent; indeed, the corresponding basis vectors $[\be_{\rho_k(t_j)}]_{S_2}$ are distinct, and
\[
(\bI-[\bA_k]_{S_2,S_2})^{T}
[\bSigma_k]_{S_2,S_2}^{-1}
\]
is invertible. Hence
\[
\dim\Big(\Pi_{W'^\perp}\big(\cP_k(\cB_2;W_t)\big)\Big)\ge |\cB_2|.
\]

\medskip
\noindent\textbf{Step 3: conclusion.}
Since $\cP_k(\cB_1;W_t)\subseteq W'$, we have
\[
\dim\big(\cP_k(J;W_t)\big)
\ge
\dim\big(\cP_k(\cB_1;W_t)\big)
+
\dim\Big(\Pi_{W'^\perp}\big(\cP_k(\cB_2;W_t)\big)\Big).
\]
Combining the bounds established above gives
\[
\dim\big(\cP_k(J;W_t)\big)
>
|\cB_1|+|\cB_2|
=
r.
\]
This proves Claim~1.
\end{proof}

\subsection{Proof of \emph{Claim 2.}}
\begin{proof}
Let
\[
J=\{i_1,\dots,i_r\}\subseteq R_{t-1},
\]
and suppose that $\psi(J)$ is the disjoint union of whole true blocks
\[
C_{k,1},\dots,C_{k,q}\in\cC_k.
\]
Assume further that there exists a true block parent of one of these blocks which is not yet contained in $\psi(\mathcal I_{t-1})$ and is not among the selected blocks themselves. We will prove
\[
\dim\big(\cP_k(J;W_t)\big)>r.
\]

Set
\[
S
\;\triangleq\;
\rho_k\!\Big(
Anc_k^\rblk(C_{k,1}\cup\cdots\cup C_{k,q})
\setminus \psi(\mathcal I_{t-1})
\Big).
\]
By Lemma~\ref{lem:projected-rowspace-restriction}, it suffices to show that
\begin{align}
\dim\Big\langle\,
&(\bI-[\bA_k]_{S,S})^{T}
[\bSigma_k]_{S,S}^{-1}
[\be_{\rho_k(t_j)}]_{S},
\notag\\
&
(\bI-[\bA_k]_{S,S})^{T}
[\bSigma_k]_{S,S}^{-1}
[\bu_k^{(j)}]_{S}
+
[\bd_k^{(j)}]_{S}
\ \Big|\
j\in J
\Big\rangle
>r.
\label{eq:claim2-goal}
\end{align}
The first family,
\[
(\bI-[\bA_k]_{S,S})^{T}
[\bSigma_k]_{S,S}^{-1}
[\be_{\rho_k(t_j)}]_{S},
\qquad j\in J,
\]
already consists of $r$ linearly independent vectors, since the basis vectors $[\be_{\rho_k(t_j)}]_{S}$ are distinct and
\[
(\bI-[\bA_k]_{S,S})^{T}
[\bSigma_k]_{S,S}^{-1}
\]
is invertible. Therefore, to prove \eqref{eq:claim2-goal}, it is enough to show that at least one vector from the second family is not contained in the span of the first family.

Choose a block $C_{k,i}\in\{C_{k,1},\dots,C_{k,q}\}$ and a true block parent
\[
C_{k,j}\in Pa_k^\rblk(C_{k,i})
\]
such that
\[
C_{k,j}\notin \psi(\mathcal I_{t-1})\cup \{C_{k,1},\dots,C_{k,q}\}.
\]
We claim that there exist
\[
j^\star\in \psi^{-1}(C_{k,i})
\qquad\text{and}\qquad
b\in \rho_k(C_{k,j})
\]
such that
\begin{equation}
\Big(
[\bSigma_k]_{S,S}(\bI-[\bA_k]_{S,S}^{T})^{-1}[\bd_k^{(j^\star)}]_{S}
+
[\bu_k^{(j^\star)}]_{S}
\Big)_b
\neq 0.
\label{eq:claim2-key-nonzero}
\end{equation}

We prove this by contradiction. Suppose instead that for every
\[
j\in \psi^{-1}(C_{k,i})
\qquad\text{and every}\qquad
b\in \rho_k(C_{k,j}),
\]
one has
\begin{equation}
\Big(
[\bSigma_k]_{S,S}(\bI-[\bA_k]_{S,S}^{T})^{-1}[\bd_k^{(j)}]_{S}
+
[\bu_k^{(j)}]_{S}
\Big)_b
= 0.
\label{eq:claim2-contradiction-assumption}
\end{equation}

Given the block-diagonal structure of $\bSigma_k$ and the fact that $C_{k,j}$ is a block-level parent of $C_{k,i}$, it follows that 

\begin{align*}
    &\Big([\bSigma_k]_{S,S}(\bI-[\bA_k]_{S,S}^{T})^{-1}[\bd_k^{(j)}]_{S}
+
[\bu_k^{(j)}]_{S}
\Big)_b \\ =  &\Big([\bSigma_k]_{\rho_k(C_{k,j}),\rho_k(C_{k,j})}(\bI-[\bA_k]_{S,S}^{T})^{-1}_{\rho_k(C_{k,j}), \rho_k(C_{k,i}\cup C_{k,j})}[\bd_k^{(j)}]_{\rho_k(C_{k,i}\cup C_{k,j}}
+
[\bu_k^{(j)}]_{\rho_k(C_{k,i}\cup C_{k,j}})
\Big)_b.
\end{align*}

Enumerate the nodes in $C_{k,i}$ according to the true topological order within the block:
\[
v_1,\dots,v_m.
\]
We will show inductively that no node in $C_{k,j}$ can be an ancestor of any $v_s$, which contradicts the assumption that $C_{k,j}\in Pa_k^\rblk(C_{k,i})$.

For the first node $v_1$, using Lemma \ref{lem:support-du}, by the support characterization of $\bd_k^{(v_1)}$, there is no support in the coordinates corresponding to $\rho_k(C_{k,i})$, since $v_1$ has no parent inside $C_{k,i}$. Moreover, for $b\in \rho_k(C_{k,j})$, the term $[\bu_k^{(v_1)}]_b$ vanishes by the support description of $\bu_k^{(v_1)}$. Hence \eqref{eq:claim2-contradiction-assumption} implies
\[
[\bSigma_k]_{\rho_k(C_{k,j}),\rho_k(C_{k,j})}
(\bI-[\bA_k]_{S,S}^{T})^{-1}_{\rho_k(C_{k,j}),\rho_k(C_{k,j})}
[\bd_k^{(v_1)}]_{\rho_k(C_{k,j})}
=\mathbf 0.
\]
Since
\[
[\bSigma_k]_{\rho_k(C_{k,j}),\rho_k(C_{k,j})}
(\bI-[\bA_k]_{S,S}^{T})^{-1}_{\rho_k(C_{k,j}),\rho_k(C_{k,j})}
\]
is invertible, we obtain
\[
[\bd_k^{(v_1)}]_{\rho_k(C_{k,j})}=\mathbf 0.
\]
By Assumption~\ref{ass:non-degeneracy} (the first part concerning $\bd_k^{(j)}$), this implies that no node in $C_{k,j}$ is a parent of $v_1$. Hence no node in $C_{k,j}$ is an ancestor of $v_1$, and therefore
\begin{equation}
\big[(\bI-[\bA_k]_{S,S})^{-1}\big]_{\rho_k(C_{k,j}),\,\rho_k(v_1)}=0.
\label{eq:claim2-col-zero-base}
\end{equation}

Now suppose inductively that for some $s\ge 2$, we have already shown that for every $a<s$,
\[
\big[(\bI-[\bA_k]_{S,S})^{-1}\big]_{\rho_k(C_{k,j}),\,\rho_k(v_a)}=0.
\]
We consider the node $v_s$. By the support characterization of $\bd_k^{(v_s)}$, the coordinates of
\[
[\bd_k^{(v_s)}]_{\rho_k(C_{k,i})}
\]
can only lie on earlier nodes $v_1,\dots,v_{s-1}$ inside the same block. The only new within-block support that may appear relative to previous interventions is therefore on coordinates corresponding to $v_1,\dots,v_{s-1}$. By the induction hypothesis, the columns of $(\bI-[\bA_k]_{S,S})^{-1}$ indexed by these earlier nodes vanish on the row block $\rho_k(C_{k,j})$. Hence these within-block contributions do not affect the coordinates in $\rho_k(C_{k,j})$.

Therefore, for $b\in \rho_k(C_{k,j})$, the identity \eqref{eq:claim2-contradiction-assumption} again reduces to
\[
[\bSigma_k]_{\rho_k(C_{k,j}),\rho_k(C_{k,j})}
(\bI-[\bA_k]_{S,S}^{T})^{-1}_{\rho_k(C_{k,j}),\rho_k(C_{k,j})}
[\bd_k^{(v_s)}]_{\rho_k(C_{k,j})}=\mathbf 0.
\]
Since the matrix multiplying $[\bd_k^{(v_s)}]_{\rho_k(C_{k,j})}$ is invertible, we conclude
\[
[\bd_k^{(v_s)}]_{\rho_k(C_{k,j})}=\mathbf 0.
\]
Again by Assumption~\ref{ass:non-degeneracy}, no node in $C_{k,j}$ is a parent of $v_s$, and any node within $C_{k,i}$ that serves as an ancestor of $v_s$ possesses no ancestors in $C_{k,j}$. Hence, no node in $C_{k,j}$ is an ancestor of $v_s$. Therefore
\[
\big[(\bI-\bA_k)^{-1}\big]_{\rho_k(C_{k,j}),\,\rho_k(v_s)}=0.
\]
This completes the induction.

We have thus shown that no node in $C_{k,j}$ is an ancestor of any node in $C_{k,i}$, contradicting the assumption that $C_{k,j}\in pa_k^\rblk(C_{k,i})$. Hence the contradiction assumption \eqref{eq:claim2-contradiction-assumption} is false, and \eqref{eq:claim2-key-nonzero} must hold.

Now fix $j^\star$ and $b$ as in \eqref{eq:claim2-key-nonzero}. Since $b\in \rho_k(C_{k,j})$ and the block $C_{k,j}$ is not selected, every basis vector
\[
[\be_{\rho_k(t_j)}]_{S},
\qquad j\in J,
\]
vanishes at the coordinate $b$. Consequently,
\[
[\bSigma_k]_{S,S}(\bI-[\bA_k]_{S,S}^{T})^{-1}
[\bd_k^{(j^\star)}]_{S}
+
[\bu_k^{(j^\star)}]_{S}
\]
cannot belong to the span of $\{[\be_{\rho_k(t_j)}]_{S}:j\in J\}$. Equivalently, after multiplying back by the invertible matrix
\[
(\bI-[\bA_k]_{S,S})^{T}
[\bSigma_k]_{S,S}^{-1},
\]
the vector
\[
(\bI-[\bA_k]_{S,S})^{T}
[\bSigma_k]_{S,S}^{-1}
[\bu_k^{(j^\star)}]_{S}
+
[\bd_k^{(j^\star)}]_{S}
\]
cannot lie in the span of the first $r$ vectors in \eqref{eq:claim2-goal}. Therefore the whole family in \eqref{eq:claim2-goal} has rank strictly larger than $r$, and hence
\[
\dim\big(\cP_k(J;W_t)\big)>r.
\]
This proves Claim~2.
\end{proof}

\subsection{Proof of \emph{Claim 3.}}
\begin{proof}
Let
\[
\hat C=\{i_1,\dots,i_r\}\subseteq R_{t-1},
\]
and suppose that $\psi(\hat C)$ is exactly one whole true block, denoted by $C_{k,i}$. Assume moreover that
\[
anc_k^\rblk(C_{k,i})\subseteq \psi(\mathcal I_{t-1}).
\]
We show that
\[
\dim\big(\cP_k(\hat C;W_t)\big)=|\hat C|=r.
\]

Since all true block ancestors of $C_{k,i}$ have already been recovered at the beginning of iteration $t$, the residual ancestor set appearing in Lemma~\ref{lem:projected-rowspace-restriction} reduces exactly to the block $C_{k,i}$ itself. More precisely,
\[
S
\;\triangleq\;
\rho_k\!\Big(
Anc_k^\rblk(C_{k,i})\setminus \psi(\mathcal I_{t-1})
\Big)
=
\rho_k(C_{k,i}).
\]
Therefore, by Lemma~\ref{lem:projected-rowspace-restriction}, the space
\[
\cP_k(\hat C;W_t)
\]
is contained in the span of the vectors
\[
[\bar{\bH}_k]_{S,:}^{T}
(\bI-[\bA_k]_{S,S})^{T}
[\bSigma_k]_{S,S}^{-1}
[\be_{\rho_k(t_j)}]_{S},
\]
and
\[
[\bar{\bH}_k]_{S,:}^{T}
\Big(
(\bI-[\bA_k]_{S,S})^{T}
[\bSigma_k]_{S,S}^{-1}
[\bu_k^{(j)}]_{S}
+
[\bd_k^{(j)}]_{S}
\Big),
\qquad j\in\hat C.
\]

We first consider the first family. Since $\hat C$ contains exactly all interventions corresponding to the block $C_{k,i}$, the coordinates $\rho_k(t_j)$ for $j\in\hat C$ exhaust all indices in $S=\rho_k(C_{k,i})$. Hence the vectors
\[
[\be_{\rho_k(t_j)}]_{S},
\qquad j\in\hat C,
\]
form the standard basis of $\mathbb R^{|S|}=\mathbb R^r$. Because
\[
(\bI-[\bA_k]_{S,S})^{T}
[\bSigma_k]_{S,S}^{-1}
\]
is invertible, it follows that
\[
(\bI-[\bA_k]_{S,S})^{T}
[\bSigma_k]_{S,S}^{-1}
[\be_{\rho_k(t_j)}]_{S},
\qquad j\in\hat C,
\]
are $r$ linearly independent vectors. Since $[\bar{\bH}_k]_{S,:}$ has full row rank on the residual innovation subspace, the corresponding projected vectors
\[
[\bar{\bH}_k]_{S,:}^{T}
(\bI-[\bA_k]_{S,S})^{T}
[\bSigma_k]_{S,S}^{-1}
[\be_{\rho_k(t_j)}]_{S},
\qquad j\in\hat C,
\]
also span an $r$-dimensional space. Hence
\[
\dim\big(\cP_k(\hat C;W_t)\big)\ge r.
\]

On the other hand, every vector appearing in the second family above belongs to the same $r$-dimensional ambient coordinate space indexed by $S=\rho_k(C_{k,i})$. Equivalently, before multiplying by $[\bar{\bH}_k]_{S,:}^{T}$, all coefficient vectors lie in $\mathbb R^{S}$, whose dimension is exactly $|S|=r$. Since the first family already forms a basis of this space, each vector
\[
(\bI-[\bA_k]_{S,S})^{T}
[\bSigma_k]_{S,S}^{-1}
[\bu_k^{(j)}]_{S}
+
[\bd_k^{(j)}]_{S},
\qquad j\in\hat C,
\]
must be a linear combination of the first $r$ vectors
\[
(\bI-[\bA_k]_{S,S})^{T}
[\bSigma_k]_{S,S}^{-1}
[\be_{\rho_k(t_j)}]_{S},
\qquad j\in\hat C.
\]
After left multiplication by $[\bar{\bH}_k]_{S,:}^{T}$, the same linear dependence relation is preserved. Therefore, the entire projected family still spans a space of dimension at most $r$. Thus
\[
\dim\big(\cP_k(\hat C;W_t)\big)\le r.
\]

Combining the lower and upper bounds yields
\[
\dim\big(\cP_k(\hat C;W_t)\big)=r.
\]
This proves Claim~3.
\end{proof}

\subsection{Proof of \emph{Claim 4.}}
\begin{proof}
Let $\hat C$ be the block identified in Step~1 of the current iteration, and let $C_t=\psi(\hat C)$ be the corresponding true block. Let
\[
U_{\hat C}
\;\triangleq\;
\row\Big(\{\bTheta_k^{(j)}-\bTheta_k^{(0)}\}_{j\in \hat C}\Big).
\]
Fix a block $B\in \mathcal I_{t-1}$ such that
\[
B\notin Anc_k^\rblk(C_t).
\]
Recall that Algorithm~1 defines
\[
W_t=\langle \hat{\bq}_{k,i}: i\in \mathcal I_{t-1}\rangle,
\qquad
W_{-B}=\langle \hat{\bq}_{k,i}: i\in \mathcal I_{t-1}\backslash B\rangle.
\]
We will show that
\[
\dim\big(\cP_k(\hat C;W_{-B})\big)=|\hat C|.
\]

By the induction hypothesis (IH2), the recovered block row spaces in $\hat{\bQ}_k$ coincide with the corresponding true block row spaces of $\bQ_k$, up to the intrinsic permutation and blockwise invertible transformation ambiguity. Hence, for the purpose of dimension counting, we may identify $W_t$ and $W_{-B}$ with the corresponding spans of true block rows of $\bQ_k$.

Now define
\[
V_{anc}
\;\triangleq\;
\Big\langle
[\bQ_k]_{\rho_k(i),:}
\ \Big|\
i\in C',\ C'\in Anc_k^\rblk(C_t)
\Big\rangle.
\]
By Lemma~2, the perturbation row space generated by the interventions in $\hat C$ satisfies
\[
U_{\hat C}\subseteq V_{anc}.
\]

Consider the orthogonal projection maps
\[
T_t:U_{\hat C}\to W_t^\perp,
\qquad
x\mapsto \Pi_{W_t^\perp}(x),
\]
and
\[
T_{-B}:U_{\hat C}\to W_{-B}^\perp,
\qquad
x\mapsto \Pi_{W_{-B}^\perp}(x).
\]
By definition,
\[
\cP_k(\hat C;W_t)=\mathrm{Im}(T_t),
\qquad
\cP_k(\hat C;W_{-B})=\mathrm{Im}(T_{-B}).
\]
Therefore,
\[
\dim\big(\cP_k(\hat C;W_t)\big)=\dim(\mathrm{Im}(T_t)),
\qquad
\dim\big(\cP_k(\hat C;W_{-B})\big)=\dim(\mathrm{Im}(T_{-B})).
\]

We next identify the kernels of these two maps. Since orthogonal projection onto $W^\perp$ vanishes exactly on $W$, we have
\[
\ker(T_t)=U_{\hat C}\cap W_t,
\qquad
\ker(T_{-B})=U_{\hat C}\cap W_{-B}.
\]
Thus, by the rank-nullity theorem, it is enough to prove
\[
U_{\hat C}\cap W_t
=
U_{\hat C}\cap W_{-B}.
\]

To see this, note first that $W_{-B}\subseteq W_t$, so it suffices to show that removing the block $B$ does not remove any direction lying in $U_{\hat C}$. Since $B\notin Anc_k^\rblk(C_t)$, the block row space associated with $B$ does not enter the ancestor-generated subspace $V_{anc}$ for $C_t$. Hence
\[
W_t\cap V_{anc}=W_{-B}\cap V_{anc}.
\]
Because $U_{\hat C}\subseteq V_{anc}$, it follows that
\[
U_{\hat C}\cap W_t
=
U_{\hat C}\cap (W_t\cap V_{anc})
=
U_{\hat C}\cap (W_{-B}\cap V_{anc})
=
U_{\hat C}\cap W_{-B}.
\]
Therefore,
\[
\ker(T_t)=\ker(T_{-B}).
\]

Applying rank-nullity to $T_t$ and $T_{-B}$ gives
\[
\dim(\mathrm{Im}(T_t))
=
\dim(U_{\hat C})-\dim(\ker T_t)
=
\dim(U_{\hat C})-\dim(\ker T_{-B})
=
\dim(\mathrm{Im}(T_{-B})).
\]
Hence
\[
\dim\big(\cP_k(\hat C;W_t)\big)
=
\dim\big(\cP_k(\hat C;W_{-B})\big).
\]

Finally, by Claim~3,
\[
\dim\big(\cP_k(\hat C;W_t)\big)=|\hat C|.
\]
Therefore,
\[
\dim\big(\cP_k(\hat C;W_{-B})\big)=|\hat C|.
\]
This proves Claim~4.
\end{proof}

\section{Proof of Corollary \ref{cor:zk}}
\label{app:proof cor zk}

\begin{corollary}[Corollary \ref{cor:zk} in the main text]
Assume Assumptions~\ref{ass: aligned intervention} and~\ref{ass:non-degeneracy} hold. Let $\hat{\bQ}_k$ be the output of Algorithm~\ref{alg:causal order client}, and define
\[
\hat Z_k=\hat{\bQ}_kX_k.
\]
For each true block $C\in\cC_k$, let $\hat C=\psi^{-1}(C)$ denote the corresponding estimated block. Then, for every $i\in\hat C$, the estimated latent variable $\hat Z_{k,(i)}$ lies in the span of the true latent variables indexed by the block ancestors of $C$, namely
\[
\hat Z_{k,(i)}\in \mathrm{span}\{Z_{(j)}:j\in C', C'\in Anc_k^\rblk(C)\}.
\]
Thus, up to the relabeling induced by $\psi$, each estimated latent variable is identified up to block-level ancestor mixing.
\end{corollary}

\begin{proof}
By Theorem~\ref{thm:client-main}, there exist a block-diagonal invertible matrix $\bS_k$ and a block permutation matrix $\bP_{\sigma,k}$ such that
\[
\hat{\bQ}_k=\bS_k\bP_{\sigma,k}\bQ_k.
\]
Therefore,
\[
\hat{Z}_k
=
\hat{\bQ}_kX_k
=
\bS_k\bP_{\sigma,k}\bQ_kX_k.
\]

Next, recall that the block-RQ decomposition of $\bH_k$ is
\[
\bH_k=\bR_k\bQ_k.
\]
Since $\bH_kX_k=Z_k$, we have
\[
\bR_k\bQ_kX_k=Z_k.
\]
Because $\bR_k$ is invertible, it follows that
\[
\bQ_kX_k=\bR_k^{-1}Z_k.
\]
Substituting this into the previous display yields
\[
\hat{Z}_k
=
\bS_k\bP_{\sigma,k}\bR_k^{-1}Z_k.
\]

Now fix a block $C_{k,i}$. By Lemma~\ref{lem:Rk^-1},
\[
[\bR_k^{-1}]_{C_{k,i},C_{k,j}}=0
\qquad
\text{whenever } C_{k,j}\notin Anc_k^\rblk(C_{k,i}).
\]
Hence the block $(\bR_k^{-1}Z_k)_{C_{k,i}}$ depends only on the blocks $Z_{k,C_{k,j}}$ with
\[
C_{k,j}\in Anc_k^\rblk(C_{k,i}).
\]
Equivalently, $(\bR_k^{-1}Z_k)_{C_{k,i}}$ is a linear combination of the latent blocks associated with the block ancestors of $C_{k,i}$, including $C_{k,i}$ itself.

Finally, left multiplication by the block-diagonal invertible matrix $\bS_k$ only applies an invertible linear transformation within each recovered block, while $\bP_{\sigma,k}$ only permutes the block order. Therefore, neither operation changes the ancestor-restricted block dependence pattern; they only change the within-block coordinates and the block ordering. It follows that each recovered block $\hat{Z}_{k,(i)}$ is a linear combination of the true latent blocks corresponding to the block ancestors of its matched true block (including itself), exactly as claimed.
\end{proof}

\section{Proof of Theorem \ref{thm:global}}

\label{app:proof thm global}

\begin{theorem}[Theorem \ref{thm:global} in the main text]
Assume Assumptions~\ref{ass: aligned intervention}--\ref{ass: global cover edge} hold. Let $\{(\hat{\cC}_k,\hat{\prec}_k^\rblk)\}_{k=1}^K$ be the outputs of Algorithm~\ref{alg:causal order client} and the input to Algorithm~\ref{alg:global}. Let $\hat{\prec}$ be the output of Algorithm~\ref{alg:global}. Then $(\cV,\hat{\prec})$ recovers $(\cV,\prec)$ up to relabeling under the bijection $\psi:j\mapsto t_j$.
\end{theorem}

\begin{proof}
We prove that, under the bijection $\psi: j\mapsto t_j$, the output $\hat{\prec}$ of Algorithm~\ref{alg:global} coincides exactly with the true global ancestral relation $\prec$. It suffices to show the two inclusions
\[
\prec \subseteq \hat{\prec},
\qquad
\hat{\prec} \subseteq \prec.
\]

We first prove $\prec \subseteq \hat{\prec}$. Take any two global nodes $u,v\in\cV$ such that
\[
u\prec v.
\]
If $u=v$, then the claim is trivial. Assume $u\neq v$. Since $\prec$ is the ancestral relation of the global DAG, there exists a directed path from $u$ to $v$. Equivalently, there exists a sequence of nodes
\[
u=v_0 \to v_1 \to \cdots \to v_m=v
\]
such that each edge $v_{r-1}\to v_r$ is a cover edge in the Hasse diagram of $\prec$, namely there is no node $w$ satisfying
\[
v_{r-1}\prec w\prec v_r.
\]

Fix any such cover edge $v_{r-1}\to v_r$. By Assumption~\ref{ass: global cover edge}, there exists some client $k$ such that both $v_{r-1}$ and $v_r$ appear as singleton blocks in the client-specific block structure, and moreover
\[
\{v_{r-1}\}\in pa_k^\rblk(\{v_r\}).
\]
By Theorem \ref{thm:client-main}, Algorithm \ref{alg:causal order client} recovers the client-level block structure and block-level ancestral relation correctly. Hence its output satisfies
\[
\{v_{r-1}\}\in \hat{pa}_k^\rblk(\{v_r\}).
\]
Therefore, Algorithm~\ref{alg:global} adds
\[
v_{r-1}\Tilde{\prec} v_r
\]
to $\Tilde{\prec}$ at Lines 2--3. Since this holds for every cover edge along the above path, we obtain
\[
v_0\Tilde{\prec} v_1,\quad
v_1\Tilde{\prec} v_2,\quad \dots,\quad
v_{m-1}\Tilde{\prec} v_m.
\]
Algorithm~\ref{alg:global} then defines $\hat{\prec}$ as the transitive closure of $\Tilde{\prec}$. Hence
\[
u=v_0 \hat{\prec} v_1 \hat{\prec}\cdots \hat{\prec} v_m=v,
\]
which implies
\[
u\hat{\prec} v.
\]
Since $u\prec v$ was arbitrary, we conclude that
\[
\prec \subseteq \hat{\prec}.
\]

Next we prove $\hat{\prec}\subseteq \prec$. Take any two nodes $u,v\in\cV$ such that
\[
u\hat{\prec} v.
\]
Since $\hat{\prec}$ is the transitive closure of $\Tilde{\prec}$, there exists a sequence of nodes
\[
u=v_0,\ v_1,\ \dots,\ v_m=v
\]
such that
\[
v_{r-1}\Tilde{\prec} v_r,\qquad r=1,\dots,m.
\]
It therefore suffices to show that every relation inserted into $\Tilde{\prec}$ is a true global ancestral relation.

Fix any pair $a,b\in\cV$ such that
\[
a\Tilde{\prec} b.
\]
By the construction of Algorithm~\ref{alg:global}, there exists some client $k$ such that
\[
\{a\},\{b\}\in \hat{\cC}_k,
\qquad
\{a\}\in \hat{pa}_k^\rblk(\{b\}).
\]
Again by Theorem \ref{thm:client-main}, the estimated client-level structure is correct, so in the true client-specific structure we also have
\[
\{a\}\in pa_k^\rblk(\{b\}).
\]
Since both blocks are singletons, Proposition~\ref{prop:block anc} implies that this block-level parent relation corresponds to a genuine node-level parent-child relation in client $k$. In particular,
\[
a\prec_k b,
\]
where $\prec_k$ denotes the node-level ancestral relation in the induced graph for client $k$. On the other hand, the proof of Proposition \ref{prop:margin Z} shows that $\prec_k$ is exactly the restriction of the global ancestral relation $\prec$ to $O_k$. Therefore
\[
a\prec_k b \quad\Longrightarrow\quad a\prec b.
\]
This proves that
\[
\Tilde{\prec}\subseteq \prec.
\]
Since $\prec$ is transitive and $\hat{\prec}$ is the transitive closure of $\Tilde{\prec}$, it follows immediately that
\[
\hat{\prec}\subseteq \prec.
\]

Combining the two inclusions yields
\[
\hat{\prec}=\prec.
\]
Therefore, under the bijection $\psi:j\mapsto t_j$, the recovered structure $(\cV,\hat{\prec})$ coincides exactly with the true global ancestral structure $(\cV,\prec)$, up to relabeling.
\end{proof}

\section{Auxiliary Lemmas}
\label{app: auxiliary lemmas}

\begin{lemma}
\label{lem:HKH}
Let $X = \bG Z$, where $\bG \in \mathbb{R}^{n\times p}$ has full column rank. Assume that $Cov(Z)$ is nonsingular, and define
\[
\bK \triangleq Cov(Z)^{-1}, 
\qquad 
\bH \triangleq \bG^\dagger .
\]
Then
\[
Cov(X)^\dagger = \bH^T \bK \bH.
\]
\end{lemma}

\begin{proof}
Since $X=\bG Z$, its covariance matrix is
\[
Cov(X)=\bG\,Cov(Z)\,\bG^T.
\]
Because $\bG$ has full column rank, $\bG^T$ has full row rank. Moreover, $Cov(Z)$ is invertible by assumption. According to \cite{greville1966note}, we can apply the standard pseudoinverse identity for a matrix product to obtain
\[
Cov(X)^\dagger
=
\bigl(G\,Cov(Z)\,G^T\bigr)^\dagger
=
(G^T)^\dagger Cov(Z)^\dagger G^\dagger.
\]
Using $Cov(Z)^\dagger=Cov(Z)^{-1}=\bK$ and $(\bG^T)^\dagger=(\bG^\dagger)^T=\bH^T$, we conclude that
\[
Cov(X)^\dagger
=
\bH^T \bK \bH.
\]
\end{proof}

\begin{lemma}[Support of $\bd_k^{(j)}$ and $\bu_k^{(j)}$]
\label{lem:support-du}
Under the setup of Proposition~\ref{prop:margin Z}, for each client $k\in[K]$ and intervention $j\in\cJ_k$ targeting $t_j\in O_k$, let $r=\rho_k(t_j)$. Then the vectors $\bd_k^{(j)}$ and $\bu_k^{(j)}$ defined in Assumption~\ref{ass:non-degeneracy} satisfy
\[
\supp(\bd_k^{(j)}) \subseteq \rho_k\bigl(pa_k(t_j)\bigr),
\qquad
\supp(\bu_k^{(j)}) \subseteq \rho_k\bigl(sp_k(t_j)\bigr).
\]
Equivalently, for any $v\in O_k$,
\[
v\notin pa_k(t_j)\quad \Longrightarrow\quad [\bd_k^{(j)}]_{\rho_k(v)}=0,
\qquad
v\notin sp_k(t_j)\quad \Longrightarrow\quad [\bu_k^{(j)}]_{\rho_k(v)}=0.
\]
\end{lemma}

\begin{proof}
Fix $k\in[K]$ and $j\in\cJ_k$. For notational simplicity, write
\[
O=O_k,\qquad M=M_k,\qquad t=t_j,\qquad r=\rho_k(t),
\]
and let
\[
\bB \coloneqq (\bI-[\bA]_{M,M})^{-1}.
\]

Since intervention $j$ acts only on the structural equation of node $t$ in the global latent SEM, after partitioning according to $(O,M)$, the only changed blocks of the structural matrix are the $r$-th row blocks in $[\,\cdot\,]_{O,O}$ and $[\,\cdot\,]_{O,M}$. More precisely,
\[
[\widetilde{\bA}^{(j)}]_{O,O}
=
[\bA]_{O,O}
+
\be_r[\Delta\ba^{(j)}]_{O}^{\top},
\qquad
[\widetilde{\bA}^{(j)}]_{O,M}
=
[\bA]_{O,M}
+
\be_r[\Delta\ba^{(j)}]_{M}^{\top},
\]
whereas
\[
[\widetilde{\bA}^{(j)}]_{M,O}=[\bA]_{M,O},
\qquad
[\widetilde{\bA}^{(j)}]_{M,M}=[\bA]_{M,M}.
\]

By the expression of the induced client-specific SEM in Proposition~\ref{prop:margin Z},
\[
\bA_k
=
[\bA]_{O,O}
+
[\bA]_{O,M}(\bI-[\bA]_{M,M})^{-1}[\bA]_{M,O},
\]
and similarly,
\[
\bA_k^{(j)}
=
[\widetilde{\bA}^{(j)}]_{O,O}
+
[\widetilde{\bA}^{(j)}]_{O,M}
(\bI-[\widetilde{\bA}^{(j)}]_{M,M})^{-1}
[\widetilde{\bA}^{(j)}]_{M,O}.
\]
Substituting the above block relations yields
\begin{align*}
\bA_k^{(j)}
&=
\bA_k
+
\be_r\Big(
[\Delta\ba^{(j)}]_{O}^{\top}
+
[\Delta\ba^{(j)}]_{M}^{\top}\bB[\bA]_{M,O}
\Big).
\end{align*}
Hence
\[
\bA_k^{(j)}=\bA_k+\be_r(\bd_k^{(j)})^{\top},
\qquad
(\bd_k^{(j)})^{\top}
=
[\Delta\ba^{(j)}]_{O}^{\top}
+
[\Delta\ba^{(j)}]_{M}^{\top}\bB[\bA]_{M,O}.
\]
Therefore, $\bd_k^{(j)}$ parameterizes the perturbation of the $r$-th row of $\bA_k$, namely the perturbation of the incoming directed coefficients of node $t$ in the induced client graph. By definition of the induced graph $\cG_k$, for any $v\in O_k$,
\[
v\in pa_k(t)
\quad\Longleftrightarrow\quad
[\bA_k]_{r,\rho_k(v)}\neq 0.
\]
It follows that an admissible nonzero entry of $\bd_k^{(j)}$ can only appear at coordinates indexed by $\rho_k(pa_k(t))$, i.e.,
\[
\supp(\bd_k^{(j)})\subseteq \rho_k\bigl(pa_k(t)\bigr).
\]

Next consider $\bu_k^{(j)}$. By Proposition~\ref{prop:margin Z}, the induced noise covariance satisfies
\[
\bSigma_k
=
\bI
+
[\bA]_{O,M}(\bI-[\bA]_{M,M})^{-1}
(\bI-[\bA]_{M,M})^{-T}[\bA]_{O,M}^{\top}.
\]
Under intervention $j$, only the $r$-th row of $[\bA]_{O,M}$ is perturbed, while the $M\times M$ block remains unchanged. Therefore the perturbation of $\bSigma_k$ has the form
\[
\bSigma_k^{(j)}
=
\bSigma_k
+
\be_r(\bu_k^{(j)})^{\top}
+
\bu_k^{(j)}\be_r^{\top}
+
\eta_k^{(j)}\be_r\be_r^{\top},
\]
for some scalar $\eta_k^{(j)}$, where
\[
\bu_k^{(j)}
=
[\bA]_{O,M}
(\bI-[\bA]_{M,M})^{-1}
(\bI-[\bA]_{M,M})^{-T}
[\Delta\ba^{(j)}]_{M}.
\]
Thus $\bu_k^{(j)}$ parameterizes the off-diagonal perturbation in the $r$-th row and column of $\bSigma_k$, namely the perturbation of the bidirected part incident to node $t$ in the induced mixed graph. By definition of $\cG_k$, for any $v\in O_k$,
\[
v\in sp_k(t)
\quad\Longleftrightarrow\quad
[\bSigma_k]_{r,\rho_k(v)}\neq 0
\quad\Longleftrightarrow\quad
[\bSigma_k]_{\rho_k(v),r}\neq 0.
\]
Hence a nonzero entry of $\bu_k^{(j)}$ can only appear at coordinates indexed by $\rho_k(sp_k(t))$, and therefore
\[
\supp(\bu_k^{(j)})\subseteq \rho_k\bigl(sp_k(t)\bigr).
\]

Combining the two parts proves the claim.
\end{proof}

\begin{lemma}\label{lem:projected-rowspace-restriction}
Fix a client $k\in[K]$. Let $\cI\subseteq \cC_k$ be a collection of blocks that is closed under block ancestors, namely
\[
C\in \cI,\ C'\in Anc_k^\rblk(C)\quad \Longrightarrow\quad C'\in \cI.
\]
Define
\[
W \;\triangleq\; \Big\langle [\bQ_k]_{\rho_k(i),:}\ \Big|\ i\in \bigcup_{C\in\cI} C \Big\rangle .
\]
Let $j_1,\dots,j_\ell$ be intervention indices such that
\[
C_k(j_r)\notin \cI,\qquad r=1,\dots,\ell,
\]
where $C_k(j_r)$ denotes the block in $\cC_k$ containing the target $t_{j_r}$. Define
\[
\cS_{\cI}
\;\triangleq\;
\Big( \bigcup_{r=1}^\ell Anc_k^\rblk\big(C_k(j_r)\big)\Big)\setminus \cI,
\qquad
S
\;\triangleq\;
\rho_k\!\Big(\bigcup_{C\in\cS_{\cI}} C\Big).
\]
For each $s\in S$, let $[\bar{\bH}_k]_{s,:}$ be the projection of $[\bH_k]_{s,:}$ onto $W^\perp$, and let $[\bar{\bH}_k]_{S,:}$ be the matrix formed by stacking these projected rows.

Then
\[
\Pi_{W^\perp}\!\left(
\row\Big(\{\bTheta_k^{(j)}-\bTheta_k^{(0)}\}_{j\in\{j_1,\dots,j_\ell\}}\Big)
\right)
\]
is contained in
\[
\Big\langle
[\bar{\bH}_k]_{S,:}^{T}
(\bI-[\bA_k]_{S,S})^{T}
[\bSigma_k]_{S,S}^{-1}
[\be_{\rho_k(t_j)}]_{S},
\ 
[\bar{\bH}_k]_{S,:}^{T}
\Big(
(\bI-[\bA_k]_{S,S})^{T}
[\bSigma_k]_{S,S}^{-1}
[\bu_k^{(j)}]_{S}
+
[\bd_k^{(j)}]_{S}
\Big)
\ \Big|\ 
j\in\{j_1,\dots,j_\ell\}
\Big\rangle .
\]
\end{lemma}

\begin{proof}
For each intervention index $j$, define
\[
\balpha_k^{(j)}
\;\triangleq\;
(\bI-\bA_k)^T\bSigma_k^{-1}\be_{\rho_k(t_j)},
\qquad
\bbeta_k^{(j)}
\;\triangleq\;
(\bI-\bA_k)^T\bSigma_k^{-1}\bu_k^{(j)}+\bd_k^{(j)}.
\]
By Lemma~\ref{lem: rank 2},
\[
\row(\bTheta_k^{(j)}-\bTheta_k^{(0)})
\subseteq
\langle \bH_k^T\balpha_k^{(j)},\ \bH_k^T\bbeta_k^{(j)} \rangle .
\]
Hence
\[
\row\Big(\{\bTheta_k^{(j)}-\bTheta_k^{(0)}\}_{j\in\{j_1,\dots,j_\ell\}}\Big)
\subseteq
\Big\langle
\bH_k^T\balpha_k^{(j)},\ \bH_k^T\bbeta_k^{(j)}
\ \Big|\ 
j\in\{j_1,\dots,j_\ell\}
\Big\rangle .
\]
Since orthogonal projection is linear, it suffices to analyze
\[
\Pi_{W^\perp}(\bH_k^T\balpha_k^{(j)})
\qquad\text{and}\qquad
\Pi_{W^\perp}(\bH_k^T\bbeta_k^{(j)})
\]
for each $j\in\{j_1,\dots,j_\ell\}$.

Now fix such a $j$, and write
\[
T_j
\;\triangleq\;
\rho_k\!\Big(
\bigcup_{C\in Anc_k^\rblk(C_k(j))} C
\Big).
\]
By the support characterization of $\bu_k^{(j)}$ and $\bd_k^{(j)}$ in Lemma \ref{lem:support-du}, together with the block upper-triangular structure of $\bA_k$ and the block-diagonal structure of $\bSigma_k$, we have
\[
\supp(\be_{\rho_k(t_j)})\subseteq T_j,\qquad
\supp(\bu_k^{(j)})\subseteq T_j,\qquad
\supp(\bd_k^{(j)})\subseteq T_j,
\]
and therefore
\[
\supp(\balpha_k^{(j)})\subseteq T_j,
\qquad
\supp(\bbeta_k^{(j)})\subseteq T_j.
\]
Consequently,
\[
\bH_k^T\balpha_k^{(j)}
=
[\bH_k]_{T_j,:}^T[\balpha_k^{(j)}]_{T_j},
\qquad
\bH_k^T\bbeta_k^{(j)}
=
[\bH_k]_{T_j,:}^T[\bbeta_k^{(j)}]_{T_j}.
\]
Thus, rows of $\bH_k$ indexed outside $T_j$ do not contribute at all, regardless of whether we project onto $W^\perp$.

Next, because $\cI$ is closed under block ancestors and
\[
W=\Big\langle [\bQ_k]_{\rho_k(i),:}\ \Big|\ i\in \bigcup_{C\in\cI} C \Big\rangle,
\]
the block-RQ construction implies that every row of $\bH_k$ indexed by
\[
\rho_k\!\Big(\bigcup_{C\in\cI} C\Big)
\]
already lies in $W$. Hence these rows vanish after projection onto $W^\perp$. On the other hand, for rows indexed by $T_j\setminus \rho_k(\cup_{C\in\cI}C)$, the projection may change the row vectors, but only through their projections onto $W^\perp$. Therefore,
\[
\Pi_{W^\perp}(\bH_k^T\balpha_k^{(j)})
=
[\bar{\bH}_k]_{S,:}^T[\balpha_k^{(j)}]_S,
\qquad
\Pi_{W^\perp}(\bH_k^T\bbeta_k^{(j)})
=
[\bar{\bH}_k]_{S,:}^T[\bbeta_k^{(j)}]_S.
\]
In words, the orthogonal projection acts only on the $\bH_k$-side of the generators from Lemma~\ref{lem: rank 2}, while the coefficient vectors remain unchanged. Moreover, rows indexed by blocks in $\cI$ vanish after projection because they already lie in the previously extracted subspace $W$, whereas rows outside the relevant block ancestors never contribute in the first place by the support characterization above.

It remains to identify the restricted coefficient vectors on $S$. We claim that
\[
[\balpha_k^{(j)}]_S
=
(\bI-[\bA_k]_{S,S})^T[\bSigma_k]_{S,S}^{-1}[\be_{\rho_k(t_j)}]_S,
\]
and
\[
[\bbeta_k^{(j)}]_S
=
(\bI-[\bA_k]_{S,S})^T[\bSigma_k]_{S,S}^{-1}[\bu_k^{(j)}]_S
+
[\bd_k^{(j)}]_S.
\]

Indeed, by construction of $S$, the vectors $\be_{\rho_k(t_j)}$, $\bu_k^{(j)}$, and $\bd_k^{(j)}$ have no support outside the union of blocks in $Anc_k^\rblk(C_k(j))$, and their coordinates on blocks in $\cI$ do not contribute after projection onto $W^\perp$. Moreover, $\bSigma_k$ is block diagonal, so applying $\bSigma_k^{-1}$ does not mix coordinates across distinct blocks. In addition, $\bA_k$ is block upper triangular with respect to the block topological order, so for rows indexed by $S$, the matrix $(\bI-\bA_k)^T$ only collects contributions from the same block or its block ancestors. Since $S$ already contains precisely those relevant ancestor blocks outside $\cI$, no coordinates from $S^c$ contribute to the restriction on $S$. Therefore, when evaluating the $S$-coordinates of $\balpha_k^{(j)}$ and $\bbeta_k^{(j)}$, one may replace $(\bI-\bA_k)^T\bSigma_k^{-1}$ by its restriction to $S$, which yields the two identities above.

Substituting these identities into the previous display yields
\[
\Pi_{W^\perp}\!\left(
\row\Big(\{\bTheta_k^{(j)}-\bTheta_k^{(0)}\}_{j\in\{j_1,\dots,j_\ell\}}\Big)
\right)
\]
contained in
\[
\Big\langle
[\bar{\bH}_k]_{S,:}^{T}
(\bI-[\bA_k]_{S,S})^{T}
[\bSigma_k]_{S,S}^{-1}
[\be_{\rho_k(t_j)}]_{S},
\ 
[\bar{\bH}_k]_{S,:}^{T}
\Big(
(\bI-[\bA_k]_{S,S})^{T}
[\bSigma_k]_{S,S}^{-1}
[\bu_k^{(j)}]_{S}
+
[\bd_k^{(j)}]_{S}
\Big)
\ \Big|\ 
j\in\{j_1,\dots,j_\ell\}
\Big\rangle .
\]
This proves the claim.
\end{proof}

\begin{lemma}
\label{lem:Rk^-1}
Let $\bR_k$ be the matrix in the block-RQ decomposition of $\bH_k$, where the blocks are ordered according to a block topological order. Suppose that
\[
[\bR_k]_{C_{k,i},C_{k,j}}=0
\qquad
\text{whenever } C_{k,j}\notin Anc_k^\rblk(C_{k,i}).
\]
Then the inverse matrix $\bR_k^{-1}$ satisfies the same block support property:
\[
[\bR_k^{-1}]_{C_{k,i},C_{k,j}}=0
\qquad
\text{whenever } C_{k,j}\notin Anc_k^\rblk(C_{k,i}).
\]
\end{lemma}

\begin{proof}
Write
\[
\bR_k=\bD_k+\bN_k,
\]
where $\bD_k$ is the block-diagonal part of $\bR_k$, and $\bN_k:=\bR_k-\bD_k$ is the off-diagonal part. By assumption, $[\bN_k]_{C_{k,i},C_{k,j}}$ can be nonzero only if $C_{k,j}\in anc_k^\rblk(C_{k,i})$.

Since $\bR_k$ is invertible, each diagonal block of $\bR_k$ is invertible, and hence so is $\bD_k$. Therefore,
\[
\bR_k=\bD_k(\bI+\bD_k^{-1}\bN_k).
\]
Let
\[
\bM_k:=\bD_k^{-1}\bN_k.
\]
Then $\bM_k$ has zero diagonal blocks and satisfies
\[
[\bM_k]_{C_{k,i},C_{k,j}}=0
\qquad
\text{whenever } C_{k,j}\notin anc_k^\rblk(C_{k,i})\}.
\]
Since the blocks are arranged in a block topological order, $\bM_k$ is strictly block triangular, and hence nilpotent. In particular, if $m_k$ is the number of blocks, then
\[
\bM_k^{m_k}=0.
\]
It follows that
\[
(\bI+\bM_k)^{-1}
=
\bI-\bM_k+\bM_k^2-\cdots+(-1)^{m_k-1}\bM_k^{m_k-1},
\]
and thus
\[
\bR_k^{-1}
=
(\bI+\bM_k)^{-1}\bD_k^{-1}.
\]

It remains to verify that every term in the above expression has the same block support pattern. First, $\bD_k^{-1}$ is block diagonal, so its $(i,j)$ block can be nonzero only when $i=j$, equivalently when $C_{k,j}\in Anc_k^\rblk(C_{k,i})$. Next, $\bM_k$ satisfies the same ancestor-restricted support property as above. By transitivity of the block ancestor relation, the product of two matrices with this support property again has the same support property. Hence each power $\bM_k^s$ also satisfies
\[
[\bM_k^s]_{C_{k,i},C_{k,j}}=0
\qquad
\text{whenever } C_{k,j}\notin Anc_k^\rblk(C_{k,i}).
\]
Therefore the finite sum $(\bI+\bM_k)^{-1}$ satisfies the same property, and so does its product with $\bD_k^{-1}$. Consequently,
\[
[\bR_k^{-1}]_{C_{k,i},C_{k,j}}=0
\qquad
\text{whenever } C_{k,j}\notin Anc_k^\rblk(C_{k,i}),
\]
as claimed.
\end{proof}

\section{Additional finite-sample details}
\label{app: expriments}

This appendix collects the implementation details omitted from Section \ref{sec:experiments}, including the finite-sample counterparts of the population algorithms, the full simulation setup, the exact evaluation metrics, and the hyperparameter tuning protocol.

\subsection{Finite-sample algorithms}

Algorithm \ref{alg:block ancestor FS} is the finite-sample counterpart of Algorithm \ref{alg:block ancestor}. It replaces exact row-space dimension equalities by a singular-value ratio criterion controlled by the threshold $\tau$.

\begin{algorithm}[t] 
\caption{ID-Block-Ancestor-FS} 
\label{alg:block ancestor FS} \begin{algorithmic}[1] 
\Require $\hat{C} \subseteq \cJ_k$, $\{\hat{\Theta}_k^{(j)}\}_{j \in \{0\}\cup \hat{C}}$,  $\mathcal{I}$, $\{\hat{\bq}_{k,i}\}_{i\in\cI}$ 
\Ensure $\{\hat{\bq}_{k,i}\}_{i\in\hat{C}}$, $\cA$ 
\State Initialize $\cA \gets \mathcal{I}$
\State Form \[ \Tilde{\bM}_k(\hat{C}) \gets \begin{bmatrix} \hat{\bTheta}_k^{(j_1)}-\hat{\bTheta}_k^{(0)}\\ \hat{\bTheta}_k^{(j_2)}-\hat{\bTheta}_k^{(0)}\\ \vdots\\ \hat{\bTheta}_k^{(j_{|\hat{C}|})}-\hat{\bTheta}_k^{(0)} \end{bmatrix}, \qquad \hat{C}=\{j_1,\dots,j_{|\hat{C}|}\} \] 

\For{each block $B \in \mathcal{I}$}
\State Form $\bW_{-B}$ by stacking the row vectors $\{\hat{\bq}_{k,i}\}_{i\in\cI\backslash B}$.
\State Compute \[ \bM_{-B}(\hat{C}) \gets \Tilde{\bM}_k(\hat{C}) (\bI - \bW_{-B}^T(\bW_{-B}\bW_{-B}^T)^{-1}\bW_{-B}).\]
\State Compute\[ \rho(\bM_{-B}(\hat{C})) = \frac{\sum_{i=1}^{|\hat{C}|} \sigma_i(\bM_{-B}(\hat{C}))}{\sum_i\sigma_i(\bM_{-B}(\hat{C}))},\]
where $\sigma_i(\cdot)$ denotes the $i$-th largest singular value.
\If{$ \rho(\bM_{-B}(\hat{C}))> 1-\tau$} 
\State Remove $B$ from $\cA$ 
\EndIf 
\EndFor 
\State Form $\bW_{\mathrm{anc}}$ by stacking the row vectors $\{\hat{\bq}_{k,i}\}_{i\in\cA}$.
\State Compute \[\bM \gets \Tilde{\bM}_k(\hat{C}) (\bI - \bW_{\mathrm{anc}}^T(\bW_{\mathrm{anc}}\bW_{\mathrm{anc}}^T)^{-1}\bW_{\mathrm{anc}}). \]
\State Set $\{\hat{\bq}_{k,i}\}_{i\in\hat{C}}$ to be the right singular vectors of $\bM$ corresponding to its top $|\hat{C}|$ singular values.
\State \Return $\{\hat{\bq}_{k,i}\}_{i\in\hat{C}}$, $\cA$ 
\end{algorithmic} \end{algorithm}

Algorithm \ref{alg:causal order client FS} is the finite-sample counterpart of Algorithm \ref{alg:causal order client}. It searches for the smallest subset whose projected perturbation matrix is approximately rank-$r$ under the same thresholding rule.

\begin{algorithm}[t] 
\caption{ID-Block-Causal-Order-FS} \label{alg:causal order client FS} \begin{algorithmic}[1] 
\Require $\{\hat{\Theta}_k^{(j)}\}_{j \in \{0\}\cup \cJ_k}$, $p_k$ 
\Ensure  $\hat{\mathcal C}_k$, $\hat{\prec}_k^\rblk$, $\hat{\bQ}_k$
\State Initialize $\mathcal{I}_0 \gets \emptyset$, $R_0 \gets \cJ_k$, $\widehat Q_k \gets \mathbf{0}_{p_k\times n_k}$, $t \gets 1$ 
\While{$R_{t-1}\neq \emptyset$} 
\State $\mathrm{flag} \gets \mathrm{True}$, $r\gets 1$.
\While{$\mathrm{flag}$}
\State Form $\bW_t \in\bbR^{|\cI_{t-1}|\times n_k}$ by stacking the row vectors $\{\hat{\bq}_{k,i}\}_{i\in\cI_{t-1}}$.
\State Define $\cR = \{C| C\subseteq R_{t-1},|C| = r\}$. 
\For{each subset $\Tilde{C} \subseteq \cR$ with $|\widehat C|=r$}
\State Form \[ \Tilde{\bM}_k(\Tilde{C}) \gets \begin{bmatrix} \hat{\bTheta}_k^{(j_1)}-\hat{\bTheta}_k^{(0)}\\ \hat{\bTheta}_k^{(j_2)}-\hat{\bTheta}_k^{(0)}\\ \vdots\\ \hat{\bTheta}_k^{(j_r)}-\hat{\bTheta}_k^{(0)} \end{bmatrix}, \qquad \Tilde{C}=\{j_1,\dots,j_r\} \] 
\State Compute \[ \bM_t(\Tilde{C}) \gets \Tilde{\bM}_k(\Tilde{C}) (\bI - \bW_t^T(\bW_t\bW_t^T)^{-1}\bW_t). \] 
\State Compute \[ \rho_r(\bM_t(\Tilde{C})) = \frac{\sum_{i=1}^r \sigma_i(\bM_t(\Tilde{C}))} {\sum_i \sigma_i(\bM_t(\Tilde{C}))}, \] where $\sigma_i(\cdot)$ denotes the $i$-th largest singular value.
\EndFor 
\If{there exists at least one subset $\Tilde{C}$ with $\rho_r(\bM_t(\Tilde{C})) \ge 1-\tau$}
\State Among all such subsets of size $r$, choose $\hat{C} = \Tilde{C}$ which maximizes $\rho_r(\bM_t(\hat{C}))$ 
\State Set $\mathrm{flag} \gets \mathrm{False}$
\Else 
\State Set $r \gets r+1$ 
\EndIf 
\EndWhile
\State $(\{\hat{\bq}_{k,i}\}_{i\in\hat{C}},\cA) \gets$ ID-Block-Ancestor-FS$(\hat{C}, \{\hat{\Theta}_k^{(j)}\}_{j \in \{0\}\cup \hat{C}}$,  $\cI_{t-1}$, $\{\hat{\bq}_{k,i}\}_{i\in\cI_{t-1}})$ 
\State Add $\hat{C} \hat{\prec}_k^\rblk C'$ for any $\Tilde{C} \hat{\preceq}_k^\rblk C', \Tilde{C}\in\cA$.
\State Update $\hat{\bQ}_k \gets [\Tilde{\bQ}_{k,\hat{C}};\hat{\bQ}_k]$ where $\Tilde{\bQ}_{k,\hat{C}} \in \bbR^{|\hat{C}|\times n_k}$ is the matrix formed by stacking the row vectors $\{\hat{\bq}_{k,i}\}_{i\in \hat{C}}$.
\State Update $\mathcal{I}_t \gets \{\hat{C}\}\cup \mathcal{I}_{t-1}$, $R_t \gets R_{t-1}\setminus \hat{C}$, $t\gets t+1$.
\EndWhile 
\State Set $\widehat{\mathcal C}_k \gets \mathcal{I}_{t-1}$ 
\State \Return $\hat{\mathcal C}_k$, $\hat{\prec}_k^\rblk$, $\hat{\bQ}_k$.
\end{algorithmic} 
\end{algorithm}

\subsection{Additional simulation details}

\paragraph{Data generation.}
We consider three simulation settings with different global causal structures and client-specific missing patterns as follows. Each setting is specified by a global causal coefficient matrix $\bA\in\mathbb{R}^{p\times p}$, a collection of used latent node sets $\{O_k\}_{k=1}^K$ with $O_k\subseteq [p]$, and client-specific observation dimensions $\{n_k\}_{k=1}^K$.

\paragraph{Setting A.}
Here $p=5$ and
\[
\bA=
\begin{pmatrix}
0 & 1 & 0.7 & 0 & 0\\
0 & 0 & 0 & -1 & 0\\
0 & 0 & 0 & 1.3 & -0.9\\
0 & 0 & 0 & 0 & 0.6\\
0 & 0 & 0 & 0 & 0
\end{pmatrix}.
\]
We take $K=5$ clients, with
$O_1=\{1,3,4,5\},\quad
O_2=\{1,2,4,5\},\quad
O_3=\{1,2,3,4\},\quad
O_4=\{1,2,3\},\quad
O_5=\{2,3,5\}$,
and observation dimensions
$(n_1,n_2,n_3,n_4,n_5)=(8,6,7,6,5)$.

\paragraph{Setting B.}
Here $p=8$ and
\[
\bA=
\begin{pmatrix}
0 & -0.8 & 0.1 & 0 & 0 & 0 & 0 & 0\\
0 & 0 & 0.6 & 1.2 & 0 & 0 & 0 & 0\\
0 & 0 & 0 & 0 & 0 & 0 & 0 & 1\\
0 & 0 & 0 & 0 & 0 & 0 & -0.9 & 0\\
0 & 0 & 0 & 0 & 0 & 1.1 & 0 & 0\\
0 & 0 & 0 & 0 & 0 & 0 & 0.7 & 0\\
0 & 0 & 0 & 0 & 0 & 0 & 0 & 0\\
0 & 0 & 0 & 0 & 0 & 0 & 0 & 0
\end{pmatrix}.
\]
We take $K=4$ clients, with 
$O_1=\{2,3,4,7,8\},\quad
O_2=\{1,2,3,5,6,7\},\quad
O_3=\{1,2,4,5,6,8\},\quad
O_4=\{1,3,4,5,7,8\}$,
and observation dimensions $(n_1,n_2,n_3,n_4)=(9,10,8,9)$.

\paragraph{Setting C.}
Here $p=10$ and
\[
\bA=
\begin{pmatrix}
0 & 0 & -0.9 & 0 & 0 & 0 & 0 & 0 & 0 & 0\\
0 & 0 & 1 & 0 & 0 & 0.7 & 0 & 0 & 0 & 0\\
0 & 0 & 0 & 1.2 & 0 & 0 & 0 & 0 & 0 & -0.8\\
0 & 0 & 0 & 0 & 0 & 0 & 0 & 0 & 0 & -1\\
0 & 0 & 0 & 0 & 0 & 0 & 0 & 0 & 0 & 1.3\\
0 & 0 & 0 & 0 & 0 & 0 & 1.1 & 0.9 & 0 & 0\\
0 & 0 & 0 & 0 & 0 & 0 & 0 & -0.5 & 0 & 0\\
0 & 0 & 0 & 0 & 0 & 0 & 0 & 0 & 0 & 0\\
0 & 0 & 0 & 0 & 0 & 0 & 0 & 0 & 0 & 0\\
0 & 0 & 0 & 0 & 0 & 0 & 0 & 0 & 0 & 0
\end{pmatrix}.
\]
We take $K=5$ clients, with 
$O_1=\{1,2,3,4,5,6,7,9\},\quad
O_2=\{3,5,6,7,8,10\},\quad
O_3=\{1,2,4,7,8,9,10\},\quad
O_4=\{1,3,4,5,6,7,10\},\quad
O_5=\{1,2,3,6,8,9\}$,
and observation dimensions
$(n_1,n_2,n_3,n_4,n_5)=(12,10,9,11,8)$.

We consider two types of noise distributions, namely Gaussian noise and Laplace noise. For each network setting and each noise setting, the intervention magnitudes on the intervened entries of $\bA$ and $\bSigma$ are independently drawn from the uniform distribution on $[2,3]$. By construction, the baseline nonzero entries of $\bA$ have magnitudes strictly smaller than the intervention range, which helps separate the observational structural coefficients from the intervention effects. For each client $k\in[K]$, we independently generate a column full-rank matrix $\bG_k$ with entries sampled from the uniform distribution on $[-1,1]$.

\paragraph{Sample sizes.}
We evaluate the method under finite sample sizes $n$ ranging from $10^3$ to $10^6$, and also include the population case $n=\infty$, where the estimated precision matrices are replaced by the corresponding population precision matrices.

\subsection{Evaluation metrics}

For all metrics, for ease of exposition, we first align the row ordering of $\hat{\bQ}_k$ with that of $\bQ_k$ according to the permutation induced by $\bP_\sigma$. This alignment is used only for presentation and does not affect the evaluation result.

\paragraph{Metric $E_1$: global ancestor recovery.}
We compare the true partial order $\prec$ and the estimated partial order $\hat{\prec}$ on the node set $\cV$, and report their F1-score:
\[
E_1 \triangleq \mathrm{F1}(\prec,\hat{\prec}).
\]

\paragraph{Metric $E_2$: block-level subspace error.}
For each client $k$ and each true block $C \in \mathcal C_k$, let $\hat C=\psi^{-1}(C)$ be the corresponding estimated block. We compare the row spaces of $[\bQ_k]_{\rho_k(C),:}$ and $[\hat{\bQ}_k]_{\rho_k(\psi(\hat{C})),:}$ through their projection matrices
\[
\bP_{k,C}=[\bQ_k]_{\rho_k(C),:}[\bQ_k]_{\rho_k(C),:}^T, 
\qquad 
\hat{\bP}_{k,C}=[\hat{\bQ}_k]_{\rho_k(\psi(\hat{C})),:}[\hat{\bQ}_k]_{\rho_k(\psi(\hat{C})),:}^T.
\]
We define
\[
E_{k,C} \triangleq \frac{\|\hat{\bP}_{k,C}-\bP_{k,C}\|_F^2}{\sqrt{2}},
\qquad
E_2 \triangleq\frac{1}{K}\sum_{k=1}^K \frac{1}{|\mathcal C_k|}\sum_{C\in\mathcal C_k} E_{k,C}.
\]

\paragraph{Metric $E_3$: support leakage error.}
Since $\hat{Z}_k=\hat{\bQ}_k X_k=\hat{\bQ}_k \bG_k Z_k$, Corollary \ref{cor:zk} implies that $[\hat{\bQ}_k\bG_k]_{\rho_k(i),\rho_k(j)}$ can be nonzero only if the block containing $j$ is a block-level ancestor of the block containing $i$. Let
\[
\hat{\bN}_k \triangleq \hat{\bQ}_k\bG_k,
\qquad
supp(\hat{\bN}_k)\triangleq\{(i,j):|[\hat{\bN}_k]_{\rho_k(i),\rho_k(j)}|>\epsilon\},
\]
where $\epsilon=10^{-2}$. Define the violation set
\[
\cT_k \triangleq \{(i,j): C_k(i)\not\preceq_k^{\mathrm{blk}} C_k(j), \ i,j\in O_k\}.
\]
Then the leakage error for client $k$ is
\[
E_{3,k}\triangleq\frac{|supp(\hat{\bN}_k)\cap\cT_k|}{|\cT_k|},
\qquad
E_3\triangleq\frac{1}{K}\sum_{k=1}^K E_{3,k}.
\]

\paragraph{Interpretation.}
The three metrics quantify complementary aspects of recovery: $E_1$ measures the accuracy of the global ancestral relation, $E_2$ measures the recovery of client-specific block-associated subspaces, and $E_3$ measures whether the estimated latent variables respect the block-level ancestral support implied by Corollary \ref{cor:zk}. All three metrics take values in $[0,1]$, where larger $E_1$ indicates better performance, while smaller $E_2$ and $E_3$ indicate more accurate recovery.

\subsection{Hyperparameter tuning}

The hyperparameter $\tau$ is chosen according to the sample size $n$. For each $n$, we conduct 10 additional independent runs of data generation and estimation, and search over
\[
\tau \in \{0.01,0.02,\dots,0.10\}.
\]
For each candidate value, we compute
\[
Score = E_1 - E_2 - E_3,
\]
and select the $\tau$ that achieves the highest average score over these runs. The resulting value is then used in the final evaluation for that sample size. When such simulation-based calibration is unavailable, we recommend $\tau=0.05$ as a robust default.

\subsection{Additional experimental results}

For the six scenarios induced by three graph settings and two noise settings, we plot the three evaluation metrics as functions of the sample size $n$ in Figure \ref{fig:setA_gau} and Figures \ref{fig:setA_lap}--\ref{fig:setC_lap}. Across all scenarios, recovery improves steadily as $n$ increases. In the population case $n=\infty$, we obtain $E_1=1$ and $E_2=E_3=0$, which exactly matches the population-level guarantees in Theorems \ref{thm:client-main} and \ref{thm:global} together with Corollary \ref{cor:zk}.

\begin{figure}[h]
    \centering
    \begin{subfigure}[t]{0.32\linewidth}
        \centering
        \includegraphics[width=\linewidth]{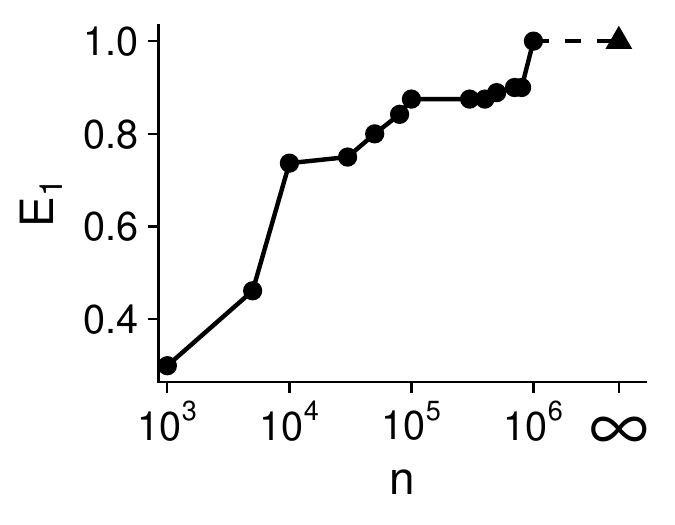}
        \caption{Global ancestral F1.}

    \end{subfigure}
    \hfill
    \begin{subfigure}[t]{0.32\linewidth}
        \centering
        \includegraphics[width=\linewidth]{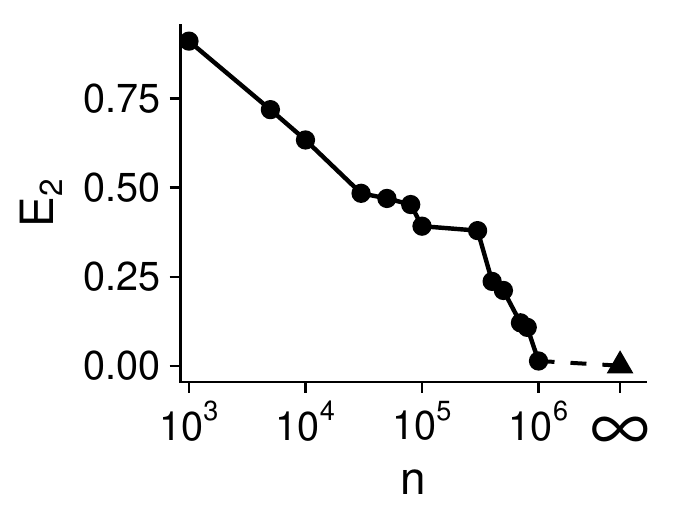}
        \caption{Subspace error of $\{\hat{\bQ}_k\}_{k=1}^K$.}

    \end{subfigure}
    \hfill
    \begin{subfigure}[t]{0.32\linewidth}
        \centering
        \includegraphics[width=\linewidth]{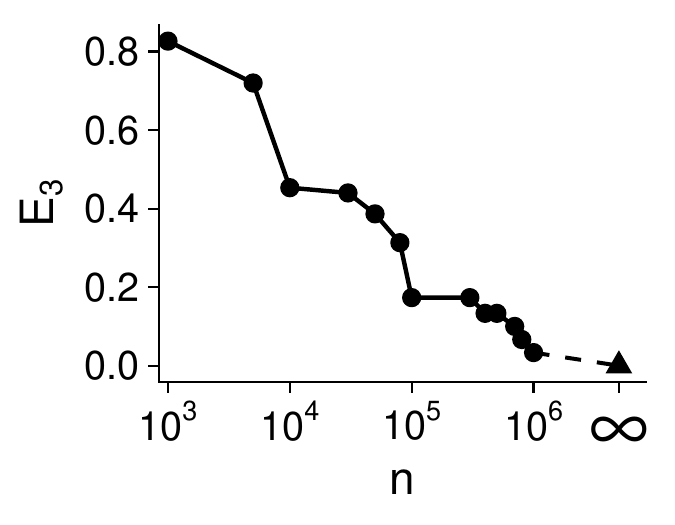}
        \caption{Support error of $\{\hat{Z}_k\}_{k=1}^K$.}
        
    \end{subfigure}
    \caption{Finite-sample performance under one graph setting A and laplace noise. As $n$ increases, global ancestral recovery improves while subspace and support errors decrease. The rightmost marker at $\infty$ corresponds to the population-level precision-matrix input.}
    \label{fig:setA_lap}
\end{figure}

\begin{figure}[h]
    \centering
    \begin{subfigure}[t]{0.32\linewidth}
        \centering
        \includegraphics[width=\linewidth]{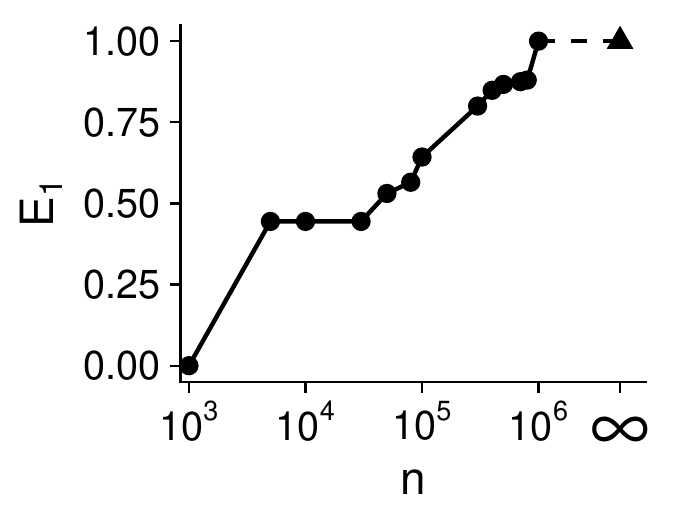}
        \caption{Global ancestral F1.}

    \end{subfigure}
    \hfill
    \begin{subfigure}[t]{0.32\linewidth}
        \centering
        \includegraphics[width=\linewidth]{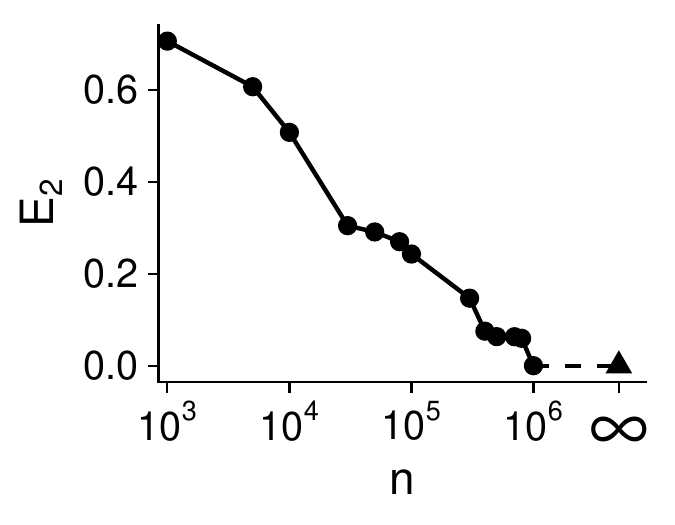}
        \caption{Subspace error of $\{\hat{\bQ}_k\}_{k=1}^K$.}

    \end{subfigure}
    \hfill
    \begin{subfigure}[t]{0.32\linewidth}
        \centering
        \includegraphics[width=\linewidth]{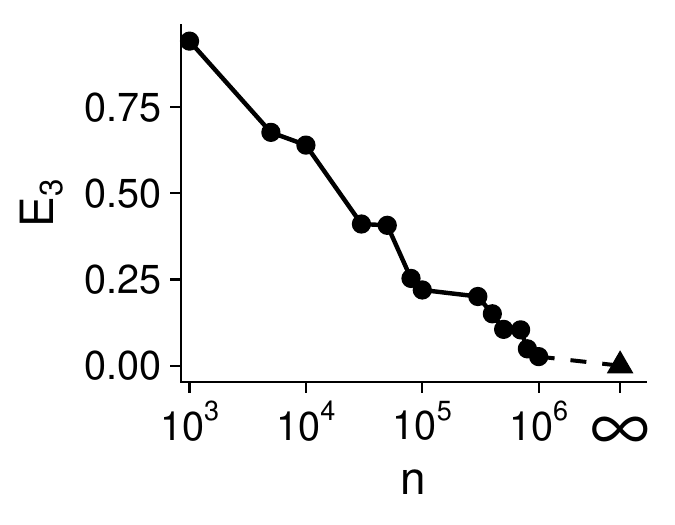}
        \caption{Support error of $\{\hat{Z}_k\}_{k=1}^K$.}
        
    \end{subfigure}
    \caption{Finite-sample performance under one graph setting B and gaussian noise. As $n$ increases, global ancestral recovery improves while subspace and support errors decrease. The rightmost marker at $\infty$ corresponds to the population-level precision-matrix input.}
    \label{fig:setB_gau}
\end{figure}

\begin{figure}[h]
    \centering
    \begin{subfigure}[t]{0.32\linewidth}
        \centering
        \includegraphics[width=\linewidth]{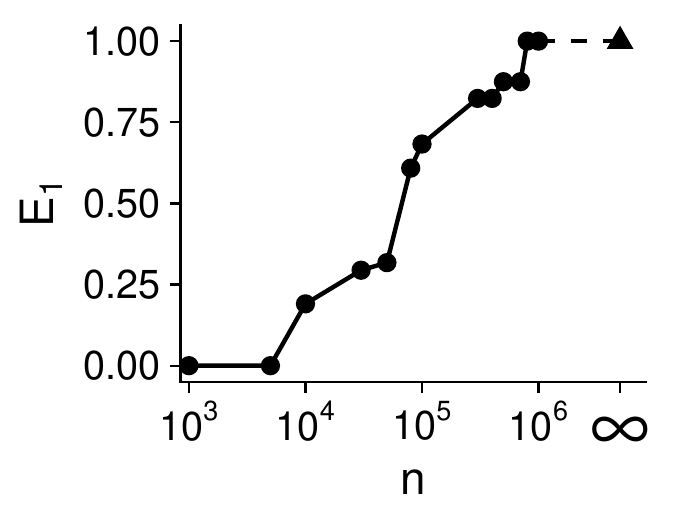}
        \caption{Global ancestral F1.}

    \end{subfigure}
    \hfill
    \begin{subfigure}[t]{0.32\linewidth}
        \centering
        \includegraphics[width=\linewidth]{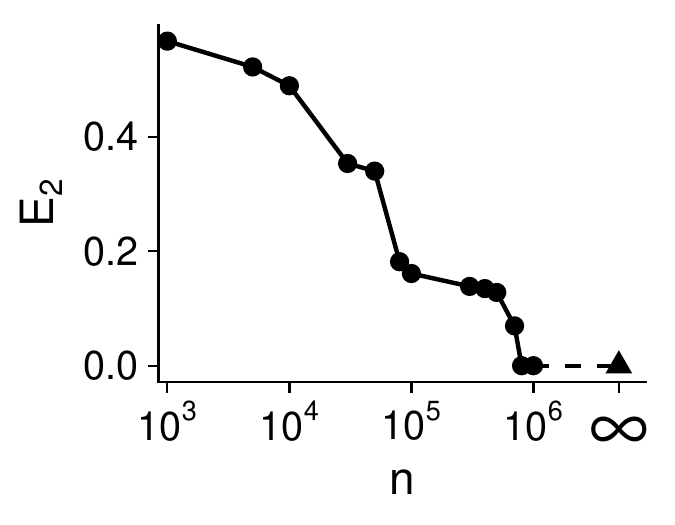}
        \caption{Subspace error of $\{\hat{\bQ}_k\}_{k=1}^K$.}

    \end{subfigure}
    \hfill
    \begin{subfigure}[t]{0.32\linewidth}
        \centering
        \includegraphics[width=\linewidth]{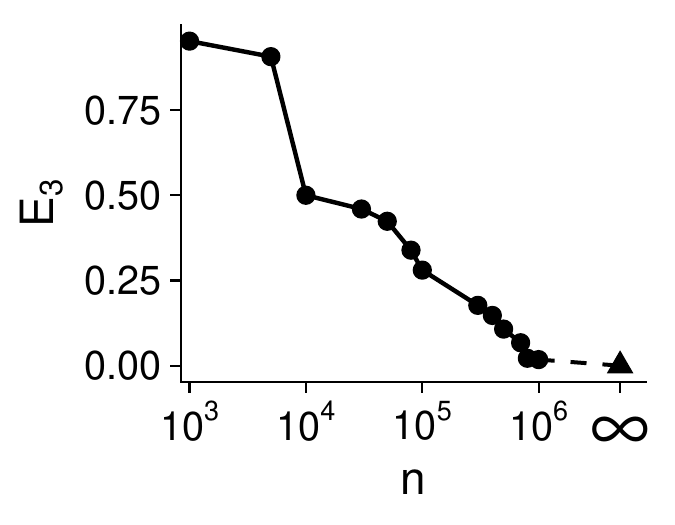}
        \caption{Support error of $\{\hat{Z}_k\}_{k=1}^K$.}
        
    \end{subfigure}
    \caption{Finite-sample performance under one graph setting B and laplace noise. As $n$ increases, global ancestral recovery improves while subspace and support errors decrease. The rightmost marker at $\infty$ corresponds to the population-level precision-matrix input.}
    \label{fig:setB_lap}
\end{figure}

\begin{figure}[h]
    \centering
    \begin{subfigure}[t]{0.32\linewidth}
        \centering
        \includegraphics[width=\linewidth]{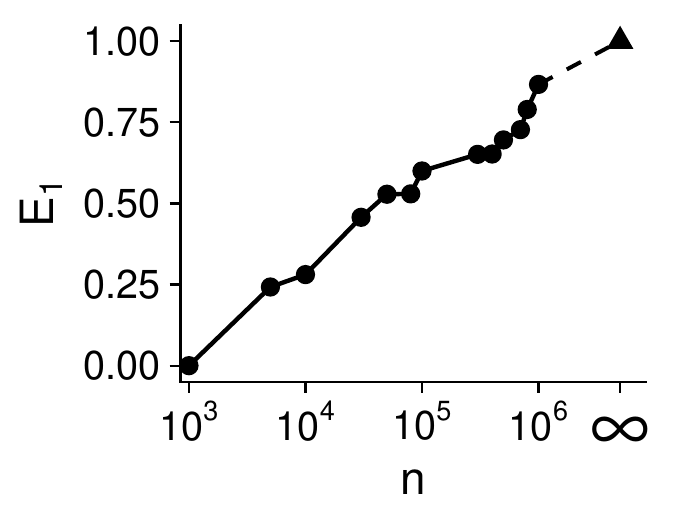}
        \caption{Global ancestral F1.}

    \end{subfigure}
    \hfill
    \begin{subfigure}[t]{0.32\linewidth}
        \centering
        \includegraphics[width=\linewidth]{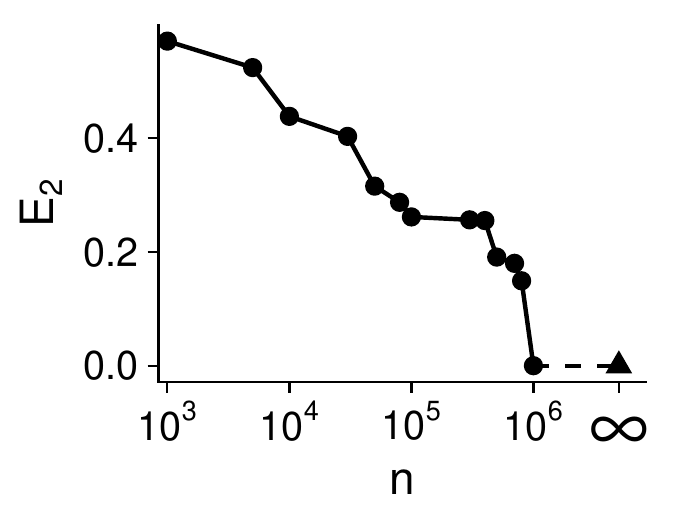}
        \caption{Subspace error of $\{\hat{\bQ}_k\}_{k=1}^K$.}

    \end{subfigure}
    \hfill
    \begin{subfigure}[t]{0.32\linewidth}
        \centering
        \includegraphics[width=\linewidth]{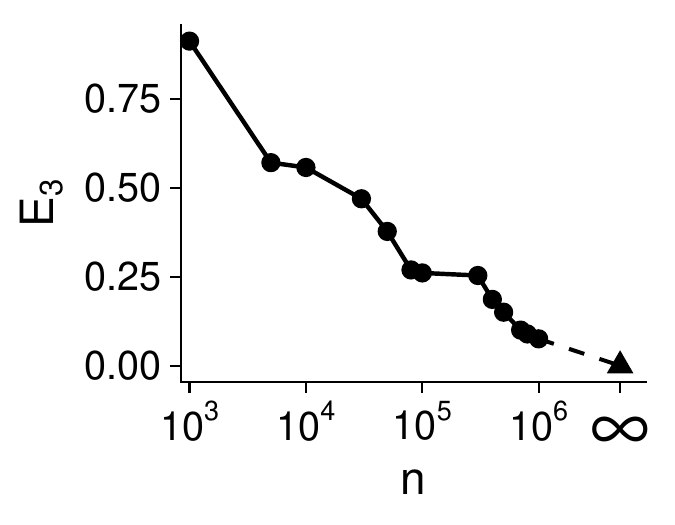}
        \caption{Support error of $\{\hat{Z}_k\}_{k=1}^K$.}
        
    \end{subfigure}
    \caption{Finite-sample performance under one graph setting C and gaussian noise. As $n$ increases, global ancestral recovery improves while subspace and support errors decrease. The rightmost marker at $\infty$ corresponds to the population-level precision-matrix input.}
    \label{fig:setC_gau}
\end{figure}

\begin{figure}[h]
    \centering
    \begin{subfigure}[t]{0.32\linewidth}
        \centering
        \includegraphics[width=\linewidth]{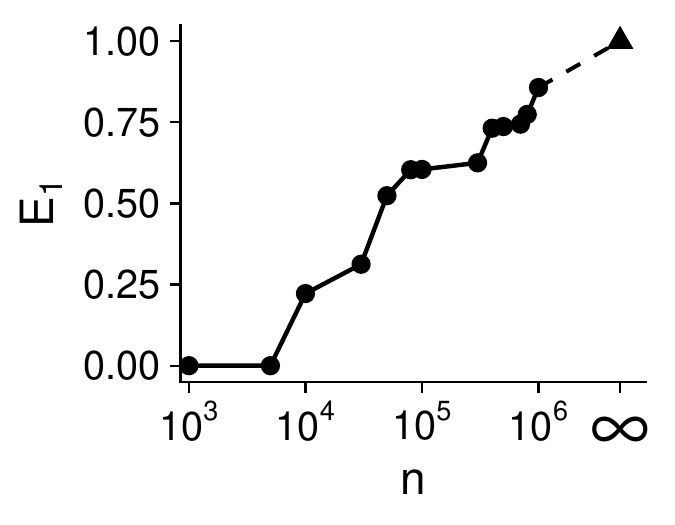}
        \caption{Global ancestral F1.}

    \end{subfigure}
    \hfill
    \begin{subfigure}[t]{0.32\linewidth}
        \centering
        \includegraphics[width=\linewidth]{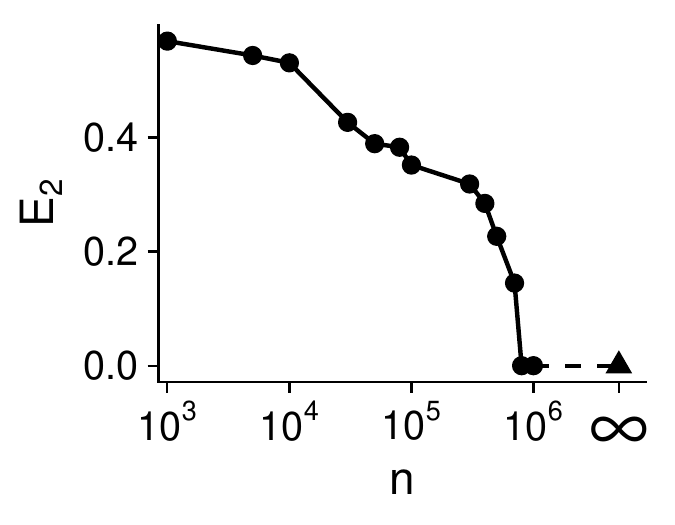}
        \caption{Subspace error of $\{\hat{\bQ}_k\}_{k=1}^K$.}

    \end{subfigure}
    \hfill
    \begin{subfigure}[t]{0.32\linewidth}
        \centering
        \includegraphics[width=\linewidth]{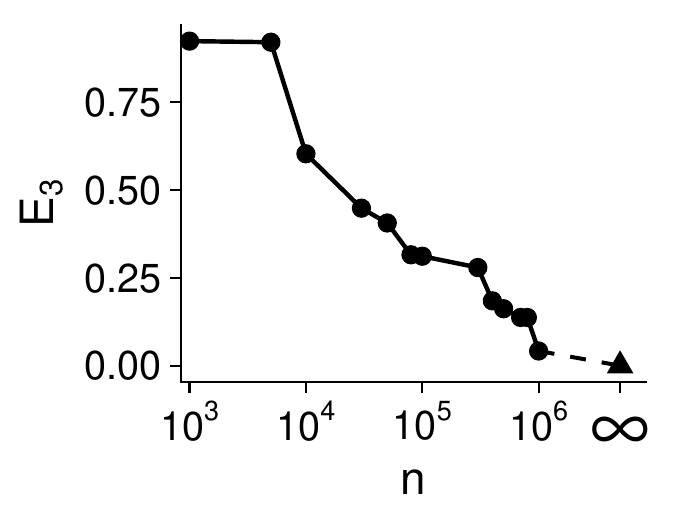}
        \caption{Support error of $\{\hat{Z}_k\}_{k=1}^K$.}
        
    \end{subfigure}
    \caption{Finite-sample performance under one graph setting C and laplace noise. As $n$ increases, global ancestral recovery improves while subspace and support errors decrease. The rightmost marker at $\infty$ corresponds to the population-level precision-matrix input.}
    \label{fig:setC_lap}
\end{figure}

\subsection{Sensitivity and scalability analysis}
\label{app:sensitivity_scalability}

We further evaluate sensitivity to the global latent dimension $p$, the number of clients $K$, and the client coverage pattern. For each $p$, we generate an Erd\H{o}s--R\'enyi DAG under a fixed topological order, include each admissible edge independently with probability $3/p$, and sample its coefficient from $\operatorname{Unif}[-1,1]$. Each node is independently included in $O_k$ with probability $q$, and the observation dimension is set to $n_k=\lceil 1.5p_k\rceil$, where $p_k=|O_k|$. The remaining parameters follow the simulation protocol above. Because $p_k$, cross-client overlap, and the maximum induced block size are jointly determined by the missingness pattern, we use $q$ as the coverage control: $\mathbb E[p_k]=qp$ and $\mathbb E[|O_k\cap O_\ell|]=q^2p$, while larger $q$ generally reduces marginalization-induced block ambiguity.

We conduct three controlled experiments: (i) $p\in\{10,20,30,40\}$ with $K=15$ and $q=0.6$; (ii) $K\in\{5,10,15,20\}$ with $p=20$ and $q=0.6$; and (iii) $q\in\{0.5,0.6,0.7,0.8\}$ with $p=20$ and $K=15$. We report $E_1,E_2,E_3$ for $n\in\{10^4,5\times10^4,10^5,5\times10^5,10^6,\infty\}$ in Tables~\ref{tab:sensitivity_p}--\ref{tab:sensitivity_q}. Runtime is measured from population precision-matrix inputs and therefore excludes data generation and precision estimation.

\begin{table}[t]
\centering
\caption{Sensitivity to the global latent dimension $p$ ($K=15$, $q=0.6$). Runtime is measured at $n=\infty$.}
\label{tab:sensitivity_p}
\small
\setlength{\tabcolsep}{6pt}
\begin{tabular}{ccrrrrrr}
\toprule
$p$ & Metric & $10^4$ & $5\!\times\!10^4$ & $10^5$ &
$5\!\times\!10^5$ & $10^6$ & $\infty$ \\
\midrule
10 & $E_1$ & 0.19 & 0.45 & 0.56 & 0.58 & 0.81 & 1.00 \\
   & $E_2$ & 0.32 & 0.30 & 0.24 & 0.14 & 0.00 & 0.00 \\
   & $E_3$ & 0.57 & 0.39 & 0.25 & 0.18 & 0.07 & 0.00 \\
   & Runtime (s) & -- & -- & -- & -- & -- & 0.17 \\
\cmidrule(lr){1-8}
20 & $E_1$ & 0.17 & 0.49 & 0.50 & 0.62 & 0.85 & 0.93 \\
   & $E_2$ & 0.45 & 0.28 & 0.27 & 0.18 & 0.10 & 0.00 \\
   & $E_3$ & 0.48 & 0.38 & 0.27 & 0.14 & 0.04 & 0.01 \\
   & Runtime (s) & -- & -- & -- & -- & -- & 0.44 \\
\cmidrule(lr){1-8}
30 & $E_1$ & 0.17 & 0.38 & 0.47 & 0.55 & 0.71 & 0.82 \\
   & $E_2$ & 0.49 & 0.38 & 0.30 & 0.11 & 0.08 & 0.03 \\
   & $E_3$ & 0.50 & 0.42 & 0.31 & 0.09 & 0.05 & 0.01 \\
   & Runtime (s) & -- & -- & -- & -- & -- & 2.67 \\
\cmidrule(lr){1-8}
40 & $E_1$ & 0.13 & 0.41 & 0.45 & 0.55 & 0.69 & 0.77 \\
   & $E_2$ & 0.38 & 0.34 & 0.25 & 0.20 & 0.07 & 0.05 \\
   & $E_3$ & 0.49 & 0.37 & 0.24 & 0.14 & 0.06 & 0.02 \\
   & Runtime (s) & -- & -- & -- & -- & -- & 4.23 \\
\bottomrule
\end{tabular}
\end{table}

\begin{table}[t]
\centering
\caption{Sensitivity to the number of clients $K$ ($p=20$, $q=0.6$). Runtime is measured at $n=\infty$.}
\label{tab:sensitivity_K}
\small
\setlength{\tabcolsep}{6pt}
\begin{tabular}{ccrrrrrr}
\toprule
$K$ & Metric & $10^4$ & $5\!\times\!10^4$ & $10^5$ &
$5\!\times\!10^5$ & $10^6$ & $\infty$ \\
\midrule
5  & $E_1$ & 0.13 & 0.18 & 0.19 & 0.24 & 0.32 & 0.34 \\
   & $E_2$ & 0.42 & 0.34 & 0.22 & 0.18 & 0.10 & 0.00 \\
   & $E_3$ & 0.58 & 0.48 & 0.34 & 0.11 & 0.00 & 0.00 \\
   & Runtime (s) & -- & -- & -- & -- & -- & 0.25 \\
\cmidrule(lr){1-8}
10 & $E_1$ & 0.15 & 0.44 & 0.46 & 0.55 & 0.76 & 0.81 \\
   & $E_2$ & 0.36 & 0.28 & 0.27 & 0.14 & 0.10 & 0.00 \\
   & $E_3$ & 0.48 & 0.39 & 0.28 & 0.13 & 0.08 & 0.00 \\
   & Runtime (s) & -- & -- & -- & -- & -- & 0.37 \\
\cmidrule(lr){1-8}
15 & $E_1$ & 0.17 & 0.49 & 0.50 & 0.62 & 0.85 & 0.93 \\
   & $E_2$ & 0.45 & 0.28 & 0.27 & 0.18 & 0.10 & 0.00 \\
   & $E_3$ & 0.48 & 0.38 & 0.27 & 0.14 & 0.04 & 0.01 \\
   & Runtime (s) & -- & -- & -- & -- & -- & 0.44 \\
\cmidrule(lr){1-8}
20 & $E_1$ & 0.18 & 0.52 & 0.54 & 0.67 & 0.92 & 1.00 \\
   & $E_2$ & 0.50 & 0.29 & 0.21 & 0.14 & 0.05 & 0.00 \\
   & $E_3$ & 0.47 & 0.36 & 0.25 & 0.17 & 0.04 & 0.00 \\
   & Runtime (s) & -- & -- & -- & -- & -- & 0.49 \\
\bottomrule
\end{tabular}
\end{table}

\begin{table}[t]
\centering
\caption{Sensitivity to the client coverage probability $q$ ($p=20$, $K=15$). Runtime is measured at $n=\infty$.}
\label{tab:sensitivity_q}
\small
\setlength{\tabcolsep}{6pt}
\begin{tabular}{ccrrrrrr}
\toprule
$q$ & Metric & $10^4$ & $5\!\times\!10^4$ & $10^5$ &
$5\!\times\!10^5$ & $10^6$ & $\infty$ \\
\midrule
0.5 & $E_1$ & 0.12 & 0.46 & 0.46 & 0.57 & 0.75 & 0.82 \\
    & $E_2$ & 0.55 & 0.36 & 0.29 & 0.15 & 0.08 & 0.00 \\
    & $E_3$ & 0.47 & 0.39 & 0.18 & 0.10 & 0.08 & 0.00 \\
    & Runtime (s) & -- & -- & -- & -- & -- & 0.53 \\
\cmidrule(lr){1-8}
0.6 & $E_1$ & 0.17 & 0.49 & 0.50 & 0.62 & 0.85 & 0.93 \\
    & $E_2$ & 0.45 & 0.28 & 0.27 & 0.18 & 0.10 & 0.00 \\
    & $E_3$ & 0.48 & 0.38 & 0.27 & 0.14 & 0.04 & 0.01 \\
    & Runtime (s) & -- & -- & -- & -- & -- & 0.44 \\
\cmidrule(lr){1-8}
0.7 & $E_1$ & 0.18 & 0.54 & 0.54 & 0.73 & 0.91 & 1.00 \\
    & $E_2$ & 0.44 & 0.29 & 0.22 & 0.14 & 0.10 & 0.00 \\
    & $E_3$ & 0.49 & 0.36 & 0.26 & 0.15 & 0.00 & 0.00 \\
    & Runtime (s) & -- & -- & -- & -- & -- & 0.82 \\
\cmidrule(lr){1-8}
0.8 & $E_1$ & 0.17 & 0.52 & 0.59 & 0.69 & 0.94 & 1.00 \\
    & $E_2$ & 0.41 & 0.32 & 0.27 & 0.18 & 0.10 & 0.00 \\
    & $E_3$ & 0.46 & 0.38 & 0.27 & 0.09 & 0.00 & 0.00 \\
    & Runtime (s) & -- & -- & -- & -- & -- & 0.77 \\
\bottomrule
\end{tabular}
\end{table}

Across all settings, recovery improves with the sample size: $E_1$ increases, while $E_2$ and $E_3$ generally decrease. Because the random client subsets do not enforce Assumption~\ref{ass: global cover edge}, global recovery can remain incomplete even at $n=\infty$; for example, the population $E_1$ decreases from $1.00$ to $0.77$ as $p$ increases from $10$ to $40$. Nevertheless, the population client-level errors remain small, with $E_2\leq 0.05$ and $E_3\leq 0.02$ across these values of $p$, confirming that incomplete coverage primarily affects global assembly rather than client-level recovery.

Increasing either the number of clients or their coverage improves global recovery. At $n=\infty$, $E_1$ increases from $0.34$ to $1.00$ as $K$ increases from $5$ to $20$, and from $0.82$ to $1.00$ as $q$ increases from $0.5$ to $0.8$, consistent with Assumption~\ref{ass: global cover edge} being satisfied more frequently. Runtime increases from $0.17$ to $4.23$ seconds as $p$ grows from $10$ to $40$, but only from $0.25$ to $0.49$ seconds as $K$ grows from $5$ to $20$. This agrees with the fact that computation depends more strongly on the client-level dimensions and block search, while its dependence on $K$ is at most linear. Runtime is not necessarily monotone in $q$: increasing $q$ enlarges $p_k$ and $n_k$, but generally reduces the induced block sizes and hence the combinatorial block search.

\end{document}